\documentclass[letterpaper]{article} 
\usepackage[draft]{aaai2027}  
\usepackage[hyphens]{url}  
\usepackage{graphicx} 
\usepackage{etoc}
\usepackage{natbib}  
\usepackage{caption} 
\usepackage{amsmath, amssymb, amsthm}
\usepackage{cleveref}

\usepackage{algorithm}
\usepackage{algpseudocode}
\algrenewcommand\algorithmicindent{0.4em}
\algrenewcommand{\algorithmiccomment}[1]{\;\;{\small$\triangleright$ #1}}
\usepackage{newfloat}
\usepackage{listings}
\DeclareCaptionStyle{ruled}{labelfont=normalfont,labelsep=colon,strut=off} 
\floatstyle{ruled}
\newfloat{listing}{tb}{lst}{}
\floatname{listing}{Listing}
\usepackage{amsmath,amsfonts,amsthm}
\usepackage{mathtools}
\usepackage{booktabs}
\usepackage{cancel}
\usepackage{makecell}
\usepackage{enumitem}
\usepackage{thmtools}
 \usepackage{multirow} 
\usepackage{subcaption}

\declaretheorem[name=Theorem]{thm}

\declaretheorem[name=Assumption,sibling=thm]{assumption}
\declaretheorem[name=Definition,sibling=thm,style=definition]{definition}

\newcommand{\cI}{\text{\textcircled{I}}}
\newcommand{\cII}{\text{\textcircled{II}}}
\newcommand{\dJFB}{d_{x}^{JFB}}

\newcommand{\Jcal}{\mathcal{J}}

\newcommand{\T}{^{\top}}
\newcommand{\inner}[2]{\left\langle #1, #2 \right\rangle}
\newcommand{\norm}[1]{\left\lVert#1\right\rVert_2}

\newcommand{\tA}{\tilde{A}}
\newcommand{\Honey}{\mathcal{H}_{1,y}}
\newcommand{\Htwoz}{\mathcal{H}_{2,z}}

\DeclareMathOperator*{\argmin}{arg\,min}

\newenvironment{restated}[1]
  {\renewcommand{\theassumption}{\ref*{#1}}
   \addtocounter{assumption}{-1}
   \assumption}
  {\endassumption}

\title{End-to-End Learning of Safe Optimal Feedback Control \\
in High Dimensions with Control Barrier Function Layers}
\author{
    Xingjian Li\textsuperscript{\rm 1}\equalcontrib,
    Kelvin Kan\textsuperscript{\rm 2}\equalcontrib,
    Deepanshu Verma\textsuperscript{\rm 3}\equalcontrib, \\
    Krishna Kumar\textsuperscript{\rm 1},
    Stanley Osher\textsuperscript{\rm 2},
    Samy Wu Fung\textsuperscript{\rm 4}
}
\affiliations{
    \textsuperscript{\rm 1}UT Austin, \textsuperscript{\rm 2}UCLA, \textsuperscript{\rm 3}Clemson University, \textsuperscript{\rm 4}Colorado School of Mines\\

}

\begin{document}

\maketitle
\addtocontents{toc}{\protect\setcounter{tocdepth}{-1}}

\begin{abstract}
 We consider the problem of learning high-dimensional semi-global feedback controllers under hard safety constraints enforced by control barrier functions (CBFs). Incorporating CBFs into end-to-end policy training requires embedding a quadratic-program-based safety filter as an optimization layer, but computational and differentiation bottlenecks have largely restricted prior approaches to low-dimensional systems, typically with at most 16 state dimensions. We address this limitation by combining operator splitting with the recently developed Jacobian-Free Backpropagation (JFB) method to enable scalable end-to-end training while preserving hard safety guarantees through the CBF safety filter. We justify this training methodology theoretically using nonsmooth analysis techniques and demonstrate its effectiveness on high-dimensional multi-agent nonlinear control problems with state and control dimensions up to 1200 and 400, respectively.
\end{abstract}

\section{Introduction}
We consider the problem of end-to-end learning of safe feedback controllers for high-dimensional nonlinear control systems. Specifically, we consider systems of the form
\begin{subequations}
\label{eq:original_problem}
\begin{align}
\label{eq:sys_ctrl_obj}
\min_{u \in U} \quad & \int_0^T L(t, z_x, u)\, dt + \omega  G(z_x(T)),
\end{align}
\begin{align} \label{eq:sys_dyn}
\text{s.t.} \quad & \dot z_x = f(t,z_x) + g(t,z_x)u(t), \\
&h(z_x(t)) \geq 0, \quad  z_x(0) = x,
\end{align}
\end{subequations} 
where $z_x \in \mathbb{R}^n$ denotes the state variable initialized at $x$, $u(t) \in U \subset \mathbb{R}^m$ is the control, $L$ is the running cost, and $G$ is the terminal cost weighted by $\omega > 0$. The function $h(z_x(t)) \colon \mathbb{R}^n \to \mathbb{R}$ defines the safe set through its superlevel set, and the constraint $h(z_x(t)) \geq 0$ requires the trajectory to remain safe for all time. Our goal is to learn $u^\star(t, z)$, that approximately solves~\eqref{eq:original_problem} simultaneously for all $x \sim \rho$, thereby yielding a semi-global controller that works for a wide range of $(t, z)$ without online retraining~\cite{onken2022neural, ruthotto2020machine}.
\begin{figure}[t]
    \centering
    \begin{subfigure}{0.48\columnwidth}
        \centering
        \includegraphics[width=\textwidth]{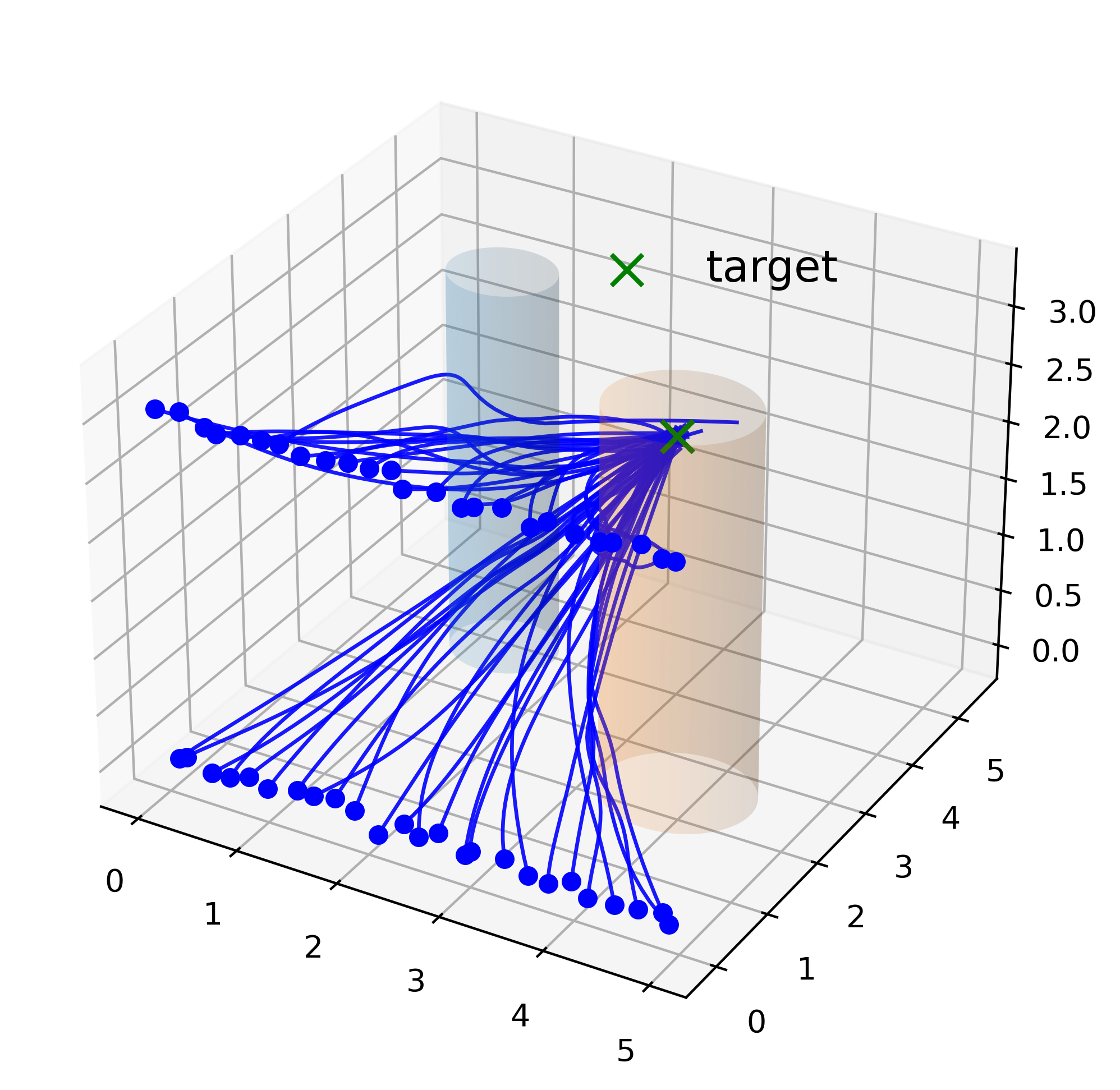}
    \end{subfigure}
    \hfill
    \begin{subfigure}{0.48\columnwidth}
        \centering
        \includegraphics[width=\textwidth]{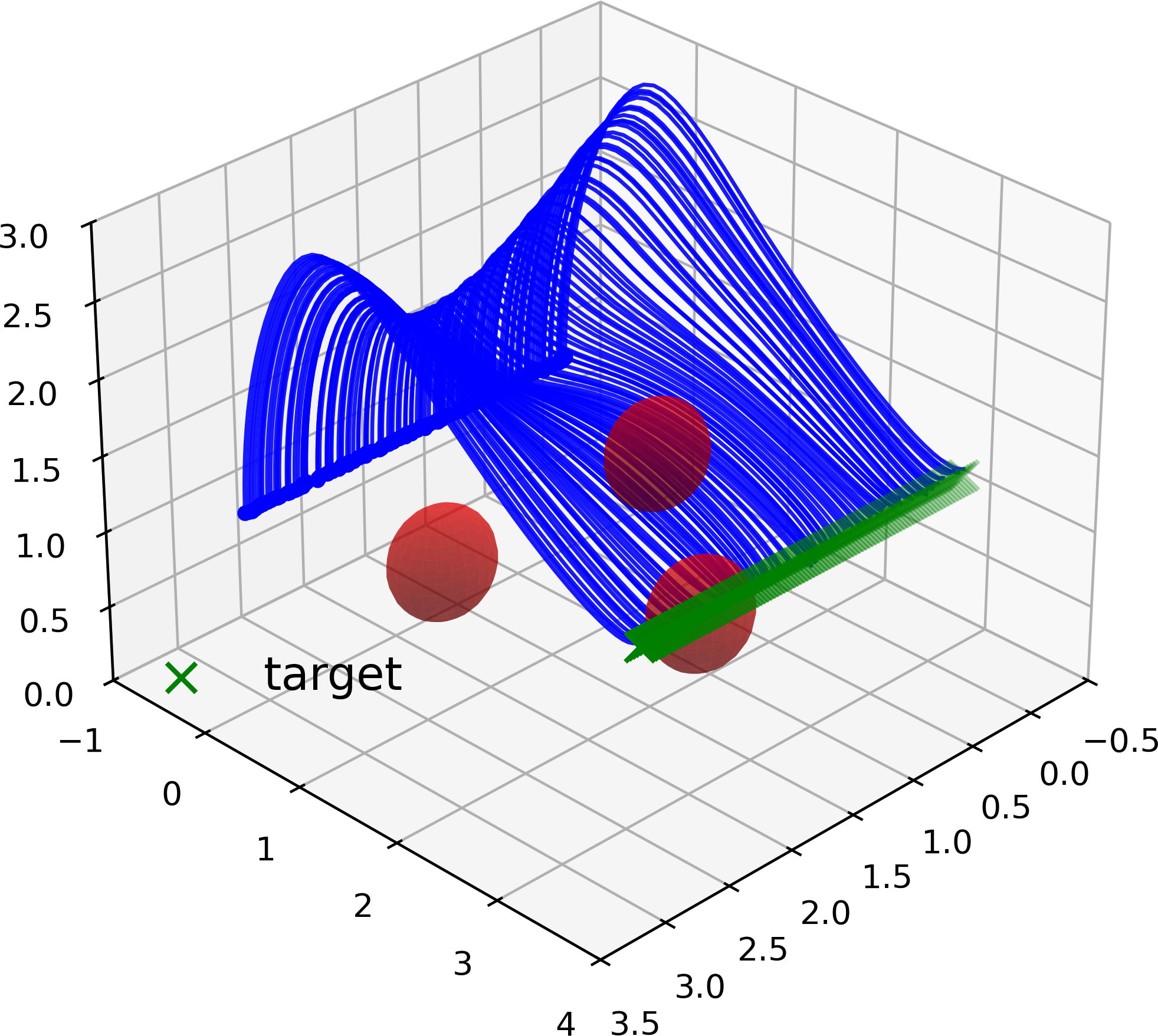}
    \end{subfigure}    
    \caption{\small{High-dimensional control examples: trajectories of 50 single integrators (left) and 100 quadcopters (right) generated by trained \emph{semi-global feedback} controllers.}}
    \label{fig:100_quadrotor_eyecandy}
\end{figure}
A widely used framework for enforcing safety constraints is given by control barrier functions (CBFs)~\cite{ames2016control,ames2019control,xiao2021high}. CBFs guarantee forward invariance of the safe set by imposing a pointwise inequality on the control input. In practice, this condition is typically enforced through a quadratic program (QP) solved at each time step, which minimally modifies a nominal control input to ensure safety.

For neural feedback policies, this naturally leads to a QP-based safety filter embedded as an optimization layer within the training loop. This end-to-end formulation is attractive because \emph{the learned controller can account for the downstream safety filter during training, rather than treating safety only as a post-processing step}. However, differentiating through QP layers introduces substantial computational overhead, and existing end-to-end approaches have largely been restricted to low-dimensional systems, typically on the order of tens of state variables. Scaling such methods to high-dimensional nonlinear control problems therefore remains a major computational challenge.

In this work, we address this challenge by developing a scalable end-to-end training framework for semi-global neural feedback controllers with CBF-QP filters. By combining operator splitting with Jacobian-Free Backpropagation (JFB), our approach avoids the differentiation bottlenecks that arise when training through high-dimensional QP layers. This enables end-to-end learning for controlling safety-critical nonlinear systems at scale far beyond those considered in prior work, while retaining the hard safety guarantees provided by the embedded CBF filter (see Figure~\ref{fig:100_quadrotor_eyecandy}).

\subsection{Our Contributions}

Several lines of work address constraints in learning-to-control problems. Differentiable programming approaches for optimal control, including differentiable MPC~\cite{amos2018differentiable}, Pontryagin Differentiable Programming~\cite{jin2020pontryagin}, and their extensions~\cite{oshin2023differentiable,jin2021safe}, enable end-to-end differentiation through trajectory optimization but are fundamentally local trajectory solvers rather than approaches for learning feedback policies.
Work such as~\cite{onken2021neural,drgovna2022differentiable,cortez2022differentiable,yoo2022dynamic,mowlavi2023optimal,drgovna2024learning} incorporates safety constraints through soft penalties, which simplify optimization but cannot guarantee constraint satisfaction.
More closely related are projection-based safety layers built on differentiable optimization layers such as OptNet~\cite{amos2017optnet} and CVXPY Layers~\cite{agrawal2019differentiable}. While these methods enable end-to-end training with hard safety guarantees~\cite{chen2021enforcing,xiao2023barriernet,min2024hardnet}, they have largely been limited to relatively small-scale problems due to the computational cost and optimization challenges associated with differentiating through optimization layers. 
Due to space limitations, we defer a comprehensive review of related work to the Appendix.

We propose a new framework for end-to-end learning of safe semi-global feedback controllers that overcomes the computational and theoretical limitations of prior approaches. Our main contributions are as follows.

\begin{enumerate}

\item \textbf{Operator-splitting and Jacobian-free backpropagation for scalable end-to-end training.}
We develop a scalable end-to-end training framework for neural feedback controllers with embedded CBF-based safety layers by combining three-operator splitting with Jacobian-Free Backpropagation.
The operator-splitting formulation decomposes the CBF induced QPs into simple subproblems, rendering high-dimensional scaling possible, while JFB enables efficient differentiation through the resulting fixed-point iterations without solving large linear systems.
Together, these ingredients make end-to-end training through high-dimensional safety filters computationally tractable by avoiding the formulation of the full KKT system, and reducing the cost of backpropagation significantly compared to implicit differentiation method.

\item \textbf{Convergence guarantees under nonsmooth safety layers.}
Leveraging the theory of Clarke generalized Jacobians, we establish convergence guarantees under gradient flow for the proposed training procedure despite nonsmoothness introduced by the CBF layers. This analysis provides a rigorous foundation for training safety-constrained implicit networks with JFB, even when active constraints induce discontinuities in the gradient.

\item \textbf{Numerical validation across dynamics, constraints, and scale.}
We demonstrate the effectiveness of the proposed framework on a range of high-dimensional nonlinear multi-agent control problems with embedded CBF safety layers. The experiments span different dynamics, safety constraints, and problem sizes, including settings with state dimensions exceeding $1000$ and control dimensions exceeding $100$. These results show that the proposed approach scales well beyond the dimensions considered in prior end-to-end safe control learning methods.

\end{enumerate}

Together, these results demonstrate that high-dimensional safety-critical feedback control can be learned end-to-end in regimes previously considered computationally prohibitive. 

\section{Background}

\subsection{Control Barrier Functions}

Consider the control-affine system~\eqref{eq:sys_dyn}, where $f$ and $g$ are locally Lipschitz. 
We drop the subscript in $z_x$ for ease of presentation. Let the safe set be defined by
\begin{equation}
\mathcal C := \{ z \in \mathbb{R}^n : h(z) \ge 0 \},
\end{equation}
where $h : \mathbb{R}^n \to \mathbb{R}$ is continuously differentiable. Control barrier functions (CBFs) enforce forward invariance of $\mathcal C$ by imposing a pointwise constraint on the control input.

For a control-affine system, the time derivative of $h$ along trajectories is
\begin{equation}
\dot h(z)
=
\nabla h(z)^\top f(t,z)
+
\nabla h(z)^\top g(t,z)u .
\end{equation}
A standard CBF condition requires the existence of an extended class-$\mathcal K$ function $\alpha$, i.e., a continuous, strictly increasing function satisfying $\alpha(0)=0$, such that
\begin{equation}
\dot h(z) + \alpha(h(z)) \ge 0,
\label{eq:cbf_condition}
\end{equation}
Importantly,~\eqref{eq:cbf_condition} is affine in the control action $u$.

\begin{thm}[Forward invariance under CBF constraints~\cite{ames2019control}]
Let function $h$ be continuously differentiable and let $\mathcal C$ be defined as above. 
Suppose $u(t)$ is locally Lipschitz, $z(0) \in \mathcal C$, and~\eqref{eq:cbf_condition} holds 
for all $t \ge 0$. Then $\mathcal C$ is forward invariant under $u(t)$. In particular,
$z(t) \in \mathcal{C}$ for all $t \ge 0$, as such the safety condition
$h(z(t)) \ge 0$ is maintained.
\end{thm}

In words, the CBF condition ensures that trajectories cannot leave the safe set by requiring the derivative of the safety function to be sufficiently nonnegative near the boundary of $\mathcal C$. Thus, safety can be enforced by restricting the control input to satisfy the constraint set consisting of the control-affine inequality induced by~\eqref{eq:cbf_condition}:
\begin{equation}
C(z) := \left\{ u \in \mathbb{R}^m : A(z)u \le b(z) \right\}.
\label{eq:Axbconstraint}
\end{equation}
During training, the safety filter maps the nominal neural feedback control $u_\theta(t,z)$ to its projection onto this feasible set:
\begin{equation}
u_\theta^\star(t,z)
=
P_{C(z)}(u_\theta(t,z))
=
\argmin_{u \in C(z)}
\frac12\norm{u-u_\theta(t,z)}^2 .
\label{eq:projection}
\end{equation}
Thus, $u^\star_\theta(t,z)$ is obtained by solving a CBF-based quadratic program (CBF-QP) computed at every time step, and for every trajectory that ensures that~\eqref{eq:cbf_condition} is satisfied. 
In this work we also employ High Order Control Barrier Functions (HOCBFs), the details of which we include in the Appendix.

\subsection{Davis--Yin Splitting and Jacobian-Free Backpropagation}
\label{subsec:DYS}

The CBF-QP safety filter in~\eqref{eq:projection} must be evaluated and differentiated through at every time step of every sampled trajectory during training. Although each projection is often inexpensive at each step, stacking them over time can add to significant increase in computation, more importantly, backward differentiating through the QP can become a major computational bottleneck in high-dimensional problems.
\subsubsection{Davis--Yin Splitting}
To make this step scalable, we follow~\cite{mckenzie2023differentiating} and solve the CBF-QP using Davis--Yin splitting (DYS)~\cite{davis2017three}, which represents the projection through the fixed point of a splitting operator. We then use Jacobian-Free Backpropagation (JFB)~\cite{fung2022jfb, yin2022learning, knutson2026logical} to differentiate through this fixed-point representation without backpropagating through all solver iterations or solving the linear systems required by classical implicit differentiation~\cite{bai2019deep,lorraine2020optimizing,el2021implicit, fung2026generalization}. This allows the CBF-QP layer to be embedded efficiently within end-to-end feedback-control training. 

In particular, we first introduce slack variables $s \in \mathbb{R}^c$ for the $c$ inequality constraints and let $y=(u,s) \in \mathbb{R}^{m+c}$. Then the inequality~\eqref{eq:Axbconstraint} can be written as
\begin{equation}
A(z)u+s=b(z), \qquad s \ge 0.
\label{eq:slack_eq}
\end{equation}
This gives the lifted feasible set
\begin{equation}
\widetilde C(z) := C_1 \cap C_2(z),
\end{equation}
where
\begin{equation}
\begin{aligned}
C_1 &:= \left\{(u,s): s \ge 0 \right\}, \\
C_2(z) &:= \left\{(u,s): A(z)u+s=b(z) \right\}.
\end{aligned}
\end{equation}
Both projections are simple: $P_{C_1}$ is obtained by thresholding the slack variables, while $P_{C_2(z)}$ is projection onto an affine subspace.
Let
\begin{equation}\label{eq:Q_def}
Q_u = [I , 0] \in \mathbb{R}^{m \times (m+c)}
\end{equation}
denote the matrix that extracts the control component from $y=(u,s)$. The projection problem~\eqref{eq:projection} is equivalently written in lifted variables as
\begin{equation}
y_\theta^\star
=
\argmin_{y \in C_1 \cap C_2(z)}
\frac12 \norm{Q_u y-u_\theta(t,z)}^2,
\ \ 
u_\theta^\star(t,z)=Q_u y_\theta^\star.
\end{equation}
For a step size $\zeta>0$, DYS defines the fixed-point operator
\begin{equation}
\begin{aligned}
&T_\theta(y;z)
=
y-P_{C_1}(y)  \\
&+
P_{C_2(z)}
\left(
2P_{C_1}(y)-y
-\zeta Q_u^\top
\big(
Q_uP_{C_1}(y)-u_\theta(t,z)
\big)
\right)
\end{aligned}
\label{eq:DYS_operator}
\end{equation}
where $u_\theta(t,z)$ is the nominal control. The safety-filtered control is recovered from a fixed point
\begin{equation}
y^\star = T_\theta(y^\star;z)
\label{eq:fixed_point}
\end{equation}
by applying a final $C_1$ projection
\begin{equation}
u_\theta^\star(t,z)=Q_uP_{C_1}(y^\star)=Q_u y^\star.
\label{eq:filtered_control_DYS}
\end{equation}
The final equality holds because $P_{C_1}$ only modifies the slack variables and leaves the control component unchanged.

The principal advantage of DYS lies in its ability to decompose a complex optimization problem into inexpensive projection and gradient evaluations~\cite{ryu2022large}. By avoiding the repeated solution of large coupled KKT systems, DYS is particularly attractive for large-scale optimization and scales naturally to multi-agent problems with hundreds or even thousands of state variables.

\subsubsection{Jacobian-Free Backpropagation}

The DYS formulation provides an efficient fixed-point method for computing the CBF-QP projection. However, differentiating through this fixed point remains costly as one must either backpropagate through all DYS iterations or use implicit differentiation, which requires solving a large linear system accurately at each time step. Indeed, differentiating~\eqref{eq:fixed_point} with respect to $\theta$ gives
\begin{equation}
\frac{d y_\theta^\star}{d\theta}
=
\frac{\partial T_\theta}{\partial y}
\frac{d y_\theta^\star}{d\theta}
+
\frac{\partial T_\theta}{\partial \theta},
\end{equation}
where the partial derivatives are evaluated at $(y_\theta^\star;z)$. Rearranging the terms yields
\begin{equation}
\frac{d y_\theta^\star}{d\theta}
=
\left(I -
\frac{\partial T_\theta}{\partial y}
\right)^{-1}
\frac{\partial T_\theta}{\partial \theta}.
\label{eq:implicit_diff_formula}
\end{equation}
In high-dimensional problems, evaluating this expression can dominate the cost of training, as this must be done \emph{at each time step and for each trajectory with high accuracy}.

To avoid this bottleneck, we use Jacobian-Free Backpropagation (JFB)~\cite{fung2022jfb, heaton2021feasibility, heaton2023explainable}, which replaces the inverse term in~\eqref{eq:implicit_diff_formula} by an identity:
\begin{equation}
\frac{d y_\theta^\star}{d\theta}
\approx
\frac{\partial T_\theta}{\partial \theta}.
\label{eq:jfb}
\end{equation} 
Thus, JFB avoids both solver unrolling and implicit linear solves, reducing the per-time-step gradient cost for each trajectory from $\mathcal{O}((m+c)^3 + (m+c)^2p)$ to $\mathcal{O}((m+c)p)$ where $p$ is the total network parameter count. The cost of factorizing a dense matrix and performing the subsequent matrix multiplication is incurred at every time step and scales linearly with both the batch size and the trajectory length. In contrast, JFB requires only gradients with respect to the trainable parameters, making the computational savings substantial.
Notably, while~\cite{mckenzie2023differentiating} establishes descent for a single QP layer, our setting embeds CBF-QPs throughout semi-global optimal-control trajectories; our analysis accounts for this repeated, trajectory-level composition.

\section{End-to-End Training with Embedded CBF-QP Safety Filters Along Trajectories}

We now describe the full semi-global optimal-control training problem with embedded CBF-QP safety filters. The training objective is
\begin{subequations}
\begin{align}
\min_\theta \; \mathbb{E}_{x\sim\rho} \; J_x(\theta)
&:=
\int_0^T L(t,z_x,u^\star_\theta)\,dt + \omega G(z_x(T)), \label{eq:training_problem1}\\
\text{s.t.}\quad
\dot z_x = f(t,z_x)& + g(t,z_x)u^\star_\theta, \quad z_x(0)=x, \label{eq:training_problem2}\\
u^\star_\theta &= Q_u y_\theta^\star, \quad
y_\theta^\star = T_\theta(y_\theta^\star; z_x), \label{eq:training_problem3}
\end{align}
\end{subequations}
where $T_\theta$ is the DYS operator defined in~\eqref{eq:DYS_operator}. The fixed-point equation in~\eqref{eq:training_problem3} computes the CBF-QP projection at each state along the trajectory, which ensures that the filtered control satisfies the CBF constraints during both training and deployment. As previously mentioned, the semi-global nature of the formulation comes from optimizing over a distribution of initial conditions $x\sim\rho$, rather than a single trajectory~\cite{onken2021neural, vidal2023taming, Verma-HJB-RL, meng2026recent}. This differs from trajectory-specific differentiable optimal-control approaches~\cite{jin2020pontryagin,jin2021safe} and is essential for learning feedback controllers that generalize across initial states. The main computational challenge is that each sampled trajectory requires many CBF-QP solves, each of which must be differentiated through during training. Alg.~\ref{alg:dys_training} summarizes the resulting DYS-JFB training procedure, Despite the computational advantage JFB can be implemented trivially with minimal code changes.
The remainder of this section analyzes the convergence of this trajectory-level training scheme, which accounts for both the embedded CBF-QP projections and the JFB gradient approximation. To the best of our knowledge this is the first convergence proof for policy training with embedded nonsmooth projection layers. 

\begin{algorithm}[t]
\caption{End-to-End Training with DYS-JFB Scheme}
\label{alg:dys_training}
\begin{algorithmic}[1]
\State \textbf{Input:} learning rate $\eta$, horizon $T$, step size $\Delta t$, total time steps $N_t$, DYS operator $T_\theta$, batch size $B$, initial state distribution $\rho$, policy network $u_\theta$
\State \textbf{Initialize} policy parameters $\theta$
\While{not converged}
    \State Sample $\{x_i\}_{i=1}^B \sim \rho$
    \State $J \gets 0$
    \For{$i = 1,\dots,B$}
        \State $z_0^{(i)} \gets x_i$
        \For{$k = 0,\dots,N_t-1$}
            \State $t_k \gets k\Delta t$
            \State $u_\theta^{\mathrm{nom}} \gets u_\theta(t_k,z_k^{(i)})$
            \State Initialize $y$
            \State \texttt{stop-gradient}
            \While{DYS not converged}
                \State $y \gets T_\theta(y;z_k^{(i)},u_\theta^{\mathrm{nom}})$
            \EndWhile
            \State $y \gets T_\theta(y;z_k^{(i)},u_\theta^{\mathrm{nom}})$ \Comment{JFB update \& enable-gradient}
            \State $u_k^\star \gets Q_u y$
            \State $z_{k+1}^{(i)} \gets z_k^{(i)} + \Delta t \big(f(t_k,z_k^{(i)}) + g(t_k,z_k^{(i)})u_k^\star\big)$
            \State $J \gets J + \Delta t \, L(t_k,z_k^{(i)},u_k^\star)$
        \EndFor
        \State $J \gets J + \omega G(z_{N_t}^{(i)})$
    \EndFor
    \State $J \gets J/B, \quad \theta \gets \theta - \eta \nabla_\theta J$
\EndWhile
\State \textbf{Output:} trained policy network $u_\theta (t, z)$
\end{algorithmic}
\end{algorithm}

\subsection{Convergence Analysis}
The convergence analysis is complicated as a result of (i) the nonsmoothness of the DYS operator and (ii) the biased stochastic gradients induced by JFB. Importantly we note that the safety layer involves a projection onto the nonnegative slack variables, so $T_\theta$ need not be differentiable everywhere. We therefore use Clarke generalized Jacobians to handle the nonsmooth dependence of the fixed-point map.

\begin{definition}[Clarke Generalized Jacobian~\cite{clarke1990optimization}]
\label{def:Clarke}
Let $F : \mathbb{R}^n \to \mathbb{R}^m$ be locally Lipschitz. The Clarke generalized Jacobian of $F$ at $x$ is
\begin{equation}
\begin{split}
\partial^C F(x)
=
\mathrm{conv}
\Big\{
\lim_{k\to\infty} JF(x_k)
\; \Big| \;
x_k \to x,\;
\\
F \text{ is differentiable at } x_k
\Big\},
\end{split}
\end{equation}
where $JF(x_k)$ denotes the classical Jacobian and $\mathrm{conv}(\cdot)$ denotes the convex hull. Any matrix $M \in \partial^C F(x)$ is called a generalized Jacobian of $F$ at $x$. When $m=1$, $\partial^C F(x)$ is the Clarke subdifferential, and any $v \in \partial^C F(x)$ is called a Clarke subgradient.
\end{definition}

\paragraph{JFB approximation of Generalized Jacobian} When $T_{\theta}$ is not differentiable with respect to $y$ at certain points, the implicit differentiation \eqref{eq:implicit_diff_formula} can be realized in the Clarke sense. Specifically, by the Clarke chain rule~\cite[Corollary~2.6.6]{clarke1990optimization},
\begin{equation}
\partial_\theta y_\theta^\star
 \in \left\{
\left(I -
M_y
\right)^{-1}
\frac{\partial T_{\theta}}{\partial \theta} \middle| M_y \in {\partial}_y^C T_{\theta} \right\} =: \partial_\theta^C y_\theta^\star .
\label{eq:clarke_restricted_implicit_diff_formula}
\end{equation}

To avoid the computational bottleneck of matrix inversion in~\eqref{eq:clarke_restricted_implicit_diff_formula}, we employ JFB and use the approximate Jacobian defined in~\eqref{eq:jfb}.
Interestingly, by using this approximation, we get rid of the generalized Jacobian. 

Consider the trajectory cost $J_x(\theta)$ defined in
\eqref{eq:training_problem1}--\eqref{eq:training_problem3}. The Clarke subdifferential of the trajectory cost at $x$ and $t$ is given by
\begin{align}\label{eq:integrand_definitions}
\begin{split}
& \left\{ \left({\partial_\theta {u_\theta^\star}(t,z_x(t))}\right)^\top
h_{\theta,x}(t) \middle| {\partial_\theta {u_\theta^\star}(t,z_x(t))} \in {\partial^C_\theta {u_\theta^\star}(t,z_x(t))} \right\} \\
&= \left\{ \left({\partial_\theta {y_\theta^\star}(t,z_x(t))}\right)^\top
Q_u^\top
h_{\theta,x}(t) \middle| {\partial_\theta {y_\theta^\star}(t,z_x(t))} \in {\partial^C_\theta {y_\theta^\star}(t,z_x(t))} \right\} \\
&= \left\{ \left(\left(I -
M_y
\right)^{-1}
\frac{\partial T_{\theta}}{\partial \theta}\right)^\top
Q_u^\top h_{\theta,x}(t) \middle| M_y \in { {\partial}_y^C} T_{\theta}  \right\} \\
&=: \partial^C_\theta J_{x, t}(\theta).
\end{split}
\end{align}
This expression follows from the Clarke chain rule~\cite[Theorem~2.6.6]{clarke1990optimization}. The first equality is a direct consequence of the linear relationship between $u_\theta^\star$ and $y_\theta^\star$ established in~\eqref{eq:training_problem3}. The second equality uses~\eqref{eq:clarke_restricted_implicit_diff_formula}. Moreover, \begin{align*}
h_{\theta,x}(t)
&=
\nabla_u L\!\big(t,z_x(t),u^\star_\theta(t,z_x(t))\big) +
g(t, z_x(t))^\top p_x(t)
\end{align*}
with $g$ defined in~\eqref{eq:original_problem}. $h_{\theta,x}(t)$
is the gradient of the Hamiltonian with respect to the control input,
where $p_x(t)$ is the adjoint variable satisfying the adjoint
equation~\cite{evans1983introduction}.

In contrast to the set-valued Clarke subdifferential, the corresponding JFB approximation is 
\begin{align}\label{eq:integrand_definitions_JFB}
\begin{split}
w_{\theta,x}(t)
&=
\left(\frac{\partial T_\theta}{\partial\theta}(y_\theta^\star(t,z_x(t));z_x(t))\right)^\top
Q_u^\top
h_{\theta,x}(t).
\end{split}
\end{align}
Note that the approximation~\eqref{eq:integrand_definitions_JFB} not only avoids the expensive matrix inversion in the true Clarke subdifferential~\eqref{eq:integrand_definitions}, which is required at each discretized time step and each sample trajectory during training, but also replaces the set-valued Clarke subdifferential with a single-valued function.
Thus, for each $x$, the JFB approximated gradient of the trajectory cost  $J_x(\theta)$ is given by
\begin{align}
d^{\mathrm{JFB}}_x(\theta)
= \int_0^T w_{\theta,x}(t)\, dt.
\label{eq:full_derivatives}
\end{align}

\subsection{Assumptions} 
We first state the main assumptions, which are largely standard in control and optimization theory.
\begin{assumption}[Well-conditioned CBF function]\label{assm:fullrank_A} 
     The matrix $A(z) \in \mathbb{R}^{c \times m}$ has full row rank for all $z$.
    \end{assumption}

\begin{assumption}[Compactness]\label{assm:compactness} 
    The domains of $x$, $z$ and $\theta$ are compact. 
\end{assumption}
Assumptions~\ref{assm:fullrank_A}--\ref{assm:compactness} are essential in establishing~\Cref{thm:contraction} and \Cref{cor:contraction}: the DYS operator is contractive with respect to $y$, hence the fixed point algorithm always converges.

\begin{assumption}[Smoothness] The following holds.
        \begin{enumerate}
            \item\label{assm:smooth} $T_{\theta}$ is $C^1$ with respect to $(\theta, z)$ and Lipschitz continuous in $(y, z)$ uniformly in $\theta$. Furthermore, $F(\theta)=\mathbb{E}_x[J_x(\theta)]$ 
            is Lipschitz continuous with respect to $\theta$ and is bounded from below.
            \item\label{assm:bounded}  ${d^{\mathrm{JFB}}_x(\theta)}$ is uniformly bounded over $x$ and $\theta$.
            \item\label{assm:continuous}  $u_\theta^\star(t, z_x)$ is continuous with respect to $t$ and $x$ for all $\theta$. 
        \item\label{assm:lipschitz} there exists an integrable function $k(t, x)$ with $\mathbb{E}_x [\int_0^T k(t, x) dt] < \infty$ such that $\forall \, x$ and $t$, $u_\theta^\star(t, z_x)$ is $k(t, x)$-Lipschitz continuous with respect to $\theta$. 
        \end{enumerate}  
        \label{assumption:T}
    \end{assumption}
    Here, Assumption~\ref{assumption:T} generally holds for the optimal control problems~\eqref{eq:training_problem1}--\eqref{eq:training_problem3} we consider.
    Assumptions~\ref{assumption:T}.\ref{assm:smooth}--\ref{assumption:T}.\ref{assm:bounded} are standard regularity conditions for proving convergence of the training algorithm; see also~\cite[Assumption~4.1]{gelphman2026convergence}. Assumptions~\ref{assumption:T}.\ref{assm:continuous}--\ref{assumption:T}.\ref{assm:lipschitz} let us differentiate the objective~\eqref{eq:training_problem1} through its integrals over time $t$ and initial condition $x$. Roughly speaking, by~\cite[Theorem~2.7.2]{clarke1990optimization}, they let us move $\partial_\theta$ inside the integrals. This is an essential property to establish convergence; see the proof of~\Cref{thm:expect_suff_descent} for more details.

\begin{assumption}[Conditioning and coupling of the JFB integrand matrices]\label{assumption:M} Let $M^u_\theta \coloneqq Q_u \frac{\partial T_{\theta}}{\partial \theta}(y_\theta^\star(t, z); z)
\in \mathbb{R}^{m \times p}$ denote the Jacobian
corresponding to the control $u$.
We assume that for all $\theta, t, z$:
\begin{enumerate}
    \item[(i)] \emph{(rank)} $M^{u}_{\theta}  \in \mathbb{R}^{m \times p}$ has full row rank; write $G_\theta \coloneqq M^{u}_{\theta} (M^{u}_{\theta})^\top \in \mathbb{R}^{m \times m}$, $G_\theta \succ 0$, and let $\sigma_{+}^2 \ge \sigma_{-}^2 > 0$ denote its largest and smallest eigenvalues;
    \item[(ii)] \emph{(conditioning)} $\exists \gamma_\mathrm{eff}>0$, s.t. $\displaystyle \kappa(G_\theta) = \frac{\sigma_{+}^2}{\sigma_{-}^2} < \frac{1}{\gamma_{\mathrm{eff}}}$;
    \item[(iii)] \emph{(weak coupling, informal)} the contribution of the slack $s$ to the alignment $\norm{\Xi_\theta}$ between the true and JFB gradients, is uniformly bounded by an explicit threshold $\rho_\theta$.
\end{enumerate}
\end{assumption}

    Assumption~\ref{assumption:M} is modified from standard assumptions in the JFB literature~\citep{fung2022jfb,gelphman2026convergence}, which typically requires the Jacobian $\frac{\partial T_{\theta}}{\partial \theta}(y^\star; z)$ in the JFB approximation~\eqref{eq:integrand_definitions_JFB} to be sufficiently well-conditioned. Part (iii) is stated informally here for brevity; the precise condition is given in the Appendix.
    
    \begin{assumption}[Variance Bound]
        \label{assumption:expectation_integrand_inner_product}
        $\forall \theta, t, x$, and measurable selection $v_{\theta, x}(t) \in \partial_\theta^C J_{x, t}(\theta)$, $\exists 0 < \delta_{var} < \frac{\rho_\theta - \norm{\Xi_\theta}}{1-\gamma}$ such that 
        \begin{equation}
        \max(\sqrt{\text{Var}_x[v_{\theta,x}(t)]}, \sqrt{\text{Var}_x[w_{\theta,x}(t)]})^2 \leq \delta_{var} \norm{ \mathbb{E}_x \left[ h_{\theta,x}\right]}^2,
        \label{eq:variance_vw}
        \end{equation}
        where $w_{\theta,x}$ is defined in~\ref{eq:integrand_definitions_JFB}, $\rho_\theta$ is defined in \eqref{eq:Xi_bound} 
        in the Appendix, $\gamma$ is the contraction constant in Corollary~\ref{cor:contraction}.
    \end{assumption}
    Assumption~\ref{assumption:expectation_integrand_inner_product}  is needed to establish the descent condition in~\Cref{lemma:expected_inner_product}: the mean subgradient $\mathbb{E}_x[v_{\theta, x}]$ and the mean JFB integrand $\mathbb{E}_x[w_{\theta, x}]$ have nonnegative inner product. It does so by controlling the variance of the sample-wise integrands $v_{\theta, x}$ and $w_{\theta, x}$. This assumption is similar to variance bounds for standard convergence guarantees of SGD; see, e.g.,~\cite{bottou2018optimization}.

    \begin{assumption}
    For all $\theta, x,$ $t$, and measurable selection $v_{\theta, x}(t) \in \partial_\theta^C J_{x, t}(\theta)$, the following hold.
    \begin{enumerate}
        \item Each element in the vectors $v_{\theta,x}, w_{\theta,x}$ is integrable on $[0,T]$ with respect to $t$. Moreover, $v_{\theta,x},w_{\theta,x}$ are integrable on $[0,T] \times \Omega$, where $\Omega$ is the sample space of the distribution of initial conditions, $\rho$.
        \item $\exists  \ \delta_v, \delta_w, a_v, a_w \geq 0 \text{ and }\epsilon_v > 0$ such that
        \begin{align*}
        &\|\mathbb{E}_x[v_{\theta,x}(t) - C_v]\|_2 \leq 
        a_v +\delta_v \inf_{\phi \in \partial^C_\theta \mathbb{E}_x [J_x(\theta)]} \norm{\phi} , \\
        &\|\mathbb{E}_x[w_{\theta,x}(t) - C_w]\|_2 \leq a_w + \delta_w \left\| \mathbb{E}_x[\dJFB]\right\|_2,
        \end{align*}
        and
        \begin{align*}
        &\max(a_v + \delta_v \inf_{\phi \in \partial^C_\theta \mathbb{E}_x [J_x(\theta)]} \norm{\phi} , a_w + \delta_w \left\| \mathbb{E}_x[\dJFB]\right\|_2)^2 \\
        & \leq \delta_{v, \theta}^{2} - \frac{\epsilon_v}{T^2} \inf_{\phi \in \partial^C_\theta \mathbb{E}_x [J_x(\theta)]} \norm{\phi}^2, 
        \end{align*}
        where $w_{\theta,x}$ is defined in~\eqref{eq:integrand_definitions_JFB}, $C_v = \frac{1}{T} \int_0^T v_{\theta,x}(t) dt$, $C_w = \frac{1}{T} \int_0^T w_{\theta,x}(t) dt$, and $\delta_{v, \theta} := \sqrt{\frac{\rho_\theta - \norm{\Xi_\theta}}{1-\gamma} - \delta_{var}} \norm{ \mathbb{E}_x \left[ h_{\theta,x}\right]}$.
    \end{enumerate}
\label{assumption:expectation_timeaverage}
\end{assumption}
Assumption~\ref{assumption:expectation_timeaverage} controls the deviation of the sample-wise integrands $v_{\theta, x}$ and $w_{\theta, x}$ from their time averages and is a standard assumption for optimal control problems; see, e.g.,~\cite{gelphman2026end} and can be numerically supported in Figure~\ref{fig:theory_validation} in the Appendix. It establishes \Cref{thm:expect_suff_descent}: the JFB update $\mathbb{E}_x[d_x^{JFB}(\theta)]$ is a descent direction.

\begin{table*}[t]
\centering
\small
\begin{tabular}{l c c c c c}
\toprule
\makecell{Problem \\ (\# agents, \# obstacles)} & 
\makecell{dims \\ $(n,m)$} &
\makecell{DYS-JFB \\ \textbf{(Ours)}} & 
DYS-AD & 
\makecell{CVXPYLayers \\ (IFT)} & 
\makecell{CVXPYLayers \\ (IFT + reg.)} \\
\midrule
\makecell{Single Integrator \\(50, 2) } & $(150,150)$ & (82.28, 7.49e-1) & (117.70, 1.50e1) & (75.87, 1.28e-1) & (76.28, 1.24e-1) \\
\makecell{Double Integrator\\(1,3)} &
(4,2) &
(0.14, 4.91e-3) &
(0.36, 1.01e0) &
(0.23, 1.25e-3) &
(0.33, 6.33e-4) \\
\makecell{Double Integrator \\ (6,3)}  & $(24,12)$     & (5.19, 1.32e-2) & (5.92, 1.06e-1) & (5.46, 6.16e-4) & (5.39, 5.82e-4) \\
\makecell{Quadcopter \\(5,3)}        & $(60,20)$     & (23.14, 1.23e-1) & (23.75, 6.57e-2) & -- & (25.40, 3.15e-2) \\
\makecell{Quadcopter \\(30,3)}       & $(360,120)$   & (172.07, 1.82e-1) & (173.60, 8.88e-1) & -- & -- \\
\makecell{Quadcopter \\(100,3)}       & $(1200,400)$   & (795.33, 1.66e0) & -- & -- & -- \\
\bottomrule
\end{tabular}
\caption{
Results across system configurations. $(n,m)$ denotes the state and control dimensions. We report both the average running cost $L$ and terminal cost $G$ over random initial states (both lower the better). 
``--'' indicates training divergence due to memory/runtime blow-up or numerical instability.
\textbf{Our method is the only one which successfully solves all the problems.}
}
\label{tab:results_obj}
\end{table*}

\subsection{Convergence Results}
For brevity, we present only the main results here and defer all proofs to the Appendix. The analysis proceeds in three
steps: (i) contraction of the DYS operator, (ii) descent direction, and (iii) convergence.

\begin{restatable}[Contraction of $T_\theta$ with respect to $y$]{cor}{corcontraction}\label{cor:contraction}
Under Assumptions~\ref{assm:fullrank_A} and~\ref{assm:compactness},
and if the DYS hyperparameter $\zeta \in (0, 1)$, then there exist $\gamma \in (0,1)$ such that for any $M_y \in \partial^C_y T_{\theta}(y; z)$,
\begin{equation}
        \norm{M_y} \leq \gamma,
    \end{equation}
for all $\theta$, $y$ and $z$.
\end{restatable}
Corollary~\ref{cor:contraction} guarantees that the DYS
operator~\eqref{eq:DYS_operator} is contractive in $y$, uniformly in
$\theta$ and $z$, which establishes the descent conditions that follow. 
Crucially, \textbf{the bound holds for every generalized Jacobians} $M_y \in \partial^C_y T_\theta(y;z)$, extending the theory in the literature. In \cite{gelphman2026convergence}, this condition is only assumed. In \cite{mckenzie2023differentiating}, this condition is proved only for a single selection of generalized Jacobian. We prove it  for all via new technical tool detailed in the Proof of Contraction in the Appendix.

\begin{restatable}{lem}{lemexpected}\label{lemma:expected_inner_product}
    Under Assumptions~\ref{assm:fullrank_A}--\ref{assumption:expectation_timeaverage}, for any selection of Clarke subgradient $v_{\theta, x}(t) \in \partial^C_\theta J_{x, t}(\theta)$ that is measurable with respect to $x$ and $t$,
    \begin{equation}\label{eq:expect_descent}
    \inner{\mathbb{E}_x[v_{\theta,x}]}{\mathbb{E}_x[w_{\theta,x}]} \geq \delta_{v, \theta}^2 \geq 0, \quad \forall  t, \theta.
    \end{equation}
\end{restatable}

\begin{restatable}[Descent Direction]{thm}{thmdescent}\label{thm:expect_suff_descent}
    Under Assumptions~\ref{assm:fullrank_A}--\ref{assumption:expectation_timeaverage}, we have, for any Clarke subgradient $\xi \in \partial_\theta^C \mathbb{E}_x [ J_x(\theta)]$
    \begin{equation}
    \langle \xi, \mathbb{E}_x[d_x^{JFB}(\theta)] \rangle \geq \epsilon_v \inf_{\phi \in \partial^C_\theta \mathbb{E}_x [J_x(\theta)]} \norm{\phi}^2,
\end{equation}
where $\epsilon_v$ is defined in Assumption~\ref{assumption:expectation_timeaverage} and $\partial_\theta^C \mathbb{E}_x [ J_x(\theta)]$ is set of Clarke subdifferential of $\mathbb{E}_x[J_x(\theta)]$.
\end{restatable}
Lemma~\ref{lemma:expected_inner_product} and Theorem~\ref{thm:expect_suff_descent} show that the JFB update is a descent direction for all Clarke subgradients $\xi \in \partial_\theta^C \mathbb{E}_x [ J_x(\theta)]$. \textbf{The guarantee is uniform over the entire subdifferential} rather than a single element and \textbf{holds despite the bias of the JFB update} (also verified numerically in Figure~\ref{fig:theory_validation} in the Appendix). Our derivation achieves both by handling the nondifferentiability introduced by the DYS operator; see the Proof of Descent Direction in the Appendix. 

\begin{restatable}[Convergence]{thm}{thmconvergence}\label{thm:gradient_flow_converge}
    Under Assumptions~\ref{assm:fullrank_A}--\ref{assumption:expectation_timeaverage}, let $\theta(\tau)$ be the trajectory generated by the continuous-time gradient flow dynamics
\begin{equation}\label{eq:gradient_flow}
    \frac{d\theta(\tau)}{d\tau} = - \mathbb{E}_x \left[ d_x^{JFB}(\theta(\tau)) \right], \quad \text{subject to} \quad \theta(0)=\theta_{\rm init}.
\end{equation}
Then, the trajectory asymptotically converges to a Clarke stationary point in the sense that
$$\liminf_{\tau \to \infty} \inf_{\phi \in \partial_\theta^C \mathbb{E}_x  [ J_x(\theta(\tau))]} \left\| \phi \right\|_2 = 0.$$
Here, $\partial_\theta^C \mathbb{E}_x [ J_x(\theta)]$ is set of Clarke subdifferential of $\mathbb{E}_x[J_x(\theta)]$.
\end{restatable}
Theorem~\ref{thm:gradient_flow_converge} shows that the JFB gradient
flow converges to a Clarke stationary point. While policy learning methods with embedded nonsmooth safety layers exist, to the best of our knowledge, this is the first convergence guarantee for such methods.
Nondifferentiability renders traditional convergence proof techniques inapplicable.
Our main technical invention is to carefully analyze the behavior of Clarke subgradients $\hat{\xi} \in \partial_\theta^C \mathbb{E}_x [J_x(\hat{\theta}(\tau))]$ at $\hat{\theta}(\tau)$ in a neighborhood of $\theta(\tau)$. Combined with the descent conditions, this yields the convergence guarantee.

\begin{table*}[t]
\centering
\small
\begin{tabular}{l c l c  c}
\toprule
\makecell{Problem \\ (\# agents)} & 
\makecell{dims \\ $(n,m)$} &
Method & 
\makecell{Peak CPU Mem.\\ (MB)} & 
\makecell{Peak GPU Mem.\\ (MB)} \\
\midrule
\multirow{3}{*}{Single Integrator (50)}
& \multirow{3}{*}{$(150,150)$}
& \textbf{DYS-JFB (Ours)}  & \textbf{1591.4} & \textbf{564.5} \\
& & DYS-AD                 & 4175.0          & 6206.5 \\
& & CVXPYLayers (IFT)      & 7127.1          & 121.9 \\
\midrule
\multirow{3}{*}{Double Integrator (1)}
& \multirow{3}{*}{$(4,2)$}
& \textbf{DYS-JFB (Ours)}  & \textbf{1547.8} & \textbf{23.0} \\
& & DYS-AD                 & 2023.1          & 41.2 \\
& & CVXPYLayers (IFT)      & 1995.6          & 28.9 \\
\midrule
\multirow{3}{*}{Double Integrator (6)}
& \multirow{3}{*}{$(24,12)$}
& \textbf{DYS-JFB (Ours)}   & \textbf{1562.2} & \textbf{36.3} \\
& & DYS-AD                  & 3991.1          & 294.0 \\
& & CVXPYLayers (IFT)       & 3021.0          & 39.6 \\
\midrule
\multirow{3}{*}{Quadcopter (5)}
& \multirow{3}{*}{$(60,20)$}
& \textbf{DYS-JFB (Ours)}   & \textbf{1597.6} & \textbf{38.3} \\
& & DYS-AD                  & 5892.4          & 556.5 \\
& & CVXPYLayers (IFT)       & 2942.7          & 112.8 \\
\midrule
\multirow{3}{*}{Quadcopter (30)}
& \multirow{3}{*}{$(360,120)$}
& \textbf{DYS-JFB (Ours)}   & \textbf{1603.2} & \textbf{235.2} \\
& & DYS-AD                  & 7051.3          & 4027.9 \\
& & CVXPYLayers (IFT)       & 6441.1          & 2239.7 \\
\midrule
\multirow{3}{*}{Quadcopter (100)}
& \multirow{3}{*}{$(1200,400)$}
& \textbf{DYS-JFB (Ours)}   & \textbf{1604.3} & \textbf{1997.6} \\
& & DYS-AD                  & 7837.4         & -- \\
& & CVXPYLayers (IFT)       & 33206.0          & -- \\
\bottomrule
\end{tabular}
\caption{
Training memory usage across system configurations. $(n,m)$ denotes the state and control dimensions. Both CPU and GPU memories correspond to peak usage during training. Lower is better for all metrics. For $100$ quadcopter example, both AD and CVXPY Layers exceed the $16$ GB VRAM allocation cap. 
}
\label{tab:results_mem}
\end{table*}

\section{Numerical Results}

Our numerical experiments consist of a total of six different control tasks. We consider three different dynamical models under varying obstacle configurations. Specifically, we consider a single-integrator model in 3D, commonly used in swarm control; a double-integrator model in 2D, standard in control literature; and, finally, a quadcopter model in 3D that exhibits nonlinear dynamics and a high-dimensional state space. For each example, we consider multi-agent path planning tasks with multiple obstacles, adding to the difficulty of the problems. This ensures that our claims are thoroughly supported and that our proposed approach is robust. In addition to our proposed approach, we conduct comparative experiments under two different settings: a baseline that uses automatic differentiation (AD) to unroll all projection layers, preserving exact gradient information but incurring high computational cost; and a CVXPY Layer–based projection solver that leverages implicit differentiation for gradient calculation, which is the current standard approach for end-to-end training under safety constraints.

Due to limited space, we defer details of the experiments to the Appendix. 
Numerical results regarding trained models' performance are presented in Table~\ref{tab:results_obj} where we report the results over different initial states after each training run. We note that \textit{our proposed approach is the only method that can successfully solve all test examples without numerical blowup or memory/time explosion.} 
Additional to that our approach generally achieves competitive performance across the benchmarks and outperforms the AD counterpart. In examples where CVXPY Layers-based solvers converge we notice that they can achieve lower terminal costs, though our approach generally yields lower running costs, indicating smoother trajectories.

We also note that CVXPY Layers, while commonly presented as the standard method for such solvers, can struggle with numerical instability in the presence of complex dynamics. To ensure a fair comparison, we additionally test with a regularized QP at each projection step by introducing relaxation variables. While this improves the numerical stability of training, the solver still struggles as we increase the dimensionality of the problem due to gradient explosion and solver breakdown.
We attribute this to the ill-conditioning of the KKT system during backpropagation, as documented in~\cite{bai2020multiscale,bai2021stabilizing}. In contrast, the JFB update we use does not require explicit Jacobian inverse, ensuring stable numerical performance across all experiments.

Aside from performance and training stability, a more prominent advantage of our approach and key to its scalability lies in its efficiency. 
Memory usage poses as a main bottleneck when scaling existing solvers to high-dimensional problems. We compare peak memory usage across all methods in Table~\ref{tab:results_mem}, \textit{our proposed approach achieves significantly lower memory footprint on both CPU and GPU across all examples}, and in some cases over $10\times$ reduction. 
We observe similar significant savings in training time as well when comparing to both AD and CVXPY Layers-based solvers,  these results are included in the Appendix and Figure~\ref{fig:training_time}.
We attribute these improvements to two main factors. First, DYS yields a closed-form projection update, rendering each fixed-point iteration fast and cheap. Second, the JFB update neither requires tracking and storing large computation graphs nor solving additional linear systems during backpropagation, making it both time and memory efficient.

Safety is ensured by design of our approach, to validate this, we present numerical results in Figure~\ref{fig:barrier_projection} in the Appendix evaluating the barrier values along trajectory rollouts. Notice here the minimum barrier values remain above zero, indicating the constraints are always satisfied, thus finalizing and substantiating all of our claims.

\section{Conclusion}

In this work, we introduce a scalable end-to-end training framework for safe semi-global feedback control with embedded CBF-QP safety filters. By combining Davis--Yin splitting with Jacobian-Free Backpropagation, our method avoids the computational bottlenecks of gradient unrolling and implicit differentiation while retaining hard safety guarantees. We provide convergence guarantees for the resulting nonsmooth optimization problem using novel techniques and demonstrate the approach across a range of nonlinear multi-agent systems. Comparisons and numerical results show that the proposed framework scales to substantially larger problems than existing methods while requiring significantly smaller training time and memory usage.

\newpage
\newpage

\bibliography{aaai2026}

\appendix
\onecolumn

\begin{figure}[t]
     \centering
     \begin{minipage}{0.9\linewidth}
     \centering
     \begin{subfigure}[b]{0.33\textwidth}
         \centering
         \includegraphics[width=\textwidth]{Figs/single_integrator_swarm_traj.png}
         \caption{Single integrator 50 agents}
         \label{fig:row1_col1}
     \end{subfigure}
     \hfill
     \begin{subfigure}[b]{0.3\textwidth}
         \centering
         \includegraphics[width=\textwidth]{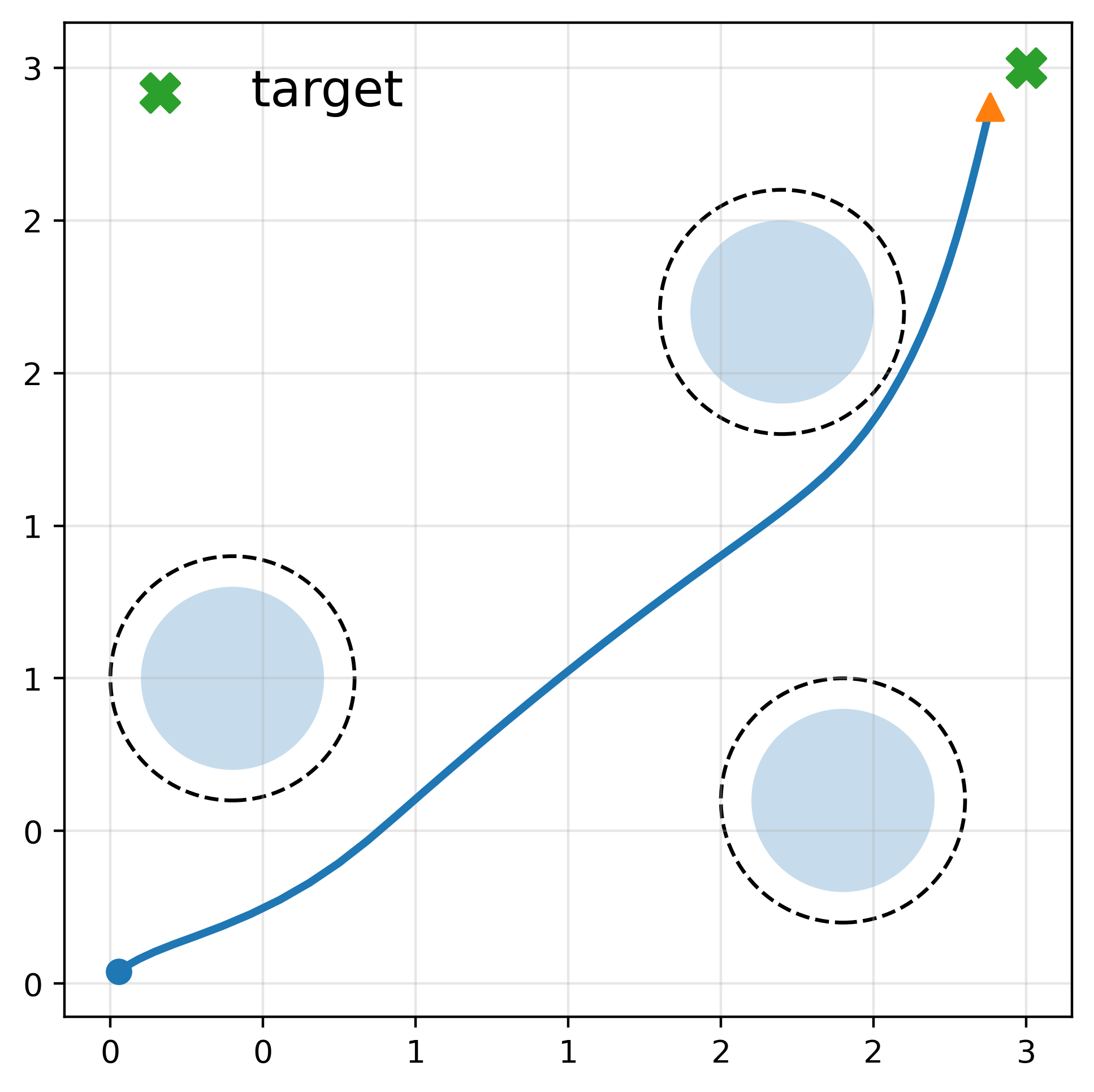}
         \caption{Double integrator 1 agent}
         \label{fig:row1_col2}
     \end{subfigure}
     \hfill
     \begin{subfigure}[b]{0.3\textwidth}
         \centering
         \includegraphics[width=\textwidth]{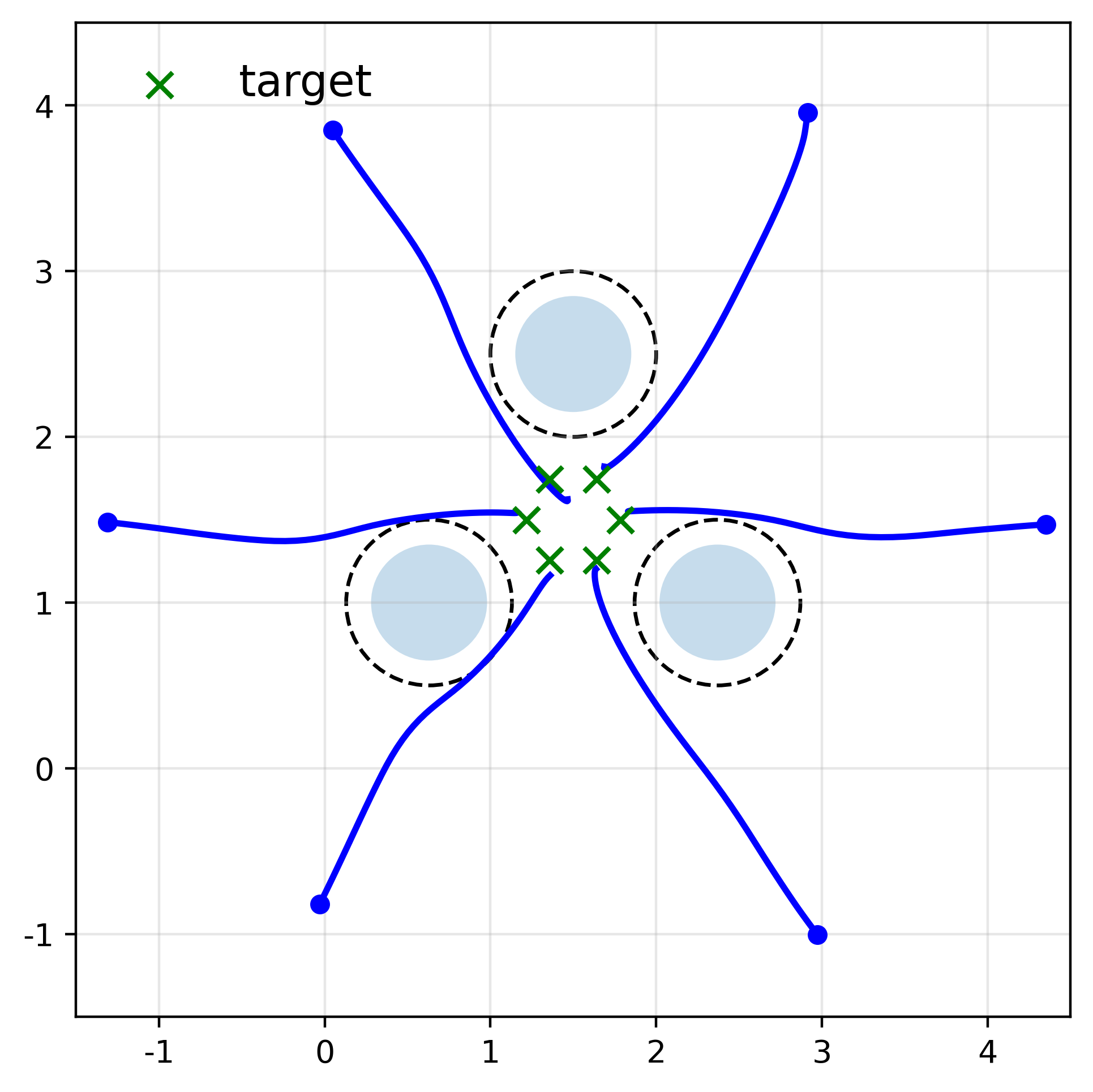}
         \caption{Double integrator 6 agents}
         \label{fig:row1_col3}
     \end{subfigure}

     \vspace{10pt} 

     \begin{subfigure}[b]{0.32\textwidth}
         \centering
         \includegraphics[width=\textwidth]{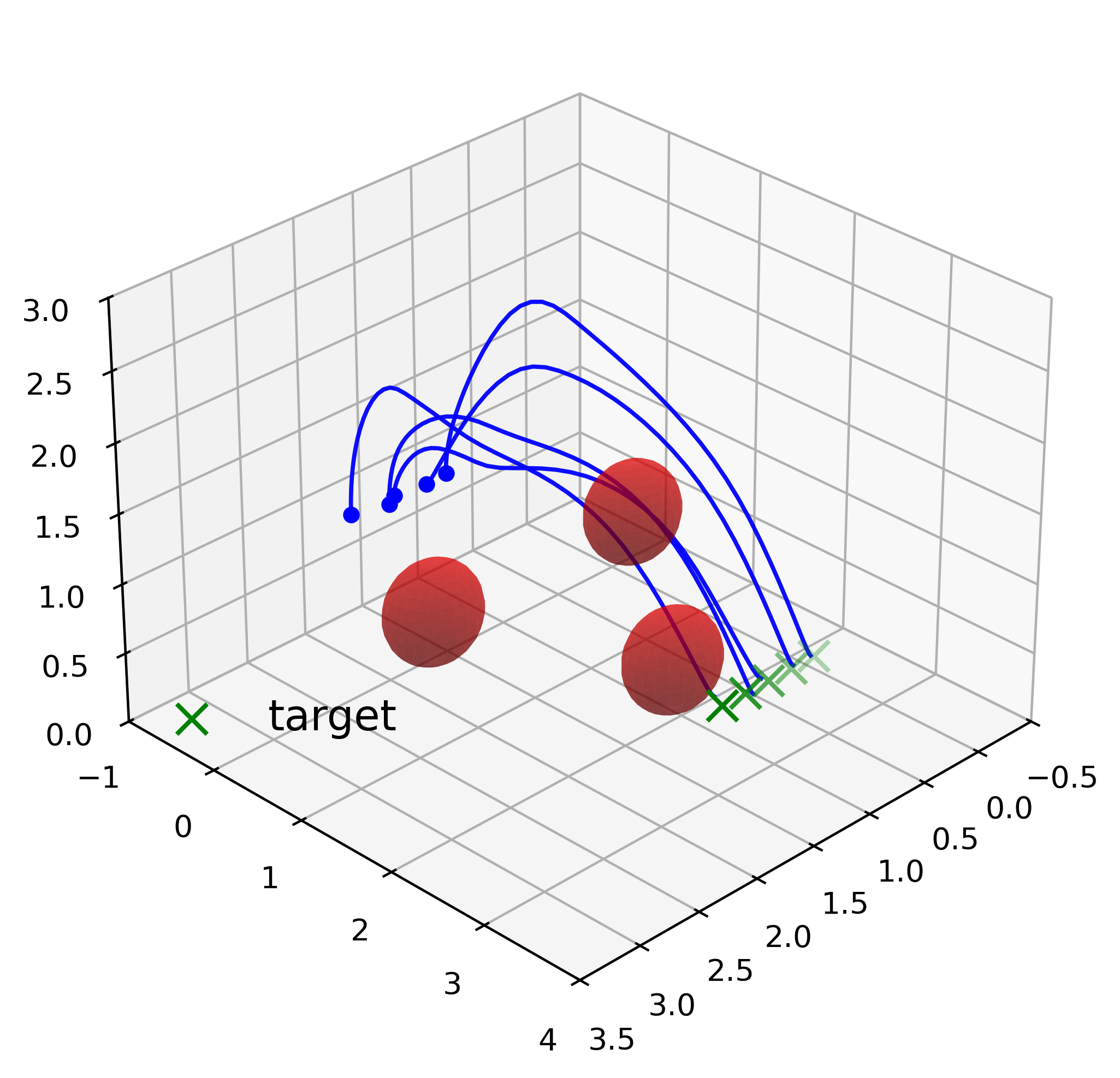}
         \caption{Quadcopter 5 agents}
         \label{fig:row2_col1}
     \end{subfigure}
     \hfill
     \begin{subfigure}[b]{0.32\textwidth}
         \centering
         \includegraphics[width=\textwidth]{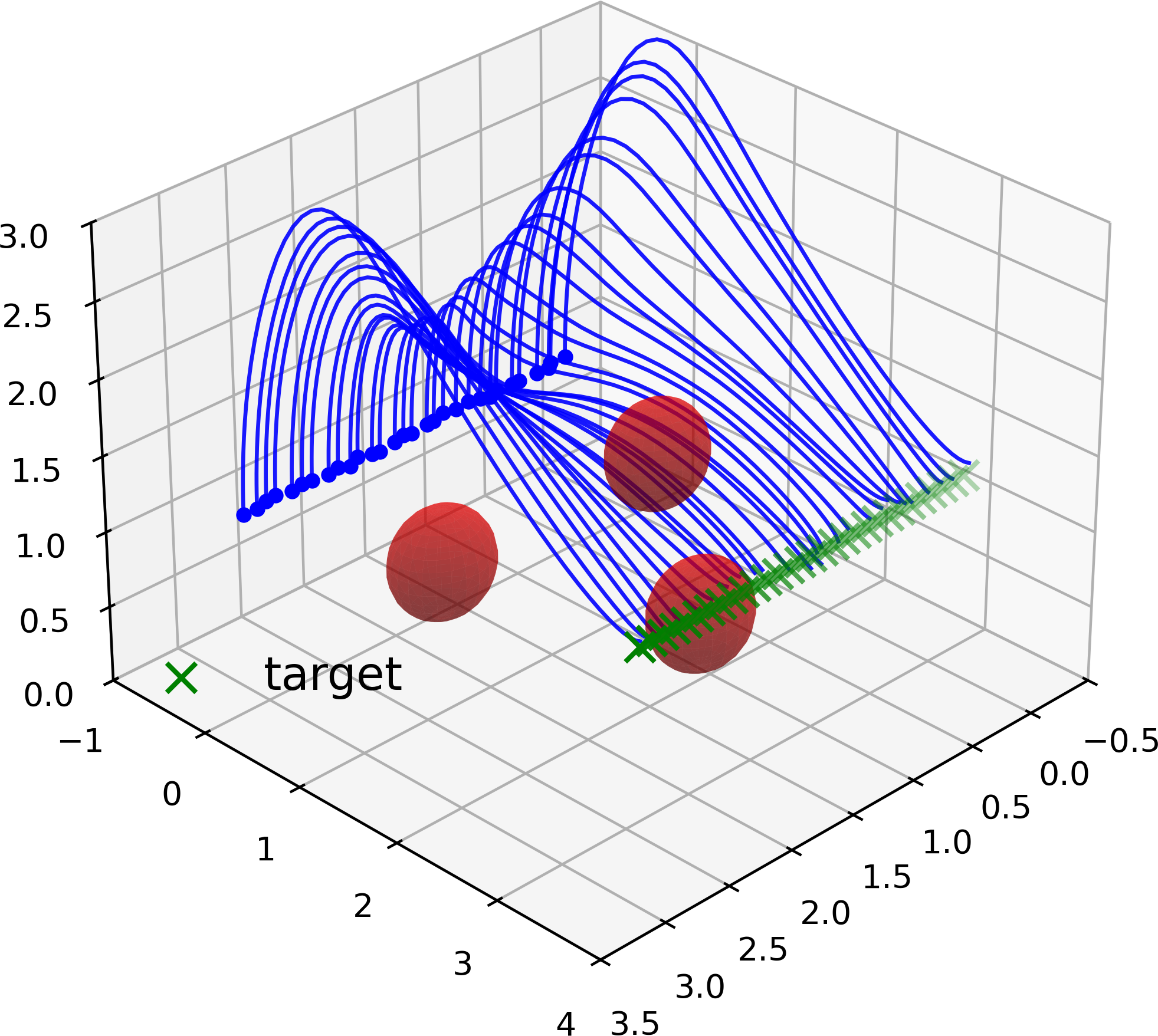}
         \caption{Quadcopter 30 agents}
         \label{fig:row2_col2}
     \end{subfigure}
     \hfill
     \begin{subfigure}[b]{0.32\textwidth}
         \centering
         \includegraphics[width=\textwidth]{Figs/quadcopter_100_finetune_traj.png}
         \caption{Quadcopter 100 agents}
         \label{fig:row2_col3}
     \end{subfigure}
    \end{minipage}
     \caption{Numerical results for six control scenarios. In each case, the feedback controller is trained in environments with multiple agents and obstacles. The learned policy successfully drives each agent to its corresponding target under minimal control effort, while ensuring the trajectories are smooth.}
     \label{fig:six_panel_grid}
\end{figure}

\begin{figure}[t]
    \centering
    \begin{subfigure}[b]{0.45\textwidth}
        \centering
        \includegraphics[width=\textwidth]{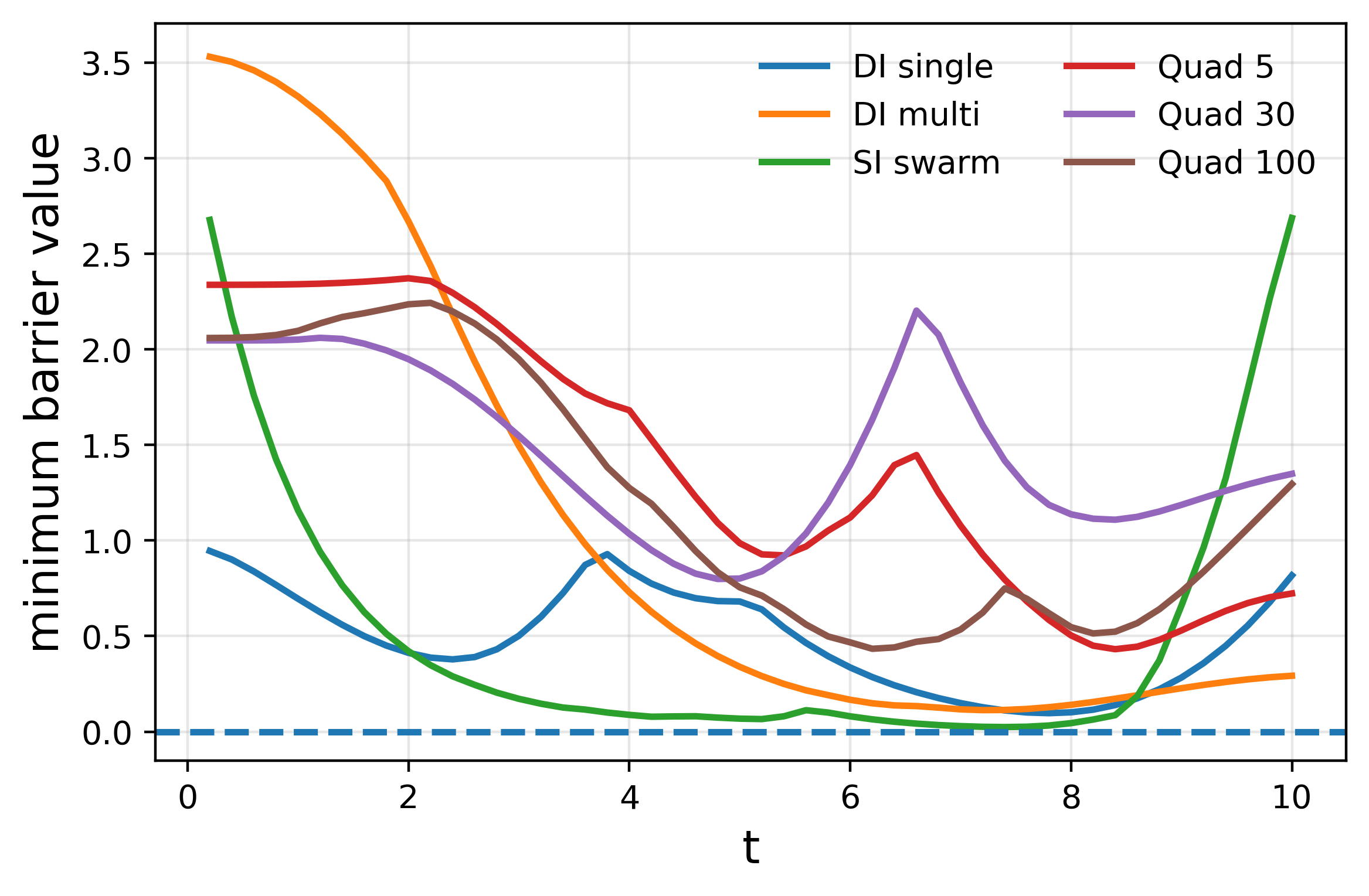}
        \caption{Minimum barrier values along trajectories.}
        \label{fig:barrier}
    \end{subfigure}
    \hfill
    \begin{subfigure}[b]{0.45\textwidth}
        \centering
        \includegraphics[width=\textwidth]{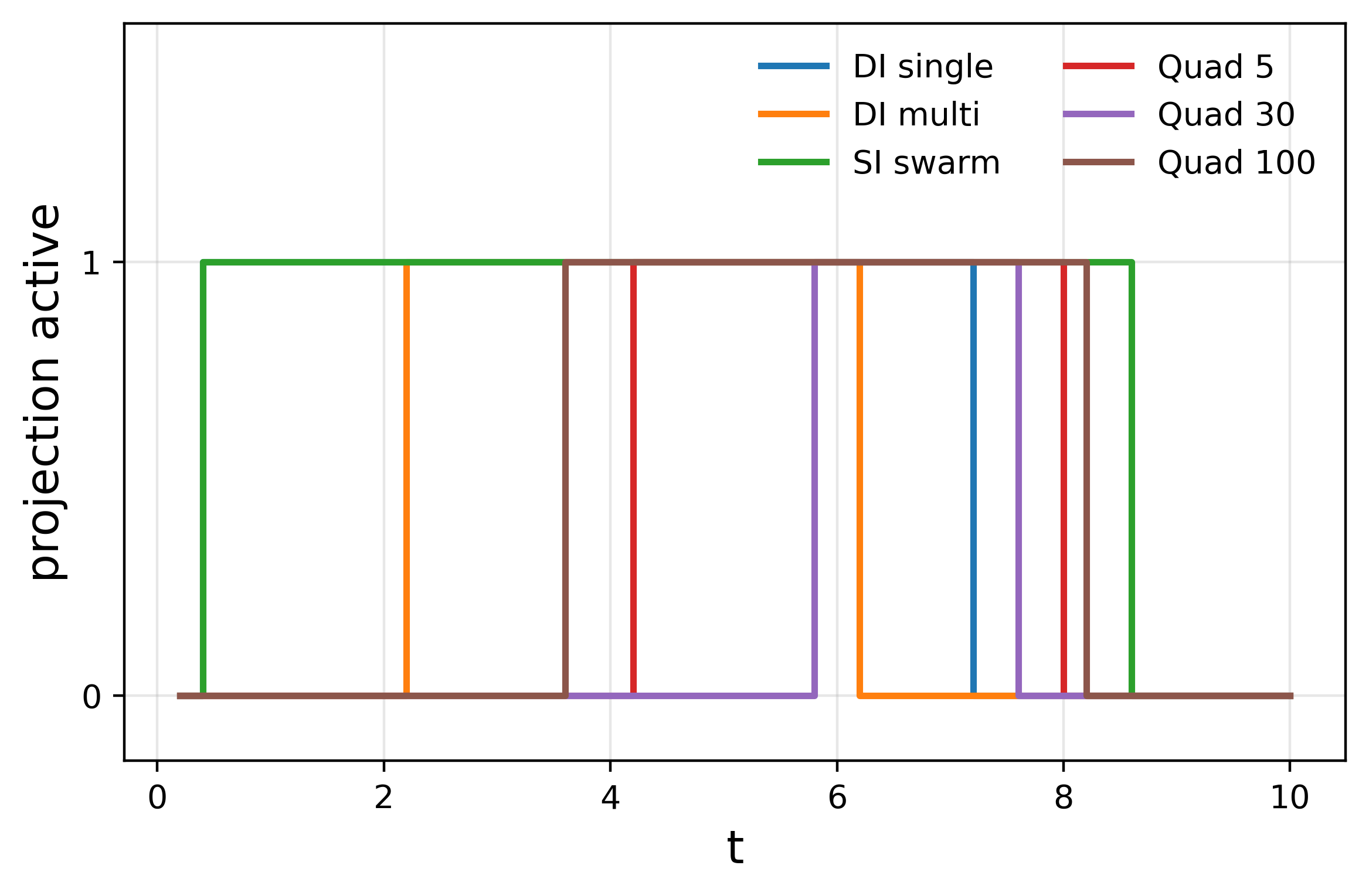}
        \caption{Projection activity along trajectories}
        \label{fig:proj_act}
    \end{subfigure}
    \caption{Evaluation of barrier functions and projection activations on trained models for each example. Here safety is ensured throughout as the minimum barrier values remain positive along trajectory rollouts. Projections only occur as the nominal controls push the trajectories into the obstacles, when far away from the obstacles, projection will not trigger. }
    \label{fig:barrier_projection}
\end{figure}

\section{Related Work}

Safety is a fundamental requirement in many control and automation systems, including robotics, autonomous driving, aerial vehicles, and industrial cyber-physical systems, where controllers must satisfy physical and operational constraints while accomplishing complex tasks~\cite{guiochet2017safety,schwarting2018planning,brunke2022safe,muller2022safeguarding,hsu2023safety}.To address this challenge, a wide range of control, optimization, and verification frameworks have been proposed, including robust and adaptive control, model predictive control (MPC), reachability analysis, formal verification, runtime assurance architectures, and control barrier functions (CBFs)~\cite{zhou1996robust,tomlin2000game,ames2019control,schwenzer2021review,hobbs2023runtime, fung2025mean, tomaselli2026mean}.
More recently, differentiable safety filters have emerged as a promising bridge between classical safety-critical control and learning-based methods. However, efficiently and reliably integrating such safety filters into end-to-end learning pipelines remains a central challenge in the field.

Seminal work in differentiable optimal control enables end-to-end learning by differentiating through trajectory optimization problems. Representative approaches include differentiable MPC~\cite{amos2018differentiable}, which differentiates through the KKT conditions of the underlying optimization problem, and Pontryagin Differentiable Programming (PDP)~\cite{jin2020pontryagin}, which computes gradients via Pontryagin Maximum Principle. Subsequent work extends these frameworks to robust~\cite{oshin2023differentiable} and safety-constrained~\cite{jin2021safe} optimal control problems. While elegant and effective, these methods are local solvers in nature and solve for individual trajectories at a time. They can struggle to scale, and generally do not apply to learning algorithms.

An important line of work~\cite{drgovna2022differentiable,cortez2022differentiable,drgovna2024learning} penalizes constraint violations as part of the loss, adjusting their strength by controlling penalty weights. While these methods are easy to implement, they can be difficult to optimize in practice due to the additional loss terms. Constraints violations can also happen and for safety-critical systems, additional safety corrections are often needed and applied as a post-processing step, which can further degrade model performance and lead to more complex pipelines and slower online inference. Similar penalty methods are also used in~\cite{tessler2018reward,onken2021neural,tearle2021predictive,yoo2022dynamic,mowlavi2023optimal,kan2025stability,kan2026optimal, vidal2023taming, lin2021alternating, vidal2025kernel, agrawal2022random, li2024neural}. Extension to safety in reinforcement learning also exist, see~\cite{choi2020reinforcement,wabersich2021predictive}

Pioneered by differentiable optimization layers such as OptNet~\cite{amos2017optnet} and CVXPY Layers~\cite{agrawal2019differentiable}, projection-based safety layers have emerged as a promising direction for end-to-end learning of feedback policies with safety guarantees. Existing work~\cite{chen2021enforcing,xiao2023barriernet,min2024hardnet,white2024projected,utkarsh2025end,li2025zero} demonstrates strong performance in balancing feasibility and optimality across dynamical systems, control policies, and other applications. However, projection layers can incur nontrivial computational costs during training, leading to significantly longer training times and higher memory usage. Additionally, optimization difficulty poses a further challenge. As a result, prior work has primarily focused on low-dimensional control problems with short time horizons. Developing scalable and computationally efficient pipelines therefore remains an open problem, which we aim to address in this work.

More recently, \cite{so2024train} proposes policy neural control barrier functions (PNCBFs), which learn a barrier certificate through the value function of a nominal policy. While this avoids directly enforcing CBF constraints during training and shows improved scalability to higher-dimensional, input-constrained systems, it does not address the integration of projection-based safety layers into end-to-end learning pipelines. Similar works also include~\cite{taylor2020learning,robey2020learning,dawson2023safe,liu2023safe,zhang2023neural}.

Importantly, our work focuses on scalable end-to-end learning of feedback controllers that preserve hard safety guarantees without external verification steps. As part of our main contribution, we address both the computation and theoretical openings in the field, to the best of our knowledge, our work is the first to enable large scale training with provable convergence guarantees.

\section{Control Barrier Functions for Dynamics with High Relative Degree}
\label{sec:HOCBF}

A standard CBF requires the safety function to have relative degree one. When the control input appears only after differentiating $h$ multiple times along the system dynamics, higher-order control barrier functions (HOCBFs) extend the CBF framework to guarantee forward invariance.

Suppose $h$ has relative degree $l$. Define the sequence of functions recursively by
\begin{align}
h_0(z) &:= h(z),\\
h_1(z) &:= \dot h_{0}(z) + \alpha_1(h_{0}(z)), \\
\vdots \\
h_l(z) &:= \dot h_{l-1}(z) + \alpha_l(h_{l-1}(z)),
\end{align}
where each $\alpha_i$ is an extended class-$\mathcal K$ function. The associated sets are
\begin{equation}
\mathcal C_i := \{z\in\mathbb R^n : h_{i-1}(z)\ge0\},
\qquad i=1,\ldots,l.
\end{equation}
A function $h$ is a HOCBF if
\begin{equation}
h_l(z)\ge0,
\label{eq:hocbf_condition}
\end{equation}
for all $z\in\bigcap_{i=1}^l\mathcal C_i$. Since the control input first appears in the $l$-th derivative of $h$,~\eqref{eq:hocbf_condition} is once again affine in the control input.

\begin{thm}[Forward invariance under HOCBF constraints~\cite{xiao2019control}]
Let $h$ be a HOCBF of relative degree $l$. Suppose $u(t)$ is locally Lipschitz,
$z(0)\in\bigcap_{i=1}^l\mathcal C_i$, and~\eqref{eq:hocbf_condition} holds along the resulting trajectory for all
$t\ge0$. Then $\bigcap_{i=1}^l\mathcal C_i$ is forward invariant under control $u(t)$. In particular,
$h(z(t))\ge0$ for all $t\ge0$.
\end{thm}
Compared with standard CBFs, HOCBFs enable safety constraints with arbitrary relative degree while preserving an affine constraint on the control input. As with standard CBFs, this constraint can be enforced by solving a quadratic program at each time step. In this work, we employ both CBFs and HOCBFs, as our numerical experiments consider systems with both first and second order dynamics. The corresponding formulations are presented in the following sections.

\section{Proof of Theorems}
Our proof is established in three main parts. First, we show that the DYS operator~\eqref{eq:DYS_operator} is contractive with respect to $y$. Second, we show that the JFB update is a descent direction. Finally, we establish the convergence of our training algorithm.

We first (re-)state the theory and then provide the proof.

\section{Proof of Contraction}

\subsection{Proof of Lemma~\ref{lem:dPC2}}

\begin{restatable}{lem}{lemdPC2}\label{lem:dPC2}
Let 
\begin{equation}
    C_2(z) = \{(u,s)\in \mathbb{R}^{m+c} : A(z)u + s = b(z)\},
\end{equation}
where $A(z) \in \mathbb{R}^{c \times m}$, $u \in \mathbb{R}^m$, $s \in \mathbb{R}^c$, $b(z) \in \mathbb{R}^c$. Let 
\begin{equation}\label{eq:defH2z}
\mathcal{H}_{2, z} := \textrm{Null}([A(z) \; I_{c \times c}]),
\end{equation}
then 
\begin{equation}
    \frac{d P_{C_2(z)}}{d y}(y) = P_{\mathcal{H}_{2, z}}, \quad \text{for all } y \in \mathbb{R}^{m+c}.
\end{equation}
\end{restatable}

\begin{proof}
    For clarity, we temporarily suppress the dependence on $z$ in this proof. The set $C_2$ can be re-written as 
    \begin{align*}
        C_2 &= \{(u,s)\in \mathbb{R}^{m+c} : Au + s = b\} \\
        &= \left\{(u,s)\in \mathbb{R}^{m+c} : \underbrace{[A \; I_{c\times c}]}_{=:\tilde{A}} \begin{bmatrix} u \\ s \end{bmatrix} = b \right\}.
    \end{align*}
    Thus, the projection operator is given by
    \begin{align*}
        P_{C_2}(y) &= y - \tA^\dagger (\tA y-b) = y -\tA^\dagger\tA y + \tA^\dagger b,
    \end{align*}
    where $\tA^\dagger$ denotes the pseudoinverse of $\tA$. Since $\tA$ contains the identity matrix $I_{c \times c}$ in its last $c$ columns, it has full row rank. Thus, denote $\tA= U \Sigma V^\top$ as the compact SVD of $\tA$, where $U \in \mathbb{R}^{c \times c}$, $\Sigma \in \mathbb{R}^{c \times c}$ has nonzero diagonal entries, and $V^\top \in \mathbb{R}^{c \times (m+c)}$. Thus $\tA^\dagger = V \Sigma^{-1} U^\top$ and the operator can be written as
    \begin{align*}
        P_{C_2}(y) &= y - VV^\top y + \tA^\dagger b.
    \end{align*}
    Differentiating the operator with respect to $y$ gives
    \begin{equation}\label{eq:dPC2_ortho}
        \frac{d P_{C_2(z)}}{d y}(y) = I - VV^\top.
    \end{equation}
    Here, $VV^\top$ is the orthogonal projection operator onto $\mathrm{range}(V)=\mathrm{range}(\tA^\top)$. Thus, \eqref{eq:dPC2_ortho} is the orthogonal projection operator onto $\mathrm{range}(\tA^\top)^\perp = \mathrm{Null}(\tA) = \mathcal{H}_2$, by the Fundamental Theorem of Linear Algebra.
\end{proof}

\subsection{Proof of Lemma~\ref{lem:dPC1}}

\begin{restatable}{lem}{lemdPC1}\label{lem:dPC1}
Let
\begin{equation}
    C_1 = \{(u,s) \in \mathbb{R}^{m+c}: s \ge 0\}.
\end{equation}
Suppose that $y$ is a differentiable point of $P_{C_1}$, that is,  $y_i \neq 0$ (the $i$th entry of $y$ is nonzero) for all $m+1 \leq i \leq m+c$, then
\begin{equation}\label{eq:PHoney}
    \frac{d P_{C_1}}{d y} (y)= \begin{bmatrix} I_{m \times m} & 0_{m \times c} \\ 0_{c \times m} & \mathrm{diag}(\mathbf{1}[{y_{m+1:m+c}>0}]) \end{bmatrix} =: P_{\Honey},
\end{equation}
where $\mathrm{diag}(\mathbf{1}[{y_{m+1:m+c}>0}]) \in \mathbb{R}^{c \times c}$, $\mathbf{1}$ is the element-wise indicator function, and 
\begin{equation}\label{eq:defH1y}
    \Honey = 
    \mathrm{span}\left(\{e_1,\ldots,e_m\} \cup \{e_{m+i} \mid 1 \leq i \leq c,\ y_{m+i} > 0\}\right).
\end{equation}
\end{restatable}

\begin{proof}
    Since 
    \begin{align*}
        P_{C_1}(y) = \begin{bmatrix}
            y_1 \\ \vdots \\ y_m \\ \mathbf{1}[{y_{m+1}}>0] y_{m+1}\\ \vdots \\ \mathbf{1}[{y_{m+c}}>0] y_{m+c}
        \end{bmatrix} = \begin{bmatrix} I_{m \times m} & 0_{m \times c} \\ 0_{c \times m} & \mathrm{diag}(\mathbf{1}[{y_{m+1:m+c}>0}]) \end{bmatrix}  \begin{bmatrix} y_1 \\ \vdots \\ y_m \\ y_{m+1} \\ \vdots \\ y_{m+c}
        \end{bmatrix},
    \end{align*}
    taking the derivative with respect to $y$, we obtain
    \begin{align*}
        \frac{d P_{C_1}}{d y} (y) = \begin{bmatrix} I_{m \times m} & 0_{m \times c} \\ 0_{c \times m} & \mathrm{diag}(\mathbf{1}[{y_{m+1:m+c}>0}]) \end{bmatrix},
    \end{align*}
    where the matrix on the right-hand-side is the orthogonal projection operator onto $\Honey$.
\end{proof}

\subsection{Proof of Theorem~\ref{thm:dT}}

\begin{restatable}{thm}{thmdT}\label{thm:dT}
Let $C_2(z) = \{(u,s)\in \mathbb{R}^{m+c} : A(z)u + s = b(z)\}$, and $T_\theta(y; z)$ be the DYS operator defined in~\eqref{eq:DYS_operator}. Then at any differentiable point $y \in \mathbb{R}^{m+c}$ for $P_{C_1}$,
\begin{equation}
    \frac{d T_\theta}{dy}(y; z) = P_{\mathcal{H}_{2, z}^\perp} P_{\mathcal{H}_{1, y}^\perp}   + P_{\mathcal{H}_{2, z}} \left( I - \zeta Q_u^\top Q_u  \right) P_{\Honey}.
\end{equation}
\end{restatable}

\begin{proof}
    Using~\eqref{eq:DYS_operator}, we have
    \begin{align*}
        \frac{d T_\theta}{dy}(y; z) &= I + \frac{dP_{C_2(z)}}{dy} (\tilde{y}) \left( 2 \frac{dP_{C_1}}{dy}(y) - I - \zeta Q_u^\top Q_u \frac{dP_{C_1}}{dy}(y)  \right) - \frac{dP_{C_1}}{dy}(y), \\
        \intertext{where $\tilde{y}=2P_{C_1(z)}(y)-y-\zeta Q_u^\top(Q_u P_{C_1(z)}(y)-u_\theta(t, z))$,}
        &= I + P_{\mathcal{H}_{2, z}} \left( 2P_{\Honey} - I - \zeta Q_u^\top Q_u P_{\Honey} \right) - P_{\Honey}, \\
        \intertext{where we applied Lemma~\ref{lem:dPC2} and Lemma~\ref{lem:dPC1},}
        &= \left(I- P_{\Honey}\right) + P_{\mathcal{H}_{2, z}} \left( P_{\Honey} - I \right) + P_{\mathcal{H}_{2, z}} \left( P_{\Honey} - \zeta Q_u^\top Q_u P_{\Honey} \right) \\
        &= \left(I- P_{\Honey}\right) + P_{\mathcal{H}_{2, z}} \left( P_{\Honey} - I \right) + P_{\mathcal{H}_{2, z}} \left( I - \zeta Q_u^\top Q_u  \right) P_{\Honey} \\
        &= \left(I- P_{\Honey}\right) + \left(I - P_{\mathcal{H}_{2, z}^\perp} \right) \left( P_{\Honey} - I \right) + P_{\mathcal{H}_{2, z}} \left( I - \zeta Q_u^\top Q_u  \right) P_{\Honey} \\
        &= \cancel{\left(I- P_{\Honey}\right)} + \cancel{\left( P_{\Honey} - I \right)} - P_{\mathcal{H}_{2, z}^\perp} \left( P_{\Honey} - I \right) + P_{\mathcal{H}_{2, z}} \left( I - \zeta Q_u^\top Q_u  \right) P_{\Honey} \\
        &= P_{\mathcal{H}_{2, z}^\perp} \left( I -P_{\Honey}  \right) + P_{\mathcal{H}_{2, z}} \left( I - \zeta Q_u^\top Q_u  \right) P_{\Honey} \\
        &= P_{\mathcal{H}_{2, z}^\perp} P_{\mathcal{H}_{1, y}^\perp}  + P_{\mathcal{H}_{2, z}} \left( I - \zeta Q_u^\top Q_u  \right) P_{\Honey}.
    \end{align*}
\end{proof}

\subsection{Proof of Lemma~\ref{lem:intersect}}

\begin{restatable}{lem}{lemintersect}\label{lem:intersect}
Under Assumption~\ref{assm:fullrank_A}, and if $y$ is a differentiable point of $P_{C_1}$, that is,  $y_i \neq 0$ (the $i$th entry of $y$ is nonzero) for all $m+1 \leq i \leq m+c$, then $\mathcal{H}_{1,y}^\perp \cap \mathcal{H}_{2,z}^\perp = \{ 0\}$.
\end{restatable}

\begin{proof}
    By the definition of $\mathcal{H}_{1,y}$~\eqref{eq:defH1y}, and that $y_i \neq 0$ for all $m+1 \leq i \leq m+c$, we have
    \begin{equation}\label{eq:H1perpHc}
    \mathcal{H}_{1,y}^\perp = \mathrm{span} \left( \{ e_{m+i} \;|\; 1 \leq i \leq c, y_{m+i} < 0 \}\right) \subseteq \mathrm{span} \left( \{ e_{m+i} \;|\; 1 \leq i \leq c \}\right) =: \mathcal{H}_{c}.
    \end{equation}
    On the other hand, by the definition of $\mathcal{H}_{2,z}$~\eqref{eq:defH2z}, and the Fundamental Theorem of Linear Algebra,
    \[
    \mathcal{H}_{2,z}^\perp = \mathrm{range} \begin{pmatrix}
        A(z)^\top \\ I_{c \times c}
    \end{pmatrix} = \mathrm{span} \Big( (a_1+e_{m+1}), (a_2+e_{m+2}), \cdots, (a_c+e_{m+c}) \Big),
    \]
    where $a_i \in \mathbb{R}^{m+c}$ contains the $i$th row of $A(z)$ in its first $m$ entries, with the last $c$ entries being zero.

    Next, we show that $\mathcal{H}_{c} \cap \mathcal{H}_{2,z}^\perp = \{ 0\}$. Suppose $v \in \mathcal{H}_{c} \cap \mathcal{H}_{2,z}^\perp$. Since $v \in \mathcal{H}_{c}$,
    \begin{equation}\label{eq:v1}
        v = \sum_{i=1}^c \alpha_i e_{m+i},
    \end{equation}
    for some scalars $\alpha_{i}$'s. On the other hand, since $v \in \mathcal{H}_{2,z}^\perp$,
    \begin{equation}\label{eq:v2}
        v = \sum_{i=1}^c \beta_i (a_i+e_{m+i}),
    \end{equation}
    for some scalars $\beta_{i}$'s.
    Combining \eqref{eq:v1} and \eqref{eq:v2}, we have
    \begin{align*}
        0 = v-v = \sum_{i=1}^c  \beta_{i} a_{i} + \sum_{i=1}^c  (\beta_{i} - \alpha_i) e_{m+i}.
    \end{align*}
    Here, on the right-hand-side, the first sum contains vectors with nonzero entries only in the first $m$ entries, and the second sum contains vectors with nonzero entries only in the subsequent $c$ entries. Since the two sums live in complementary coordinate blocks, they must each equal zero independently. Thus, since $A(z)$ has full row rank, $a_i$'s are linearly independent and thus $\beta_i=0$ for all $1 \leq i \leq c$. Consequently, we have
    \begin{align*}
        0 = - \sum_{i=1}^c  \alpha_i e_{m+i},
    \end{align*}
    and by the linear independence of $e_{m+i}$'s, $\alpha_i=0$ for all $1 \leq i \leq c$. Thus $v$ is the zero vector. Since $v \in \mathcal{H}_{c} \cap \mathcal{H}_{2,z}^\perp$ was arbitrary, this implies $\mathcal{H}_{c} \cap \mathcal{H}_{2,z}^\perp = \{ 0\}$.

    Finally, since by~\eqref{eq:H1perpHc}, $\mathcal{H}_{1,y}^\perp \subseteq \mathcal{H}_{c}$, we have $\mathcal{H}_{1,y}^\perp  \cap \mathcal{H}_{2,z}^\perp = \{ 0\}$.
\end{proof}

\subsection{Proof of Lemma~\ref{lem:PH2norm}}
\begin{restatable}{lem}{lemPH2norm}\label{lem:PH2norm}
Under Assumption~\ref{assm:fullrank_A}, we have
\begin{equation}
    \| P_{\Htwoz} \, p\|_2^2 < \| p \|_2^2 \quad \text{for any } 0 \neq p \in \mathrm{span}(e_{m+1},e_{m+2},\cdots,e_{m+c}).
\end{equation}
\end{restatable}

\begin{proof}
    Since $0 \neq p \in \mathrm{span}(e_{m+1},e_{m+2},\cdots,e_{m+c})$, 
    \[
    p = \sum_{i=1}^c \alpha_{m+i} e_{m+i},
    \]
    for some $\alpha_{m+i}$'s, not all zero.
    Note that 
    \[
    [A(z) \; I_{c \times c}] \, p = [A(z) \; I_{c \times c}] \left( \sum_{i=1}^c \alpha_{m+i} e_{m+i} \right) = \sum_{i=1}^c \alpha_{m+i} e_{i}\neq 0,
    \]
    since not all $\alpha_{m+i}$'s are zero.
    Thus, $p \notin  \textrm{Null}([A(z) \, I_{c \times c}]) = \mathcal{H}_{2, z}$.

    Decompose $p$ as $p = P_{\Htwoz} \, p + P_{\Htwoz^\perp} \,p$. Since $p \neq 0$ and $p \notin \mathcal{H}_{2, z}$, $P_{\Htwoz^\perp} \,p \neq 0$. Thus, we have
    \[
    \| P_{\Htwoz} \, p\|_2^2 < \| P_{\Htwoz} \, p\|_2^2 + \| P_{\Htwoz^\perp} \, p\|_2^2 = \| P_{\Htwoz} \, p  +  P_{\Htwoz^\perp} \, p\|_2^2 = \| p \|_2^2.
    \]
    Here, in the first step, we used $\|P_{\Htwoz^\perp} \,p \|_2 > 0$ because $P_{\Htwoz^\perp} \,p \neq 0$, and in the second step we used that $P_{\Htwoz} \, p  \perp P_{\Htwoz^\perp} \, p$.
\end{proof}

\subsection{Proof of Theorem~\ref{thm:contraction}}
\begin{restatable}{thm}{thmcontraction}\label{thm:contraction}
Under Assumption~\ref{assm:fullrank_A},
the DYS hyperparameter $\zeta \in (0, 1)$, and $y$ is a differentiable point of $T_\theta$, 
\begin{equation}
    \left\| \frac{d T_\theta}{dy} (y;z) \right\|_2 < 1.
\end{equation}
\end{restatable}

\begin{proof}
    From the definition of $T_\theta$~\eqref{eq:DYS_operator}, since $P_{C_2}$ is smooth, $y$ is a differentiable point of $T_\theta$ if and only if $y$ is a differentiable point of $P_{C_1}$. Thus, by Theorem~\ref{thm:dT}, for any nonzero vector $p \in \mathbb{R}^{m+c}$,
    \begin{equation}\label{eq:dTtimesp}
        \frac{d T_\theta(y; z)}{dy} \, p = \underbrace{P_{\mathcal{H}_{2, z}^\perp} P_{\mathcal{H}_{1, y}^\perp} \, p}_{\cI}  + \underbrace{P_{\mathcal{H}_{2, z}} \left( I - \zeta Q_u^\top Q_u  \right) P_{\Honey} \, p}_{\cII}.
    \end{equation}
    We will bound the norms of the two terms $\cI$ and $\cII$ separately.
    
    \paragraph{First Term Analysis.} We begin with the first term $\cI$. By~\ref{lem:intersect}, $\mathcal{H}_{1,y}^\perp \cap \mathcal{H}_{2,z}^\perp = \{ 0\}$. This implies that either (a) at least one of the $\mathcal{H}_{1,y}^\perp$ or $\mathcal{H}_{2,z}^\perp$ is the trivial subspace $\{ 0\}$, or (b) the first principal angle $\tau$ between these two subspaces is nonzero, and thus 
    \begin{equation}\label{eq:angle_cond}
        1 > \cos(\tau) := \max_{\substack{v \in \mathcal{H}_{1,y}^\perp, \, w \in \mathcal{H}_{2,z}^\perp \\ \|v\|_2=\|w\|_2=1}} \langle v, w \rangle.
    \end{equation}
    \paragraph{First Term Case (a).} In case (a), we have
    \begin{align}\label{eq:perpcase1}
        \left\| P_{\mathcal{H}_{2, z}^\perp} P_{\mathcal{H}_{1, y}^\perp} \, p \right\|_2^2 
        = 0 \leq \left\| P_{\mathcal{H}_{1, y}^\perp} \, p \right\|_2^2, \quad \text{with equality if and only if } p\in \mathcal{H}_{1, y}.
    \end{align}
    \paragraph{First Term Case (b).} In case (b), we have
    \begin{align}
        \left\| P_{\mathcal{H}_{2, z}^\perp} P_{\mathcal{H}_{1, y}^\perp} \, p \right\|_2^2 &= \left\langle P_{\mathcal{H}_{2, z}^\perp} P_{\mathcal{H}_{1, y}^\perp} \, p, P_{\mathcal{H}_{2, z}^\perp} P_{\mathcal{H}_{1, y}^\perp} \, p \right\rangle \nonumber \\
        &= \left\langle P_{\mathcal{H}_{1, y}^\perp} \, p, P_{\mathcal{H}_{2, z}^\perp} P_{\mathcal{H}_{2, z}^\perp} P_{\mathcal{H}_{1, y}^\perp} \, p \right\rangle, \nonumber  \\ 
        \intertext{since orthogonal projection operators are self-adjoint,}
        &= \left\langle P_{\mathcal{H}_{1, y}^\perp} \, p, P_{\mathcal{H}_{2, z}^\perp} P_{\mathcal{H}_{1, y}^\perp} \, p \right\rangle, \nonumber \\ 
        \intertext{since projection operators are idempotent,}
        &\leq \cos(\tau) \left\| P_{\mathcal{H}_{1, y}^\perp} \, p \right\|_2 \left\| P_{\mathcal{H}_{2, z}^\perp} P_{\mathcal{H}_{1, y}^\perp} \, p \right\|_2, \nonumber \\
        \intertext{by the angle condition~\eqref{eq:angle_cond}, since $P_{\mathcal{H}_{1, y}^\perp} \in \mathcal{H}_{1, y}^\perp$ and $P_{\mathcal{H}_{2, z}^\perp} P_{\mathcal{H}_{1, y}^\perp} \, p \in \mathcal{H}_{2, z}^\perp$,}
        &\leq \cos(\tau) \left\| P_{\mathcal{H}_{1, y}^\perp} \, p \right\|_2^2, \nonumber \\
        \intertext{by the linearity and nonexpansiveness of the orthogonal projection operator $P_{\mathcal{H}_{2, z}^\perp}$,}
        &< \left\| P_{\mathcal{H}_{1, y}^\perp} \, p \right\|_2^2,\label{eq:perpcase2} \\\intertext{by~\eqref{eq:angle_cond} again.} \nonumber
    \end{align}
    \paragraph{First Term Result.} Thus, combining~\eqref{eq:perpcase1} for case (a) and~\eqref{eq:perpcase2} for case (b), we conclude that the norm squared of the first term $\cI$ satisfies
        \begin{align}\label{eq:perpcase12}
        \left\| P_{\mathcal{H}_{2, z}^\perp} P_{\mathcal{H}_{1, y}^\perp} \, p \right\|_2^2 \leq \left\| P_{\mathcal{H}_{1, y}^\perp} \, p \right\|_2^2, \quad \text{with equality if and only if } p\in \mathcal{H}_{1, y}.
    \end{align}

    \paragraph{Second Term Analysis.} We investigate the second term $\cII$. We consider the following two cases: (a) $p \in \Honey$, and (b) $p \notin \Honey$. In both cases, the second term $\cII$ can be simplified as
    \begin{align}
        &P_{\mathcal{H}_{2, z}} \left( I - \zeta Q_u^\top Q_u  \right) P_{\Honey} \, p \nonumber \\
        &\quad= P_{\mathcal{H}_{2, z}} \begin{bmatrix} I_{m \times m}-\zeta I_{m \times m} & 0_{m \times c} \\ 0_{c \times m} & I_{c \times c} \end{bmatrix} 
        \begin{bmatrix} I_{m \times m} & 0_{m \times c} \\ 0_{c \times m} & \mathrm{diag}(\mathbf{1}[{y_{m+1:m+c}>0}]) \end{bmatrix} \, p, \nonumber \\
        \intertext{by the definition of $P_{\Honey}$ in~\eqref{eq:PHoney} and the definition of $Q_u$ in~\eqref{eq:Q_def},}
        &\quad= P_{\mathcal{H}_{2, z}}
        \begin{bmatrix} I_{m \times m}-\zeta I_{m \times m} & 0_{m \times c} \\ 0_{c \times m} & \mathrm{diag}(\mathbf{1}[{y_{m+1:m+c}>0}]) \end{bmatrix} \, p \nonumber \\
        &\quad= P_{\mathcal{H}_{2, z}}  \begin{bmatrix} (1-\zeta) \, p_1 \\ \vdots \\ (1-\zeta) \, p_m \\  \mathbf{1}[{y_{m+1}}>0] \, p_{m+1} \\ \vdots \\ \mathbf{1}[{y_{m+c}}>0] \, p_{m+c}
        \end{bmatrix}. \label{eq:PHtwoz_simplify}
    \end{align}
    
    \paragraph{Second Term Case (a).} We first investigate case (a). Recall from~\eqref{eq:defH1y} the definition of $\Honey$
    \[
    \Honey =
    \mathrm{span}\left(\{e_1,\ldots,e_m\} \cup \{e_{m+i} \mid 1 \leq i \leq c,\ y_{m+i} > 0\}\right).
    \] 
    Since $0 \neq p \in \Honey$, it must have at least one nonzero entry in the index set 
    \[
    \{ 1,2,\cdots,m\} \cup \{ m+i \,|\, 1 \leq i \leq c, \; y_{m+i}>0 \}.
    \]
    Thus, under case (a), either of the following two subcases must happen: (a)(i) $p$ has at least one nonzero entry among its first $m$ entries, and (a)(ii) all of the first $m$ entries of $p$ are zero, which implies that $p$ only has nonzezro entries with index in $\{ m+i \,|\, 1 \leq i \leq c, \; y_{m+i}>0 \}$.

    Evaluating the norm squared of the second term $\cII$, we have
    \begin{align}
        \left\| P_{\mathcal{H}_{2, z}} \left( I - \zeta Q_u^\top Q_u  \right) P_{\Honey} \, p \right\|_2^2 &= \left\| P_{\mathcal{H}_{2, z}}  \begin{bmatrix} (1-\zeta) \, p_1 \\ \vdots \\ (1-\zeta) \, p_m \\  \mathbf{1}[{y_{m+1}}>0] \, p_{m+1} \\ \vdots \\ \mathbf{1}[{y_{m+c}}>0] \, p_{m+c}
        \end{bmatrix} \right\|_2^2, \nonumber \\
        \intertext{by~\eqref{eq:PHtwoz_simplify},}
        &< \left\| \begin{bmatrix}  p_1 \\ \vdots \\  p_m \\  \mathbf{1}[{y_{m+1}}>0] \, p_{m+1} \\ \vdots \\ \mathbf{1}[{y_{m+c}}>0] \, p_{m+c}
        \end{bmatrix} \right\|_2^2 = \| P_{\Honey} p \|_2^2. \label{eq:paracase1}
    \end{align}
    Here, the equality (second step) of~\eqref{eq:paracase1} follows directly from the definition of $\Honey$~\eqref{eq:defH1y}. And the strict inequality (first step) in~\eqref{eq:paracase1} is established separately for the two subcases (a)(i) and (a)(ii). 
    
    On one hand, under (a)(i), $p$ has at least one nonzero entry among its first $m$ entries. The strict inequality in~\eqref{eq:paracase1} is obtained as follows.
    \[
    \left\| P_{\mathcal{H}_{2, z}}  \begin{bmatrix} (1-\zeta) \, p_1 \\ \vdots \\ (1-\zeta) \, p_m \\  \mathbf{1}[{y_{m+1}}>0] \, p_{m+1} \\ \vdots \\ \mathbf{1}[{y_{m+c}}>0] \, p_{m+c}
        \end{bmatrix} \right\|_2^2 \leq \left\| \begin{bmatrix} (1-\zeta) \, p_1 \\ \vdots \\ (1-\zeta) \, p_m \\  \mathbf{1}[{y_{m+1}}>0] \, p_{m+1} \\ \vdots \\ \mathbf{1}[{y_{m+c}}>0] \, p_{m+c}
        \end{bmatrix} \right\|_2^2 < \left\| \begin{bmatrix}  p_1 \\ \vdots \\  p_m \\  \mathbf{1}[{y_{m+1}}>0] \, p_{m+1} \\ \vdots \\ \mathbf{1}[{y_{m+c}}>0] \, p_{m+c}
        \end{bmatrix} \right\|_2^2.
    \]
    Here, in the first step, we used the linearity and nonexpansiveness of the orthogonal projection operator $P_{\mathcal{H}_{2, z}}$. In the second step, we used $(1-\zeta)p_i<p_i$ for all $p_i \neq 0$ with $0\leq i \leq m$, since $\zeta \in (0,1)$.

    On the other hand, under (a)(ii), $p$ only has nonzezro entries with index in $\{ m+i \,|\, 1 \leq i \leq c, \; y_{m+i}>0 \}$, in such case, the vector inside the norm is
    \begin{equation}\label{eq:vec_norm_cond}
        \begin{bmatrix}  (1-\zeta)\, p_1 \\ \vdots \\  (1-\zeta)\, p_m \\  \mathbf{1}[{y_{m+1}}>0] \, p_{m+1} \\ \vdots \\ \mathbf{1}[{y_{m+c}}>0] \, p_{m+c}
        \end{bmatrix} = \begin{bmatrix}  0 \\ \vdots \\  0 \\  \mathbf{1}[{y_{m+1}}>0] \, p_{m+1} \\ \vdots \\ \mathbf{1}[{y_{m+c}}>0] \, p_{m+c}
        \end{bmatrix} \in \mathrm{span}(e_{m+1},\cdots,e_{m+c}).
    \end{equation}
    In addition, this vector is nonzero, because under the condition of subcase (a)(ii), the nonzero entries of $p$ match where the entries of $y$ are positive. Thus, the strict inequality in~\eqref{eq:paracase1} is obtained as
    \begin{equation}
        \left\| P_{\mathcal{H}_{2, z}}  \begin{bmatrix} (1-\zeta) \, p_1 \\ \vdots \\ (1-\zeta) \, p_m \\  \mathbf{1}[{y_{m+1}}>0] \, p_{m+1} \\ \vdots \\ \mathbf{1}[{y_{m+c}}>0] \, p_{m+c}
        \end{bmatrix} \right\|_2^2 < \left\|   \begin{bmatrix} (1-\zeta) \, p_1 \\ \vdots \\ (1-\zeta) \, p_m \\  \mathbf{1}[{y_{m+1}}>0] \, p_{m+1} \\ \vdots \\ \mathbf{1}[{y_{m+c}}>0] \, p_{m+c}
        \end{bmatrix} \right\|_2^2 = \left\|   \begin{bmatrix}  p_1 \\ \vdots \\  p_m \\  \mathbf{1}[{y_{m+1}}>0] \, p_{m+1} \\ \vdots \\ \mathbf{1}[{y_{m+c}}>0] \, p_{m+c}
        \end{bmatrix} \right\|_2^2,
    \end{equation}
    where in the first step, we used~\eqref{eq:vec_norm_cond} and Lemma~\ref{lem:PH2norm}, and in the second step, we used $(1-\zeta)p_i=p_i$, because $p_i=0$, for $1 \leq i \leq m$.

    \paragraph{Second Term Case (b).} Next, we consider the second case (b): $p \notin \Honey$. In such case, we have
    \begin{align}
        \left\| P_{\mathcal{H}_{2, z}} \left( I - \zeta Q_u^\top Q_u  \right) P_{\Honey} \, p \right\|_2^2 &= \left\| P_{\mathcal{H}_{2, z}}  \begin{bmatrix} (1-\zeta) \, p_1 \\ \vdots \\ (1-\zeta) \, p_m \\  \mathbf{1}[{y_{m+1}}>0] \, p_{m+1} \\ \vdots \\ \mathbf{1}[{y_{m+c}}>0] \, p_{m+c}
        \end{bmatrix} \right\|_2^2, \nonumber \\
        \intertext{by~\eqref{eq:PHtwoz_simplify},}
        &\leq  \left\| \begin{bmatrix} (1-\zeta) \, p_1 \\ \vdots \\ (1-\zeta) \, p_m \\  \mathbf{1}[{y_{m+1}}>0] \, p_{m+1} \\ \vdots \\ \mathbf{1}[{y_{m+c}}>0] \, p_{m+c}
        \end{bmatrix} \right\|_2^2, \nonumber  \\
        \intertext{by the linearity and nonexpansiveness of the orthogonal projection operator $P_{\mathcal{H}_{2, z}}$,}
        &\leq  \left\| \begin{bmatrix}  p_1 \\ \vdots \\ p_m \\  \mathbf{1}[{y_{m+1}}>0] \, p_{m+1} \\ \vdots \\ \mathbf{1}[{y_{m+c}}>0] \, p_{m+c}
        \end{bmatrix} \right\|_2^2, \nonumber \\
        \intertext{since $\zeta \in (0, 1)$,}
        &= \| P_{\Honey} p \|_2^2, \label{eq:paracase2}\\
        \intertext{which follows directly from the definition of $\Honey$~\eqref{eq:defH1y}.} \nonumber 
    \end{align}
    
    \paragraph{Second Term Result.} Thus, combining \eqref{eq:paracase1} for case (a) and \eqref{eq:paracase2} for case (b), we conclude that 
    \begin{equation}\label{eq:paracase12} 
        \|P_{\mathcal{H}_{2, z}} \left( I - \zeta Q_u^\top Q_u  \right) P_{\Honey} \, p\|_2^2 \leq \| P_{\Honey} p \|_2^2 \quad \text{with equality if and only if } p \notin \Honey.
    \end{equation}

    \paragraph{Final Bound.} Combing the bounds~\eqref{eq:perpcase12} for the first term $\cI$ and~\eqref{eq:paracase12} for the second term $\cII$, we have
    \begin{equation}\label{eq:paraallcase} 
        \left\| P_{\mathcal{H}_{2, z}^\perp} P_{\mathcal{H}_{1, y}^\perp} \, p \right\|_2^2 + \|P_{\mathcal{H}_{2, z}} \left( I - \zeta Q_u^\top Q_u  \right) P_{\Honey} \, p\|_2^2 < \left\| P_{\mathcal{H}_{1, y}^\perp} \, p \right\|_2^2 + \| P_{\Honey} p \|_2^2,
    \end{equation}
    for all $0 \neq p \in \mathbb{R}^{m+c}$.
    
    Next, taking $l_2$ norm squared on both sides of~\eqref{eq:dTtimesp}, we have
    \begin{align}
        \left\| \frac{d T_\theta(y; z)}{dy} \, p\right\|_2^2 &= \left\| P_{\mathcal{H}_{2, z}^\perp} P_{\mathcal{H}_{1, y}^\perp} \, p  + P_{\mathcal{H}_{2, z}} \left( I - \zeta Q_u^\top Q_u  \right) P_{\Honey} \, p \right\|_2^2 \nonumber \\
        &= \left\| P_{\mathcal{H}_{2, z}^\perp} P_{\mathcal{H}_{1, y}^\perp} \, p \right\|_2^2 + \left\| P_{\mathcal{H}_{2, z}} \left( I - \zeta Q_u^\top Q_u  \right) P_{\Honey} \, p \right\|_2^2, \nonumber \\
        \intertext{by orthogonality,}
        &< \left\| P_{\mathcal{H}_{1, y}^\perp} \, p \right\|_2^2 + \left\|  P_{\Honey} \, p \right\|_2^2, \nonumber \\
        \intertext{by~\eqref{eq:paraallcase} ,}
        &= \left\| P_{\mathcal{H}_{1, y}^\perp} \, p  +   P_{\Honey} \, p \right\|_2^2, \nonumber \\
        \intertext{by orthogonality,}
        &= \left\| p \right\|_2^2 . \label{eq:dTdyp_bound}
    \end{align}

    Finally, taking square root on both sides of~\eqref{eq:dTdyp_bound}, we obtain
    \[
    \left\| \frac{d T_\theta(y; z)}{dy} \, p\right\|_2 < \left\| p \right\|_2.
    \] 
    Thus, the operator 2-norm satisfies
    \[
    \left\| \frac{d T_\theta(y; z)}{dy} \right\|_2 := \sup_{ \|p\|_2=1 } \left\| \frac{d T_\theta(y; z)}{dy} \, p\right\|_2  < 1,
    \]
    because the set $\{ p\in \mathbb{R}^{m+c} \; | \;\|p\|_2=1 \}$ is compact, and the mapping $p \mapsto \left\| \frac{d T_\theta(y; z)}{dy} \, p\right\|_2$ is continuous, so the supremum is attained and the strict inequality above holds at the maximizer.
\end{proof}

\subsection{Proof of Corollary~\ref{cor:contraction}}
\corcontraction*

\begin{proof}
By Theorem~\ref{thm:contraction}, we have
    \begin{equation}\label{eq:contraction_recite}
    \left\| \frac{d T_\theta}{dy} (y;z) \right\|_2 < 1,
\end{equation}
for all differentiable points $y$. Moreover, by Lemma~\ref{lem:dPC1} and Theorem~\ref{thm:dT}, we have
\begin{align}
    \frac{d T_\theta}{dy}(y; z) &= P_{\mathcal{H}_{2, z}^\perp} P_{\mathcal{H}_{1, y}^\perp}   + P_{\mathcal{H}_{2, z}} \left( I - \zeta Q_u^\top Q_u  \right) P_{\Honey} \nonumber \\
    &= P_{\mathcal{H}_{2, z}^\perp} \begin{bmatrix} 0_{m \times m} & 0_{m \times c} \\ 0_{c \times m} & \mathrm{diag}(\mathbf{1}[{y_{m+1:m+c}\leq 0}]) \end{bmatrix}  + P_{\mathcal{H}_{2, z}} \left( I - \zeta Q_u^\top Q_u  \right) \begin{bmatrix} I_{m \times m} & 0_{m \times c} \\ 0_{c \times m} & \mathrm{diag}(\mathbf{1}[{y_{m+1:m+c}>0}]) \end{bmatrix}. \label{eq:dTdy_finite}
\end{align}

Thus, for all $\theta$, $z$, and differentiate point $y$ of $T_\theta(\cdot; z)$,
\begin{align*}
    \norm{\frac{dT_{\theta}(y; z)}{dy}} & \leq \sup_{\text{differentiable pts } \tilde{y}} \;\sup_{\tilde{\theta}, \tilde{z}} \norm{\frac{dT_{\tilde{\theta}}(\tilde{y}; \tilde{z})}{dy}} \\
    &= \sup_{\text{differentiable pts } \tilde{y}} \;\sup_{\tilde{\theta}, \tilde{z}} \Bigg\| P_{\mathcal{H}_{2, \tilde{z}}^\perp} \begin{bmatrix} 0_{m \times m} & 0_{m \times c} \\ 0_{c \times m} & \mathrm{diag}(\mathbf{1}[{\tilde{y}_{m+1:m+c}\leq 0}]) \end{bmatrix}  \\ 
    &\quad+ P_{\mathcal{H}_{2, \tilde{z}}} \left( I - \zeta Q_u^\top Q_u  \right) \begin{bmatrix} I_{m \times m} & 0_{m \times c} \\ 0_{c \times m} & \mathrm{diag}(\mathbf{1}[{\tilde{y}_{m+1:m+c}>0}]) \end{bmatrix} \Bigg\|_2, \\
    \intertext{by~\eqref{eq:dTdy_finite},}
    &= \underbrace{\max_{d \in \{ 0, 1\}^c }}_{\text{finite}} \; \underbrace{\sup_{\tilde{\theta}, \tilde{z}}}_{\text{compact}} \norm{P_{\mathcal{H}_{2, \tilde{z}}^\perp} \begin{bmatrix} 0_{m \times m} & 0_{m \times c} \\ 0_{c \times m} & \mathrm{diag}(d) \end{bmatrix}  + P_{\mathcal{H}_{2, \tilde{z}}} \left( I - \zeta Q_u^\top Q_u  \right) \begin{bmatrix} I_{m \times m} & 0_{m \times c} \\ 0_{c \times m} & \mathrm{diag}(1- d) \end{bmatrix}} \\
    & \leq \gamma,
\end{align*}
where $\gamma \in (0,1)$. In the last step, we used~\eqref{eq:contraction_recite} and the compactness of the domains of $\theta$ and $z$ (Assumption~\ref{assm:compactness}), which imply that the supremum and maximum are achieved at a maximizer.

Hence, for any sequence $\{ \tilde{y}_k\}_k$ with limiting point $\tilde{y}$, where $T_\theta(\cdot, z)$ is differentiable at $\tilde{y}_k$ for all $k \in \mathbb{N}$,
\begin{equation}
    \norm{ \lim_{k \to \infty} \frac{dT_{\theta}(\tilde{y}_k; z)}{dy} } = \lim_{k \to \infty} \norm{  \frac{dT_{\theta}(\tilde{y}_k; z)}{dy} } \leq \gamma,
\end{equation}
where we used the continuity of matrix 2-norm. Finally, since the Clarke generalized Jacobian $\partial^C_y T_{\theta}(y; z)$ is the convex hull of all such limit points of Jacobians (see Definition~\ref{def:Clarke}), any $M_y \in \partial^C_y T_{\theta}(y; z)$ also satisfies
    \begin{equation}
        \norm{M_y} \leq \gamma,
    \end{equation}
    for all $\theta$, $y$ and $z$.
\end{proof}

\section{Proof of Descent Direction}

\subsection{Proof of Lemma~\ref{lemma:coercive}}
\begin{restatable}{lem}{lemcoercive}\label{lemma:coercive}
    Under Assumptions~\ref{assm:fullrank_A} and~\ref{assm:compactness}, for any $$M_y\in \partial^C_y T_{\theta}(y_\theta^\star(t,z_x(t));z_x(t)),$$ $\Jcal_\theta :=I-M_y$ satisfies
    \begin{align}
    1-\gamma \leq \| \Jcal_\theta\T \|_2 \leq 1+\gamma \label{eq:J_bound} \\
    \inner{\psi}{\Jcal_\theta\T \psi} \geq (1-\gamma)\norm{\psi}^2 \label{eq:J_inner_bound}
    \end{align}
    for all $\psi \in \mathbb{R}^{m+c}$, $\theta$, $t$, and $x$. Here, $\gamma \in (0,1)$ is the contraction constant from Corollary~\ref{cor:contraction}.
\end{restatable}

\begin{proof}
    By Corollary~\ref{cor:contraction}, any $M_y \in \partial^C_y T_{\theta}(y^\star; z)$ satisfies
    \begin{equation}
        \norm{M_y} \leq \gamma,
    \end{equation}
    for all $\theta$, $z$, and $y^\star$ being a fixed point of the DYS iteration. Thus, we also have
    \begin{equation}
        \norm{M_y\T} \leq \gamma,
    \end{equation}
    for all $\theta$, $z$, and $y^\star$ being a fixed point of the DYS iteration.
    Hence, we have
    \begin{align*}
        \| \Jcal_\theta\T \|_2 &= \norm{I - M_y\T}\\
        &\leq \norm{I} + \norm{M_y\T}\\
        &\leq \norm{I} + \gamma \\
        &= 1 + \gamma.
    \end{align*}
    On the other hand,
    \begin{align*}
    \| \mathcal{J}_\theta\T \|_2 &= \norm{I - M_y\T} \\
    &\geq \left| \norm{I} - \norm{M_y\T} \right| \\
    &\geq \norm{I} - \gamma \\
    &= 1 - \gamma.
\end{align*}
Thus, we have
$$1-\gamma \leq \| \Jcal_\theta\T \|_2 \leq 1+\gamma. $$
    Moreover, for any $\psi \in \mathbb{R}^{m+c}$, 
    \begin{align*}
        \inner{\psi}{\Jcal_\theta\T\psi} &= \inner{\psi}{\left( I-M_y\T \right) \psi} \\
        &= \inner{\psi}{\psi} - \inner{\psi}{M_y\T \psi} \\
        &\geq \norm{\psi}^2 - \norm{M_y\T} \norm{\psi}^2 \\
        & \geq \norm{\psi}^2 - \gamma \norm{\psi}^2 \\
        &= (1-\gamma) \norm{\psi}^2 .
    \end{align*}
\end{proof}

\subsection{Proof of Lemma~\ref{lemma:integrand_inner_product}}
The proofs in this subsection use the structure of $Q_u = [I_{m \times m}, 0_{m \times c}]$, which extracts the first $m$ entries.

For any $\theta, t, z$ and any selection $M_y \in \partial_y^C T_\theta(y_\theta^\star)$, write the block decompositions
\begin{equation}\label{eq:blocks}
    M_y = \begin{bmatrix} K_{uu} & K_{us} \\ K_{su} & K_{ss} \end{bmatrix},
    \qquad
    M_\theta \coloneqq \frac{\partial T_{\theta}}{\partial \theta}(y_\theta^\star(t, z); z)
    = \begin{bmatrix} M^{u}_{\theta} \\ M^{s}_{\theta} \end{bmatrix},
\end{equation}
where $K_{uu} \in \mathbb{R}^{m \times m}$, $K_{us} \in \mathbb{R}^{m \times c}$, $K_{su} \in \mathbb{R}^{c \times m}$, $K_{ss} \in \mathbb{R}^{c \times c}$, and $M^{u}_{\theta} = Q_u M_\theta \in \mathbb{R}^{m \times p}$, $M^{s}_{\theta} \in \mathbb{R}^{c \times p}$. Thus, accordingly we have
\begin{equation}\label{eq:J_blocks}
    \Jcal_\theta \coloneqq I - M_y
    = \begin{bmatrix} I_{m \times m} - K_{uu} & -K_{us} \\ -K_{su} & I_{c \times c} - K_{ss} \end{bmatrix}.
\end{equation}

Let $\gamma \in (0,1)$ be the contraction constant from Corollary~\ref{cor:contraction}, and define
\begin{equation}\label{eq:delta_theta_def}
    \delta_\theta \coloneqq \max_{M_y \in \partial_y^C T_\theta(y_\theta^\star)} \norm{K_{su}} \in [0, \gamma],
\end{equation}
by Corollary~\ref{cor:contraction},
\begin{equation}\label{eq:def_beta_theta}
    \beta_\theta \coloneqq \frac{1 - \gamma^2 + \gamma\, \delta_\theta}{1 - \gamma} \geq 0,
\end{equation}
\begin{equation}\label{eq:def_r_theta}
    r_\theta \coloneqq \frac{1-\gamma}{\beta_\theta} \in (0,1),
\end{equation}
as $\beta_\theta \geq \frac{1-\gamma^2}{1-\gamma}=1+\gamma$, thus $r_\theta = \frac{1-\gamma}{\beta_\theta} \leq \frac{1-\gamma}{1+\gamma}<1$, and
\begin{equation}\label{eq:gamma_eff}
    \gamma_{\mathrm{eff}} \coloneqq \frac{1 - r_\theta}{1 + r_\theta}
    = 1- \frac{2(1-\gamma)^2}{2(1-\gamma)+\gamma \delta_\theta}
    \in [\gamma, 1),
\end{equation}
as $\gamma_{\mathrm{eff}}(\delta_\theta)$ is strictly less than 1, strictly increasing on $\delta_\theta \in [0, \gamma]$, and $\gamma_{\mathrm{eff}}(0)=\gamma$.

With the defined notations, we first state the full Assumption~\ref{assumption:M}.
\begin{restated}{assumption:M} 
\begin{enumerate}
    \item[(i)] \emph{(rank)} $M^{u}_{\theta}  \in \mathbb{R}^{m \times p}$ has full row rank; write $G_\theta \coloneqq M^{u}_{\theta} (M^{u}_{\theta})^\top \in \mathbb{R}^{m \times m}$, $G_\theta \succ 0$, and let $\sigma_{+}^2 \ge \sigma_{-}^2 > 0$ denote its largest and smallest eigenvalues;
    \item[(ii)] \emph{(conditioning)} $\exists \gamma_\mathrm{eff}>0$, s.t. $\displaystyle \kappa(G_\theta) = \frac{\sigma_{+}^2}{\sigma_{-}^2} < \frac{1}{\gamma_{\mathrm{eff}}}$;
    \item[(iii)] \emph{(weak coupling)} the gradient-mismatch matrix
    \begin{equation}\label{eq:Xi_def}
        \Xi_\theta \coloneqq K_{us}\, (I_{c \times c} - K_{ss})^{-1}\, M^{s}_{\theta}\, (M^{u}_{\theta})^\top \in \mathbb{R}^{m \times m}
    \end{equation}
    satisfies, for every selection $M_y \in \partial_y^C T_\theta(y_\theta^\star)$,
    \begin{equation}\label{eq:Xi_bound}
        \norm{\Xi_\theta} \le \rho_\theta \coloneqq \frac{(1-\gamma)\, (1 + r_\theta)\,  \big( \sigma_{-}^2 - \gamma_{\mathrm{eff}}\, \sigma_{+}^2 \big)}{2\, \beta_\theta } \frac{\sigma_{-}^2}{\sigma_{+}^2}.
    \end{equation}
\end{enumerate}
\end{restated}

First, we demonstrate that the Schur complement of the $s$-block of $\Jcal_\theta$ has similar bounds as the ones in Lemma~\ref{lemma:coercive}.

\begin{restatable}{lem}{lemschur}\label{lemma:schur_coercive}[Schur complement of $\Jcal_\theta$ with respect to the $(u,s)$ splitting]
For all $\theta, t, z$ and $M_y \in \partial_y^C T_\theta(y_\theta^\star)$, we have
\begin{enumerate}
    \item[(1)] $I_{c \times c} - K_{ss}$ is invertible, with $\norm{(I_{c \times c} - K_{ss})^{-1}} \le \frac{1}{1-\gamma}$;
    \item[(2)] the Schur complement
    \begin{equation}\label{eq:S_def}
        S_\theta \coloneqq (I_{m \times m} - K_{uu}) - K_{us}\, (I_{c \times c} - K_{ss})^{-1}\, K_{su} \;\in\; \mathbb{R}^{m \times m}
    \end{equation}
    satisfies
    \begin{align}
        \inner{\phi_u}{S_\theta\, \phi_u} &\ge (1-\gamma)\, \norm{\phi_u}^2, \qquad \forall\, \phi_u \in \mathbb{R}^m, \label{eq:S_inner_bound} \\
        \norm{S_\theta} &\le \beta_\theta; \label{eq:S_bound} 
    \end{align}
    \item[(3)] $S_\theta$ is invertible with $\norm{S_\theta^{-1}} \le \frac{1}{1-\gamma}$, and the first block row of $\Jcal_\theta^{-1}$ is given by
    \begin{equation}\label{eq:block_row}
        Q_u\, \Jcal_\theta^{-1} = S_\theta^{-1}\, \big[\; I_{m \times m} \quad K_{us}\, (I_{c \times c} - K_{ss})^{-1} \;\big] \;\in\; \mathbb{R}^{m \times (m+c)}.
    \end{equation}
\end{enumerate}
\end{restatable}

\begin{proof}
    To reduce technical clutter, in this proof we write $S = S_\theta$ and $\Jcal = \Jcal_\theta$. By Corollary~\ref{cor:contraction} every Clarke selection satisfies $\norm{M_y} \le \gamma < 1$. Thus, every block of $M_y$ satisfies the same bound, that is, 
    \begin{equation}\label{eq:block_norms}
    \norm{K_{uu}},\ \norm{K_{us}},\ \norm{K_{su}},\ \norm{K_{ss}} \;\le\; \gamma.
\end{equation}

\textbf{(1)} Since $\norm{K_{ss}} \le \gamma < 1$, the Neumann series gives invertibility of $I_c - K_{ss}$ with $\norm{(I_c - K_{ss})^{-1}} \le \frac{1}{1 - \norm{K_{ss}}} \le \frac{1}{1-\gamma}$.

\textbf{(2) Coercivity \eqref{eq:S_inner_bound}.} Let $\phi_u \in \mathbb{R}^m$, and set
\begin{equation}\label{eq:phi_s_def}
    \phi_s \coloneqq (I_{c \times c} - K_{ss})^{-1} K_{su}\, \phi_u \in \mathbb{R}^c,
    \qquad
    \phi \coloneqq (\phi_u, \phi_s) \in \mathbb{R}^{m+c}.
\end{equation}
Then, by \eqref{eq:J_blocks} and \eqref{eq:S_def},
\begin{equation*}
    \Jcal\, \phi
    = \begin{bmatrix} (I_{m \times m} - K_{uu})\, \phi_u - K_{us}\, \phi_s \\[2pt] -K_{su}\, \phi_u + (I_{c \times c} - K_{ss})\, \phi_s \end{bmatrix} = \begin{bmatrix} (I_{m \times m} - K_{uu})\, \phi_u - K_{us}\, (I_{c \times c} - K_{ss})^{-1} K_{su}\, \phi_u \\[2pt] \big(-(I_{c \times c} - K_{ss}) + (I_{c \times c} - K_{ss}) \big)\, \phi_s \end{bmatrix}
    = \begin{bmatrix} S\, \phi_u \\[2pt] 0 \end{bmatrix}.
\end{equation*}
Here, in the first step, we used the definition of $\Jcal$ \eqref{eq:J_blocks}. In the second step, we used the definition of $\phi_s$ \eqref{eq:phi_s_def}. In the third step, we used the definition of the Schur complement $S$ \eqref{eq:S_def}. Next, since the $s$-component of $\Jcal \phi$ vanishes,
\begin{equation*}
    \inner{\phi_u}{S \phi_u}
    = \inner{(\phi_u, \phi_s)}{(S\phi_u,\, 0)}
    = \inner{\phi}{\Jcal\, \phi}
    \overset{\eqref{eq:J_inner_bound}}{\ge} (1-\gamma)\, \norm{\phi}^2
    = (1-\gamma)\, \big( \norm{\phi_u}^2 + \norm{\phi_s}^2 \big)
    \ge (1-\gamma)\, \norm{\phi_u}^2.
\end{equation*}

\textbf{Norm bound \eqref{eq:S_bound}.} We have
\begin{align*}
    \norm{S} &\le \left(\norm{I_{m\times m}} + \norm{K_{uu}} \right) + \norm{K_{us}}\, \norm{(I_c - K_{ss})^{-1}}\, \norm{K_{su}}, \\
    \intertext{by the definition of $S$ \eqref{eq:S_def},}
    &\le (1+\gamma) + \frac{\gamma\, \norm{K_{su}}}{1-\gamma}, \\
    \intertext{by \eqref{eq:block_norms} and part \textbf{(1)}}
    &\le (1+\gamma) + \frac{\gamma\, \delta_\theta}{1-\gamma}, \\
    \intertext{by the definition of $\delta_\theta$ \eqref{eq:delta_theta_def},}
    &= \frac{1-\gamma^2 + \gamma \delta_\theta}{1 - \gamma} = \beta_\theta,
    \intertext{by the definition of $\beta_\theta $\eqref{eq:def_beta_theta}.}
\end{align*}

\textbf{(3) Invertibility} By \eqref{eq:S_inner_bound} and Cauchy-Schwarz inequality, $\norm{S\phi_u}\, \norm{\phi_u} \ge \inner{\phi_u}{S\phi_u} \ge (1-\gamma) \norm{\phi_u}^2$, so $\norm{S\phi_u} \ge (1-\gamma)\norm{\phi_u}$ for all $\phi_u \in \mathbb{R}^m$.
This implies that $S$ is injective, as $S \phi_u = 0$ if and only if $\phi_u = 0$. Thus, $S$ is an invertible matrix, with $\norm{S^{-1}} \le \frac{1}{1-\gamma}$.

For \eqref{eq:block_row}, we set $R \coloneqq S^{-1} \big[\, I_{m \times m} \;\; K_{us}(I_{c \times c} - K_{ss})^{-1} \,\big]$. We will verify $R\, \Jcal = Q_u$ by direct multiplication using \eqref{eq:J_blocks}. The first block column of $R\, \Jcal$ is
\begin{align*}
    S^{-1} \begin{bmatrix}
        I_{m \times m} \;\; K_{us}(I_{c \times c} - K_{ss})^{-1} \,
    \end{bmatrix} \begin{bmatrix} I_{m \times m} - K_{uu} \\ -K_{su}  \end{bmatrix} = S^{-1} \underbrace{\Big[ (I_{m \times m} - K_{uu}) - K_{us}\, (I_{c \times c} - K_{ss})^{-1} K_{su} \Big]}_{= S, \text{ by its definition \eqref{eq:S_def}}} = S^{-1} S = I_{m \times m},
\end{align*}
and the second is
\begin{equation*}
    S^{-1} \begin{bmatrix}
        I_{m \times m} \;\; K_{us}(I_{c \times c} - K_{ss})^{-1} \,
    \end{bmatrix} \begin{bmatrix} -K_{us} \\ I_{c \times c} - K_{ss} \end{bmatrix} = S^{-1} \Big[ -K_{us} + K_{us}\, (I_{c \times c} - K_{ss})^{-1} (I_{c \times c} - K_{ss}) \Big] = 0.
\end{equation*}
Since $\Jcal$ is invertible by Lemma~\ref{lemma:coercive}, multiplying $R\, \Jcal = Q_u$ on the right by $\Jcal^{-1}$ gives $R = Q_u \Jcal^{-1}$, as desired.
\end{proof}

\begin{restatable}{lem}{leminnerproduct}\label{lemma:integrand_inner_product}
Under Assumptions~\ref{assm:fullrank_A}--\ref{assumption:M}, for any selection of Clarke subgradient $v_{\theta, x}(t) \in \partial^C_\theta J_{x, t}(\theta)$ that is measurable with respect to $x$ and $t$,
    $$\inner{v_{\theta, x}(t)}{w_{\theta,x}(t)} \geq 0,\quad \forall x, t, \theta.$$
\end{restatable}

\begin{proof}
Fix $x, t, \theta$. By~\eqref{eq:integrand_definitions}, for any measurable selection $v_{\theta, x}(t) \in \partial^C_\theta J_{x, t}(\theta)$, there exist a measurable selection 
$M_y \in { {\partial}_y^C} T_{\theta}$ such that 
\begin{equation}\label{eq:def_v}
    v_{\theta, x}(t) = \left(\left(I -
M_y
\right)^{-1}
\frac{\partial T_{\theta}}{\partial \theta}(y_\theta^\star)\right)^\top Q_u^\top
h_{\theta,x}(t).
\end{equation}
Write $h \coloneqq h_{\theta,x}(t) \in \mathbb{R}^m$ and let $S \coloneqq S_\theta$ be the Schur complement \eqref{eq:S_def} associated with this selection.
By \Cref{lemma:schur_coercive}(2), $\norm{S} \le \beta_\theta$.

Let $M_\theta \coloneqq \frac{\partial T_{\theta}}{\partial \theta}(y_\theta^\star) \in \mathbb{R}^{(m+c) \times p}$. Substituting it and \eqref{eq:J_blocks}: $\Jcal_\theta = I - M_y$ into \eqref{eq:def_v} and \eqref{eq:integrand_definitions_JFB}, we obtain
\begin{equation*}
    v_{\theta,x} = M_\theta^\top\, \Jcal_\theta^{-\top} Q_u^\top h,
    \qquad
    w_{\theta,x} = M_\theta^\top\, Q_u^\top h = (M^{u}_{\theta})^\top h.
\end{equation*}
Therefore, expanding $Q_u\, \Jcal_\theta^{-1} M_\theta$ via the block-row formula \eqref{eq:block_row} and the splitting \eqref{eq:blocks} of $M_\theta$,
\begin{align}
    \inner{v_{\theta,x}}{w_{\theta,x}}
    &= h^\top\, Q_u\, \Jcal_\theta^{-1} M_\theta\, (M^{u}_{\theta})^\top h
    = h^\top S^{-1} \Big[ M^{u}_{\theta} + K_{us}\, (I_c - K_{ss})^{-1} M^{s}_{\theta} \Big] (M^{u}_{\theta})^\top h \nonumber \\
    &= \underbrace{h^\top S^{-1}  M^{u}_{\theta}(M^{u}_{\theta})^\top\, h}_{\eqqcolon\, \text{\textcircled{I}}} \;+\; \underbrace{h^\top S^{-1}\, \Xi_\theta\, h}_{\eqqcolon\, \text{\textcircled{II}}}, \label{eq:vw_split}
\end{align}
with $\Xi_\theta$ as in \eqref{eq:Xi_def}.
 
\emph{Term \textcircled{I}.} By Assumption~\ref{assumption:M}(i), $B_\theta \coloneqq  \left( M^{u}_{\theta}(M^{u}_{\theta})^\top \right)^{-1}$ is symmetric positive definite; let $\lambda_{+} = \sigma_{-}^{-2} \ge \lambda_{-} = \sigma_{+}^{-2} > 0$ be its largest and smallest eigenvalues and $\bar\lambda \coloneqq \frac{1}{2}(\lambda_{+} + \lambda_{-})$. Let $\psi \coloneqq M^{u}_{\theta}(M^{u}_{\theta})^\top \, S^{-\top} h  = B_\theta^{-1} S^{-\top} h\in \mathbb{R}^m$, so that $h = S^\top \left( M^{u}_{\theta}(M^{u}_{\theta})^\top \right)^{-1} \, \psi = S^\top B_\theta\, \psi$ and, using $B_\theta G_\theta = I_{m \times m}$,
\begin{align*}
    \text{\textcircled{I}} &= h^\top S^{-1} M^{u}_{\theta}(M^{u}_{\theta})^\top \, h = h^\top S^{-1} B_\theta^{-1} \, h  \\
    &= \big( S^\top B_\theta \psi \big)^\top S^{-1}\, B_\theta^{-1} \, \big( S^\top B_\theta \psi \big) = \inner{\psi}{S^\top B_\theta\, \psi} \\
    &= \inner{\psi}{S^\top (\bar\lambda I_{m \times m} + B_\theta - \bar\lambda I_{m \times m})\, \psi} \\
    &= \bar\lambda \inner{\psi}{S^\top  \psi} + \inner{\psi}{S^\top ( B_\theta - \bar\lambda I_{m \times m})\, \psi} \\
    &\geq \bar\lambda (1-\gamma) \norm{\psi}^2 + \inner{\psi}{S^\top ( B_\theta - \bar\lambda I_{m \times m})\, \psi},  \quad \text{using $\inner{\psi}{S^\top \psi} = \inner{S \psi}{\psi} \ge (1-\gamma)\norm{\psi}^2$ by \eqref{eq:S_inner_bound} of Lemma~\ref{lemma:schur_coercive},} \\
    &\geq \bar\lambda (1-\gamma) \norm{\psi}^2 - \norm{S^\top} \norm{B_\theta - \bar\lambda I_{m \times m}} \norm{\psi}^2 \\
    &\geq \frac{\lambda_{+} + \lambda_{-}}{2} (1-\gamma) \norm{\psi}^2 - \beta_\theta \left( \frac{\lambda_{+} - \lambda_{-}}{2} \norm{\psi}^2 \right)\\
    \intertext{by \eqref{eq:S_bound} of Lemma~\ref{lemma:schur_coercive}: $\norm{S^\top} = \norm{S} \le \beta_\theta$, and $\norm{B_\theta - \bar\lambda I_m} = \frac{\lambda_{+} - \lambda_{-}}{2}$}
    &= \frac{1}{2} \Big[ \beta_\theta r_\theta (\lambda_{+} + \lambda_{-} )-\beta_\theta (\lambda_{+} - \lambda_{-})\Big] \norm{\psi}^2, \quad \\
    \intertext{using the definition of $r_\theta \coloneqq \frac{1-\gamma}{\beta_\theta}$ in \eqref{eq:def_r_theta},}
    &= \frac{1}{2} \Big[ \beta_\theta r_\theta (\lambda_{+} + \lambda_{-} )-\beta_\theta (\lambda_{+} - \lambda_{-})\Big] \norm{\psi}^2 \\
    &= \frac{1}{2} \Big[ \beta_\theta (1+r_\theta) \lambda_{-} - \beta_\theta (1 - r_\theta ) \lambda_{+} \Big] \norm{\psi}^2 \\
    &= \frac{\beta_\theta(1+r_\theta)}{2} \Big[ \lambda_{-} - \frac{1 - r_\theta}{1 + r_\theta} \lambda_{+} \Big] \norm{\psi}^2 \\
    &= \frac{\beta_\theta(1+r_\theta)}{2} \Big[ \lambda_{-} - \gamma_\mathrm{eff} \lambda_{+} \Big] \norm{\psi}^2, \quad \text{by the definition of $\gamma_\mathrm{eff}$ in \eqref{eq:gamma_eff}.}
\end{align*}
Moreover, since $\norm{\psi} = \norm{B_\theta^{-1} S^{-\top} h} \ge \lambda_{\min}(B_\theta^{-1})\, \norm{S^{-\top} h} \ge \frac{\sigma_{-}^2}{\beta_\theta}\, \norm{h}$ by Lemma~\ref{lemma:schur_coercive}, and $\lambda_{-} - \gamma_{\mathrm{eff}}\, \lambda_{+} = \frac{1}{\sigma_+^2} - \gamma_\mathrm{eff} \frac{1}{\sigma_-^2} = \frac{\sigma_{-}^2 - \gamma_{\mathrm{eff}}\, \sigma_{+}^2}{\sigma_{-}^2\, \sigma_{+}^2}$, we have
\begin{align*}
    \text{\textcircled{I}} &\geq \frac{\beta_\theta(1+r_\theta)}{2} \Big[ \lambda_{-} - \gamma_\mathrm{eff} \lambda_{+} \Big] \norm{\psi}^2, \\
    &\geq \frac{ \beta_\theta (1+r_\theta)}{2} \left[ \frac{\sigma_{-}^2 - \gamma_{\mathrm{eff}}\, \sigma_{+}^2}{\sigma_{-}^2\, \sigma_{+}^2} \right] \frac{\sigma_{-}^4}{\beta_\theta^2}\, \norm{h}^2 \\
    &= \frac{ (1+r_\theta)}{2 \beta_\theta} \left[ \sigma_{-}^2 - \gamma_{\mathrm{eff}}\, \sigma_{+}^2 \right] \frac{\sigma_{-}^2}{\sigma_{+}^2}\, \norm{h}^2 \\
    &= \frac{\rho_\theta}{1-\gamma}\, \norm{h}^2, \quad \text{by the definition of $\rho_\theta$ in \eqref{eq:Xi_bound}, which is positive by Assumption~\ref{assumption:M}(ii).}
\end{align*}

\emph{Term \textcircled{II}.} By part (3) of Lemma~\ref{lemma:schur_coercive}, $\norm{S^{-1}} \le \frac{1}{1-\gamma}$, so $\text{\textcircled{II}} \ge -\frac{\norm{\Xi_\theta}}{1-\gamma}\, \norm{h}^2$.
 
\emph{Conclusion.} Combining the two bounds in \eqref{eq:vw_split},
\begin{equation}\label{eq:single_data_descent}
    \inner{v_{\theta,x}}{w_{\theta,x}} \;\ge\; \frac{\rho_\theta - \norm{\Xi_\theta}}{1-\gamma}\, \norm{h_{\theta,x}}^2 \;\ge\; 0,
\end{equation}
where the last inequality holds by Assumption~\ref{assumption:M}(iii).
\end{proof}

\textbf{Remark.}
While Assumption~\ref{assumption:M}(i) and (ii) are standard conditions that have appeared in prior work
such as~\cite{gelphman2026convergence} and~\cite{gelphman2026end}, condition (iii) is specific to the presence of the safety layer. Intuitively, it requires that the auxiliary slack variable does not dominate the gradient alignment used to update the control policy. In other words, the slack-induced component of the gradient must remain sufficiently small so that the alignment between the true Clarke subgradient and the JFB approximation is governed primarily by the control variables. This is precisely what enables the pointwise alignment result of Lemma~\ref{lemma:integrand_inner_product}. Although this condition is generally intractable to verify analytically for a given problem, the numerical results in Figure~\ref{fig:descent} provide empirical evidence that it is satisfied in practice. Moreover, the assumption can be relaxed to hold only in expectation over the initial condition $x$, leading directly to Lemma~\ref{lemma:expected_inner_product}, although we do not pursue this generalization here in this work.

\subsection{Proof of \textbf{Lemma~\ref{lemma:expected_inner_product}}}
\lemexpected*

\begin{proof}
Taking expectation with respect to $x$ on both sides of~\eqref{eq:single_data_descent}, we have
\begin{align}
        \mathbb{E}_x\left[ \inner{v_{\theta,x}}{w_{\theta,x}} \right] &\geq \mathbb{E}_x \left[ \frac{\rho_\theta - \norm{\Xi_\theta}}{1-\gamma}\, \norm{h_{\theta,x}}^2 \right] \nonumber \\
        &= \frac{\rho_\theta - \norm{\Xi_\theta}}{1-\gamma} \; \mathbb{E}_x \left[ \norm{h_{\theta,x}}^2 \right] \nonumber \\
        &\geq  \frac{\rho_\theta - \norm{\Xi_\theta}}{1-\gamma} \; \norm{ \mathbb{E}_x \left[ h_{\theta,x}\right]}^2  , \quad \text{by Jensen's inequality.} \label{eq:expect_lower}
    \end{align}

    On the other hand, we derive an upper bound for $\mathbb{E}_x\left[ \inner{v_{\theta,x}}{w_{\theta,x}} \right] $.
    Let $E_v = \mathbb{E}_{x}[v_{\theta,x}]$ and $E_w = \mathbb{E}_x[w_{\theta,x}]$, we have 
    \begin{align}
        \mathbb{E}_x\left[ \inner{v_{\theta,x}}{w_{\theta,x}} \right] 
        &= \mathbb{E}_x[\inner{v_{\theta,x}-E_v + E_v}{w_{\theta,x}-E_w+E_w}] \nonumber \\
        &= \mathbb{E}_x[\inner{v_{\theta,x}-E_v}{w_{\theta,x}-E_w}] + \inner{\mathbb{E}_x[v_{\theta,x}-E_v]}{E_w} +\inner{E_v}{\mathbb{E}_x[w_{\theta,x}-E_w]} + \inner{E_v}{E_w}  \nonumber \\ 
        &= \mathbb{E}_x[\inner{v_{\theta,x}-E_v}{w_{\theta,x}-E_w}] + \inner{0}{E_w} + \inner{E_w}{0} + \inner{E_v}{E_w} \nonumber \\
&= \mathbb{E}_x[\inner{v_{\theta,x}-E_v}{w_{\theta,x}-E_w}] + \inner{E_v}{E_w} \nonumber \\
&\leq \sqrt{\mathbb{E}_x[\|v_{\theta,x}-E_v\|^2_2]}\sqrt{\mathbb{E}_x[\|w_{\theta,x}-E_w\|^2_2]} + \inner{E_v}{E_w}, \quad \text{by Cauchy-Schwarz inequality,} \nonumber \\
&= \sqrt{\text{Var}_x[v_{\theta,x}]}\sqrt{\text{Var}_x[w_{\theta,x}]} + \inner{E_v}{E_w} \nonumber \\ 
&\leq \text{max}\left(\sqrt{\text{Var}_x[v_{\theta,x}]},\sqrt{\text{Var}_x[w_{\theta,x}]}\right)^2 + \inner{E_v}{E_w} \nonumber \\ 
&\leq \delta_{var} \norm{ \mathbb{E}_x \left[ h_{\theta,x}\right]}^2 + \inner{E_v}{E_w}, \quad \text{{ by Assumption~\ref{assumption:expectation_integrand_inner_product}}.}\label{eq:expect_lower_interm}
    \end{align}

Rearranging \eqref{eq:expect_lower_interm}, we have
\begin{align*}
     \inner{\mathbb{E}_{x}[v_{\theta,x}]}{\mathbb{E}_x[w_{\theta,x}]} &= \inner{E_v}{E_w} \nonumber \\
     &\geq \mathbb{E}_x\left[ \inner{v_{\theta,x}}{w_{\theta,x}} \right]  - \delta_{var} \norm{ \mathbb{E}_x \left[ h_{\theta,x}\right]}^2 \nonumber \\
     &\geq  \frac{\rho_\theta - \norm{\Xi_\theta}}{1-\gamma} \;  \norm{ \mathbb{E}_x \left[ h_{\theta,x}\right]}^2 - \delta_{var} \norm{ \mathbb{E}_x \left[ h_{\theta,x}\right]}^2, \quad \text{by \eqref{eq:expect_lower},} \nonumber \\
     &= \left( \frac{\rho_\theta - \norm{\Xi_\theta}}{1-\gamma} - \delta_{var} \right) \norm{ \mathbb{E}_x \left[ h_{\theta,x}\right]}^2 \\
     &= \delta_{v, \theta}^2 \geq 0,
\end{align*}
where in the second last step, we substituted that $\delta_{v, \theta} := \sqrt{\frac{\rho_\theta - \norm{\Xi_\theta}}{1-\gamma} - \delta_{var}} \norm{ \mathbb{E}_x \left[ h_{\theta,x}\right]}$ from Assumption~\ref{assumption:expectation_timeaverage}.
\end{proof}

\subsection{Proof of Theorem~\ref{thm:expect_suff_descent}}
\thmdescent*

\begin{proof}
By Assumptions~\ref{assumption:T}.\ref{assm:continuous} and~\ref{assumption:T}.\ref{assm:lipschitz} and~\cite[Theorem~2.7.2]{clarke1990optimization}, fix any $\xi \in \partial_\theta^C \mathbb{E}_x [ J_x(\theta)]$, there exists a corresponding $\xi_{x, t}$ and measurable selection $v_{\theta, x}(t) \in \partial^C_\theta J_{x, t}(\theta)$ satisfying $\xi_{x, t} =v_{\theta, x}(t)$ for $x$ and $t$ almost everywhere, such that $\xi = \mathbb{E}_x \left[ \int_0^T \xi_{x, t} dt \right]$.
Let $C_v := \frac{1}{T} \int_0^T v_{\theta,x}(t) dt = \frac{1}{T} \int_0^T \xi_{x, t} dt$ and $ C_w := \frac{1}{T} \int_0^T w_{\theta,x}(t) dt$, we have
\begin{align}
&\int_{0}^{T}\mathbb{E}_x[\xi_{x, t}]^{\top}\mathbb{E}_x[w_{\theta,x}(t)]dt \nonumber \\ &=\int_{0}^{T}\mathbb{E}_x[\xi_{x, t} - C_{v} + C_{v}]^{\top}\mathbb{E}_x[w_{\theta,x}(t) - C_w + C_w]dt \nonumber \\
&= \int_{0}^{T}\mathbb{E}_x[\xi_{x, t} - C_{v}]^{\top}\mathbb{E}_x[w_{\theta,x}(t) - C_w]dt  \nonumber \\
&+\int_{0}^{T} \mathbb{E}_x[\xi_{x, t} - C_{v}]^{\top}\mathbb{E}_x[C_w]dt \nonumber\\
&+\int_{0}^{T}\mathbb{E}_x[C_{v}]^{\top}\mathbb{E}_x[w_{\theta,x}(t) - C_w]dt + T\mathbb{E}_x[C_{v}]^{\top}\mathbb{E}_x[C_w] \nonumber \\
\begin{split}\label{eq:expect_int_vw}
&= \int_{0}^{T}\mathbb{E}_x[\xi_{x, t} - C_{v}]^{\top}\mathbb{E}_x[w_{\theta,x}(t) - C_w]dt  \\
&+\int_{0}^{T} \mathbb{E}_x[\xi_{x, t} - C_{v}]^{\top}\mathbb{E}_x[C_w]dt \\
&+\int_{0}^{T}\mathbb{E}_x[C_{v}]^{\top}\mathbb{E}_x[w_{\theta,x}(t) - C_w]dt + \frac{1}{T} \langle \xi, \mathbb{E}_x[d_x^{JFB}(\theta)] \rangle.
\end{split}
\end{align}

Here, we evaluate the second term of~\eqref{eq:expect_int_vw} as
\begin{align*}
    \int_{0}^{T} \mathbb{E}_x[\xi_{x, t} - C_{v}]^{\top}\mathbb{E}_x[C_w]dt 
    &= \mathbb{E}_x[C_w]^\top \left(\int_{0}^{T}\mathbb{E}_x[\xi_{x, t} - C_{v}]dt\right)  \\
    &= \mathbb{E}_x[C_w]^\top \mathbb{E}_x \left[ \int_{0}^{T}\xi_{x, t} - C_{v} dt\right]  \\
    &= \mathbb{E}_x[C_w]^\top \mathbb{E}_x \left[ 0 \right]  \\
    &= 0,
\end{align*}
where in the first step, we used the independence of $\mathbb{E}_x[C_w]$ on $t$. And in the second step, we used Fubini's Theorem, because $\xi_{x, t}$ is integrable on $[0,T] \times \Omega$. And in the third step, we used $\int_0^T (\xi_{x, t} - C_{v}) dt = 0$.

The same arguments yield 
\begin{equation*}
\int_{0}^{T}\mathbb{E}_x[C_{v}]^{\top}\mathbb{E}_x[w_{\theta,x}(t) - C_w]dt=0.
\end{equation*}
Thus, the second and third terms of \eqref{eq:expect_int_vw} are zero, and \eqref{eq:expect_int_vw} becomes 
\begin{align*}
    &\int_{0}^{T}\mathbb{E}_x[\xi_{x, t}]^{\top}\mathbb{E}_x[w_{\theta,x}(t)]dt \\
    &= \int_{0}^{T}\mathbb{E}_x[\xi_{x, t} - C_{v}]^{\top}\mathbb{E}_x[w_{\theta,x}(t) - C_w]dt + \frac{1}{T} \langle \xi, \mathbb{E}_x[d_x^{JFB}(\theta)] \rangle.
\end{align*}

Rearranging yields
\begin{align}
\begin{split}\label{eq:inner_expectation}
    &\frac{1}{T} \langle \xi, \mathbb{E}_x[d_x^{JFB}(\theta)] \rangle \\
    &= \int_{0}^{T}\mathbb{E}_x[\xi_{x, t}]^{\top}\mathbb{E}_x[w_{\theta,x}(t)]dt - \int_{0}^{T}  \underbrace{\mathbb{E}_x[\xi_{x, t} - C_{v}]^{\top}\mathbb{E}_x[w_{\theta,x}(t) - C_w]}_{(\text{I})}dt.
\end{split}
\end{align}
We evaluate the integrand (I) as follows
\begin{align}
\mathbb{E}_x[\xi_{x, t} - C_{v}]^{\top}\mathbb{E}_x[w_{\theta,x}(t) - C_w] &= \mathbb{E}_x[v_{\theta,x}(t) - C_v]^{\top}\mathbb{E}_x[w_{\theta,x}(t) - C_w], \quad \text{for a.e. $t$,}\\
&\leq \| \mathbb{E}_x[v_{\theta,x}-C_v]\|_2 \|\mathbb{E}_x[w_{\theta,x}-C_w]\|_2 \nonumber \\
&\leq (a_v +\delta_v \inf_{\phi \in \partial^C_\theta \mathbb{E}_x [J_x(\theta)]} \norm{\phi} ) (a_w + \delta_w \left\| \mathbb{E}_x[\dJFB] \right\|_2) \nonumber \\
&\leq \max(a_v + \delta_v \inf_{\phi \in \partial^C_\theta \mathbb{E}_x [J_x(\theta)]} \norm{\phi}, a_w +\delta_w \left\| \mathbb{E}_x[\dJFB] \right\|_2)^2 \nonumber \\
&\leq \delta_{v, \theta}^{2} - \frac{\epsilon_v}{T^2}  \inf_{\phi \in \partial^C_\theta \mathbb{E}_x [J_x(\theta)]} \norm{\phi}^2,
\label{eq:delta_theta}
\end{align}
for almost everywhere $t$, where in the second step, we used Cauchy-Schwarz inequality. And in the subsequent steps, we used Assumption~\ref{assumption:expectation_timeaverage}.

Substituting \eqref{eq:delta_theta} into~\eqref{eq:inner_expectation} yields
\begin{align*}
    & \frac{1}{T} \langle \xi, \mathbb{E}_x[d_x^{JFB}(\theta)] \rangle \\
    &\geq \int_{0}^{T}\mathbb{E}_x[v_{\theta,x}(t)]^{\top}\mathbb{E}_x[w_{\theta,x}(t)]dt - \int_0^T \left(\delta_{v, \theta}^{2} - \frac{\epsilon_v}{T^2}  \inf_{\phi \in \partial^C_\theta \mathbb{E}_x [J_x(\theta)]} \norm{\phi}^2 \right)  dt \\
    &\geq \int_{0}^{T} \delta_{v, \theta}^2 dt - \int_0^T \left(\delta_{v, \theta}^{2} - \frac{\epsilon_v}{T^2}  \inf_{\phi \in \partial^C_\theta \mathbb{E}_x [J_x(\theta)]} \norm{\phi}^2 \right)  dt,\quad \text{by \eqref{eq:expect_descent} of Lemma~\ref{lemma:expected_inner_product},} \\
    &= \frac{\epsilon_v}{T^2} \int_0^T   \inf_{\phi \in \partial^C_\theta \mathbb{E}_x [J_x(\theta)]} \norm{\phi}^2  dt \\
    &= \frac{\epsilon_v}{T}    \inf_{\phi \in \partial^C_\theta \mathbb{E}_x [J_x(\theta)]} \norm{\phi}^2.
\end{align*}

Multiplying $T$ to both sides, we obtain
\begin{align*}
    \langle \xi, \mathbb{E}_x[d_x^{JFB}(\theta)] \rangle \geq \epsilon_v \inf_{\phi \in \partial^C_\theta \mathbb{E}_x [J_x(\theta)]} \norm{\phi}^2,
\end{align*}
which holds for every $\xi \in \partial_\theta^C \mathbb{E}_x [ J_x(\theta)]$.
\end{proof}

\section{Proof of Convergence}
\subsection{Proof of Theorem~\ref{thm:gradient_flow_converge}}
\thmconvergence*

\begin{proof}
    \textbf{Step 1 (MVT):} Let $F(\theta):=\mathbb{E}_x[J_x(\theta)]$.
    By the Lebourg's mean-value theorem~\cite[Theorem~2.3.7]{clarke1990optimization}, we have
    \begin{equation}\label{eq:MVT}
        F(\theta(\tau+\Delta \tau)) - F(\theta(\tau)) = \inner{\hat{\xi}(\Delta \tau)}{\theta(\tau+\Delta \tau)-\theta(\tau)},
    \end{equation}
    for some $\hat{\xi}(\Delta \tau) \in \partial_\theta^C \mathbb{E}_x [J_x(\hat{\theta})]$, where $\hat{\theta}$ is on the line segment between $\theta(\tau+\Delta \tau)$ and $\theta(\tau)$.

    \textbf{Step 2 (Upper Right Dini Derivative):} From Assumption~\ref{assumption:T}.\ref{assm:bounded}, ${d^{\mathrm{JFB}}_x(\theta)}$ is uniformly bounded over $x$ and $\theta$. This implies that $\mathbb{E}_x \left[ d_x^{JFB}(\theta) \right]$ is uniformly bounded for all $\theta$. By \eqref{eq:gradient_flow} and the uniform boundedness of $\mathbb{E}_x \left[ d_x^{JFB}(\theta) \right]$, $\theta(\tau)$ is Lipschitz continuous. Thus, by the Lipschitz continuity of $\theta(\tau)$ and of $F$ in $\theta$ (Assumption~\ref{assumption:T}.\ref{assm:smooth}), the upper right Dini derivative 
    \begin{equation}\label{eq:Dini}
        D^+ F(\theta(\tau)) := \limsup_{\Delta \tau \to 0^+} \frac{F(\theta(\tau+\Delta \tau)) - F(\theta(\tau))}{\Delta \tau}
    \end{equation}
    exists and is finite.

    Combining \eqref{eq:MVT} and \eqref{eq:Dini}, we have
    \begin{align}
        D^+ F(\theta(\tau)) &= \limsup_{\Delta \tau \to 0^+} \frac{F(\theta(\tau+\Delta \tau)) - F(\theta(\tau))}{\Delta \tau} \\
        &= \limsup_{\Delta \tau \to 0^+} \inner{\hat{\xi}(\Delta \tau)}{\frac{\theta(\tau+\Delta \tau)-\theta(\tau)}{\Delta \tau}}.
    \end{align}
    Moreover, by the sequential characterization of limit supremum, there exists a sequence $\{ h_i \}_{i=0}^\infty$ such that $h_i \to 0^+$ as $i \to \infty$, and
    \begin{equation}\label{eq:Dini_seq}
        D^+ F(\theta(\tau)) = \lim_{i \to \infty} \inner{\hat{\xi}(h_i)}{\frac{\theta(\tau+h_i)-\theta(\tau)}{h_i}}.
    \end{equation}

    Since $F$ is Lipschitz continuous in $\theta$ (Assumption~\ref{assumption:T}.\ref{assm:smooth}), $\hat{\xi}$ is bounded~\cite[Proposition~2.1.2(a)]{clarke1990optimization}. By the Bolzano-Weierstrass theorem~\cite[Theorem~3.4.8]{bartle2000introduction}, there exists a subsequence $\{ h_{i_k} \}_{k=0}^\infty$ of $\{ h_i \}_{i=0}^\infty$ such that $\lim\limits_{k \to \infty} \hat{\xi}(h_{i_k})$ converges. Thus, by~\cite[Theorem~3.4.2]{bartle2000introduction}, \eqref{eq:Dini_seq} can be rewritten as
    \begin{equation}\label{eq:Dini_subseq} 
        D^+ F(\theta(\tau)) = \lim_{k \to \infty} \inner{\hat{\xi}(h_{i_k})}{\frac{\theta(\tau+h_{i_k})-\theta(\tau)}{h_{i_k}}}.
    \end{equation}

    We evaluate the limits of the two functions inside the inner product. On one hand, let $\xi:=\lim\limits_{k \to \infty} \hat{\xi}(h_{i_k})$. Since $\hat{\xi}(h_{i_k}) \in \partial_\theta^C \mathbb{E}_x [J_x(\hat{\theta}_{i_k})]$, where $\hat{\theta}_{i_k}$ is on the line segment between $\theta(\tau)$ and $\theta(\tau + h_{i_k})$, and $\lim\limits_{k \to \infty}\hat{\theta}_{i_k}=\theta(\tau)$, by~\cite[Theorem~2.1.5(b)]{clarke1990optimization},
    \begin{equation}\label{eq:k_lim_xi}
        \xi \in \partial_\theta^C \mathbb{E}_x [J_x(\theta(\tau))].
    \end{equation}
    On the other hand, By~\eqref{eq:gradient_flow},
    \begin{align}\label{eq:k_lim_theta}
        \lim_{k \to \infty}{\frac{\theta(\tau+h_{i_k})-\theta(\tau)}{h_{i_k}}} = \frac{d\theta(\tau)}{d\tau} = - \mathbb{E}_x \left[ d_x^{JFB}(\theta(\tau)) \right].
    \end{align}

    \textbf{Step 3 (A.E. Decreasing Objective):} Thus, we can continue to evaluate~\eqref{eq:Dini_subseq} as follows
    \begin{align}
        D^+ F(\theta(\tau)) &= \lim_{k \to \infty} \inner{\hat{\xi}(h_{i_k})}{\frac{\theta(\tau+h_{i_k})-\theta(\tau)}{h_{i_k}}} \nonumber \\
        &= \inner{\lim_{k \to \infty} \hat{\xi}(h_{i_k})}{\lim_{k \to \infty} \frac{\theta(\tau+h_{i_k})-\theta(\tau)}{h_{i_k}}}, \quad \text{by continuity of inner product,} \nonumber \\
        &= \inner{\xi}{- \mathbb{E}_x \left[ d_x^{JFB}(\theta(\tau)) \right]}, \quad \text{by \eqref{eq:k_lim_xi} and \eqref{eq:k_lim_theta},} \nonumber \\
        & \leq - \epsilon_{v} \cdot \inf_{\phi \in \partial_\theta^C \mathbb{E}_x [J_x(\theta(\tau))]} \norm{\phi}^2, \quad \text{{by Theorem~\ref{thm:expect_suff_descent}}} \label{eq:Dini_suff_descent}.
    \end{align}

    This implies a sufficient decrease in the objective function in terms of the upper-right Dini derivative.

      \textbf{Step 4 (A.E. Differentiability):}  By Assumption~\ref{assumption:T}.\ref{assm:smooth}, $F(\theta)$ is Lipschitz continuous with respect to $\theta$. Moreover, by~\eqref{eq:gradient_flow}, $\theta(\tau)=- \int_0^\tau \mathbb{E}_x \left[ d_x^{JFB}(\theta(\sigma)) \right] d\sigma$ is given by an indefinite integral, and thus it is absolutely continuous~\cite[Theorem~5.4.14]{royden1988real}. 
    
    Hence, $F(\theta(\tau))$ is absolutely continuous on every compact subinterval of $[0,\infty)$~\cite[Exercise~17\&20]{royden1988real}. This implies that {$F(\theta(\tau))$ is differentiable almost everywhere} on every subinterval $[n, n+1]$ and thus $[n, n+1)$ for all $n\in\mathbb{N}$~\cite[Corollary~5.4.12]{royden1988real}. In other words, the subset of $[n, n+1)$ on which $F(\theta(\tau))$ is nondifferentiable has Lebesgue measure zero. 
    
    Since $[0, \infty)=\cup_{n\in\mathbb{N}}[n, n+1)$, by the countable additivity of Lebesgue measure~\cite[Chapter~1]{royden1988real}, the subset of $[0, \infty)$ on which $F(\theta(\tau))$ is nondifferentiable also has Lebesgue measure zero. This means $F(\theta(\tau))$ is differentiable almost everywhere.
  
  \textbf{Step 5 (Convergence):} 
    When $F(\theta(\tau))$ is differentiable, its derivative coincides with the Dini derivative $D^+F(\theta(\tau))$~\cite[Chapter~1]{royden1988real}. Thus, by~\cite[Corollary~5.4.15]{royden1988real}, for any $\tau>0$,
    \begin{align*}
    F(\theta(\tau)) - F(\theta(0)) &= \int_0^\tau D^+ F(\theta(\sigma)) d\sigma \\
    & \leq \int^\tau_0 - \epsilon_{v} \cdot \inf_{\phi \in \partial_\theta^C \mathbb{E}_x [J_x(\theta(\sigma))]} \norm{\phi}^2 d\sigma, \quad \text{by \eqref{eq:Dini_suff_descent},}\\
    &=- \epsilon_{v} \int^\tau_0 \inf_{\phi \in \partial_\theta^C \mathbb{E}_x [J_x(\theta(\sigma))]} \norm{\phi}^2 d\sigma.
    \end{align*}

   We remark that $\epsilon_v >0$ and $\inf_{\phi \in \partial_\theta^C \mathbb{E}_x [J_x(\theta(\tau))]} \norm{\phi}^2 \geq 0$. Moreover, $F$ is bounded from below (Assumption~\ref{assumption:T}.\ref{assm:smooth}), thus $\lim\limits_{\tau \to \infty} \int^\tau_0 \inf\limits_{\phi \in \partial_\theta^C \mathbb{E}_x [J_x(\theta(\sigma))]} \norm{\phi}^2 d\sigma < \infty$. This implies $\liminf\limits_{\tau \to \infty} \inf\limits_{\phi \in \partial_\theta^C \mathbb{E}_x [J_x(\theta(\tau))]} \norm{\phi} = 0$.

\end{proof}

\textbf{Remark.}
The convergence result is established under continuous-time gradient flow rather than discrete gradient descent. The primary technical obstacle is the nonsmoothness introduced by the projection operator, which renders the standard Taylor-expansion arguments underlying most gradient descent convergence analyses inapplicable. Instead, our proof relies on a novel analysis of Clarke subgradients in a neighborhood of the gradient-flow trajectory, allowing us to establish monotonic descent despite the set-valued nature of the generalized Jacobian. To the best of our knowledge, this constitutes the first convergence analysis for end-to-end policy learning with embedded nonsmooth CBF projection layers.

\begin{figure}[t]
    \centering
    \begin{subfigure}[b]{0.4\textwidth}
        \centering
        \includegraphics[width=\textwidth]{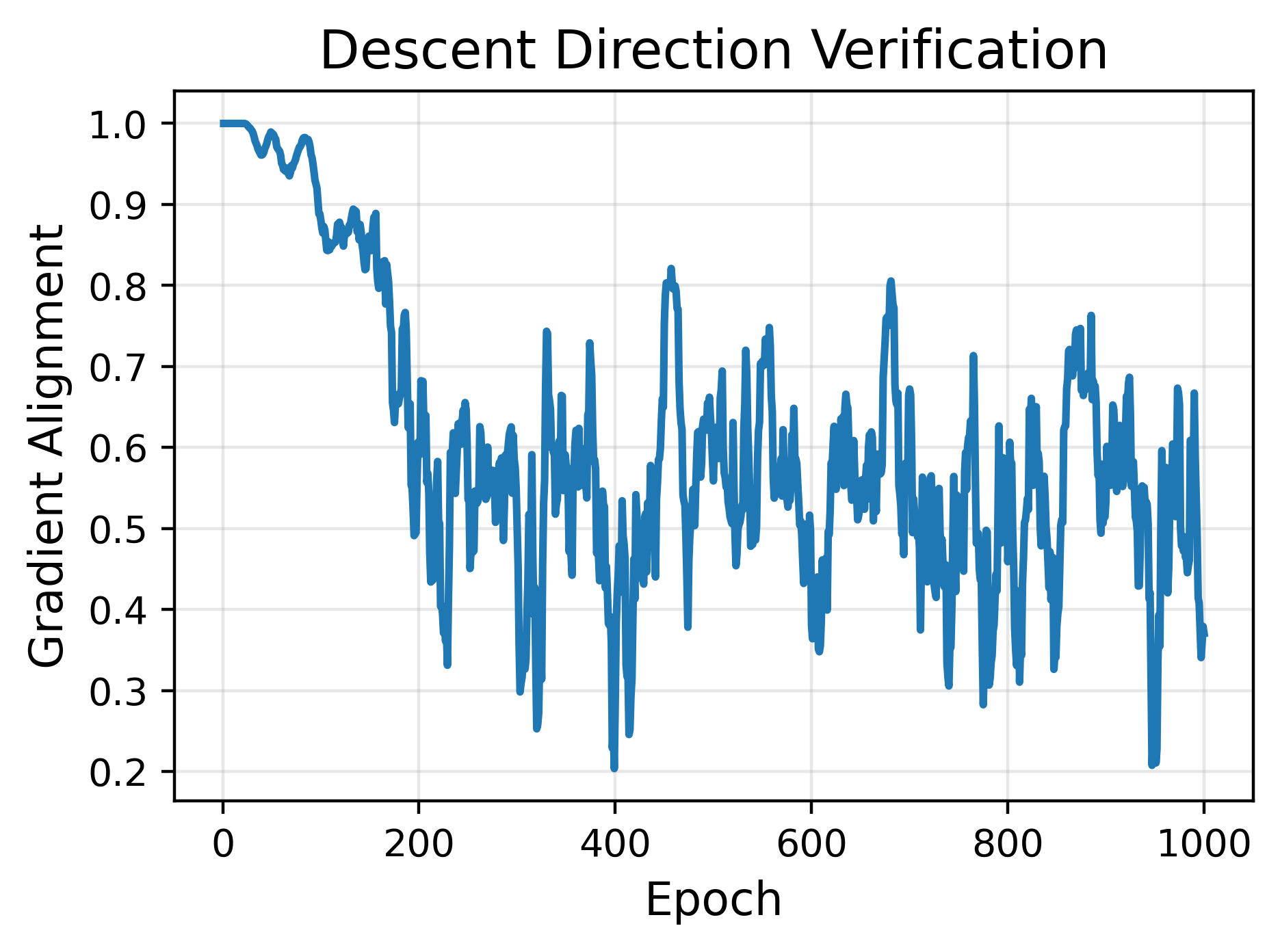}
        \caption{Descent direction verification}
        \label{fig:descent}
    \end{subfigure}
    \begin{subfigure}[b]{0.4\textwidth}
        \centering
        \includegraphics[width=\textwidth]{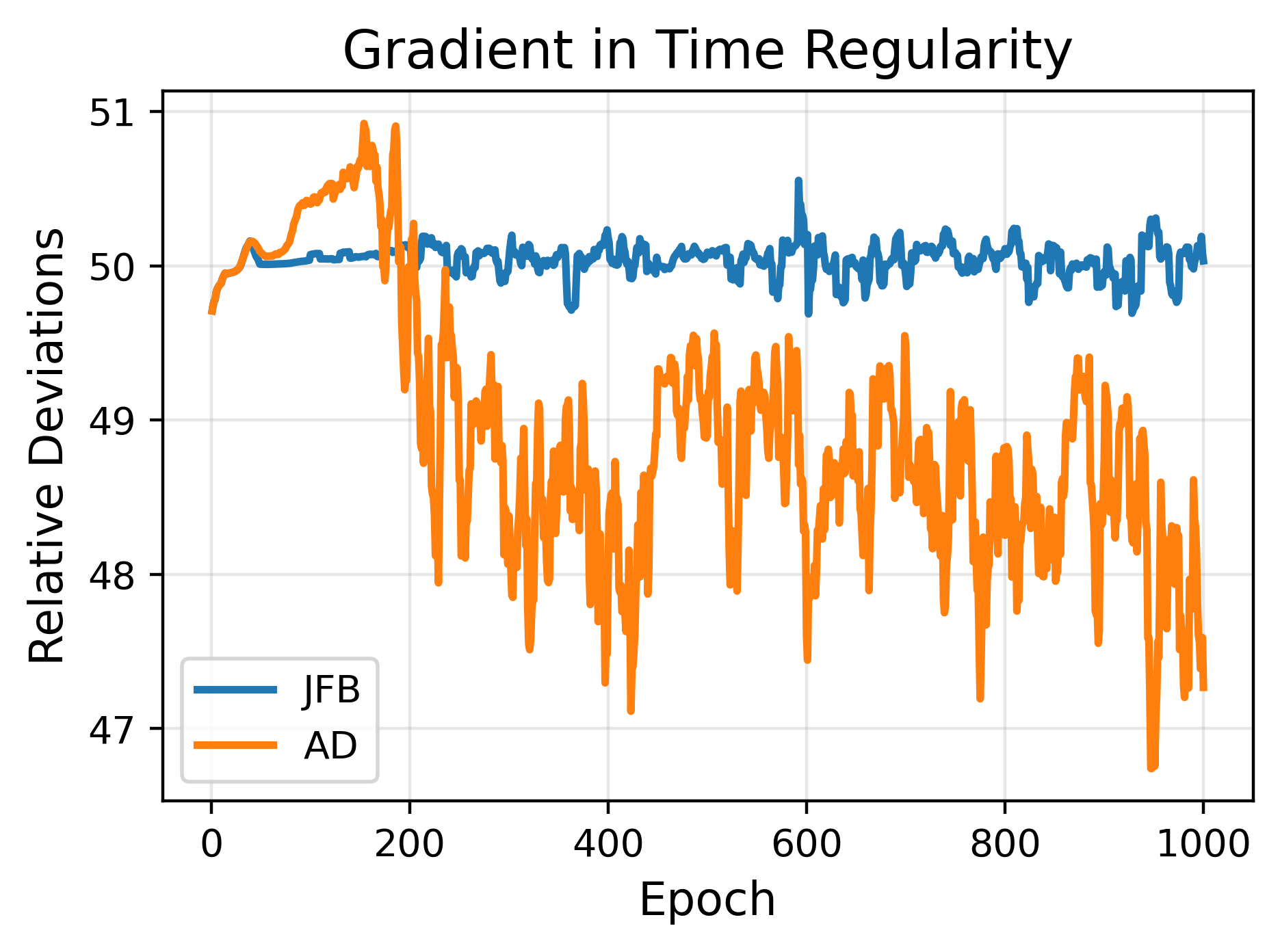}
        \caption{Gradient regularity along trajectories}
        \label{fig:grad_reg}
    \end{subfigure}
    \caption{Numerical verification of the theoretical results and key assumptions. \textbf{Left:} gradient alignment results between JFB and AD during training, while the alignment signal varies during training, we notice it remains positive during the run, indicating descent direction and serves as a direct verification of Theorem~\ref{thm:expect_suff_descent}. \textbf{Right:} Gradient regularity results along trajectories, here we establish numerical support for Assumption~\ref{assumption:expectation_timeaverage}, the time average of the gradients for both JFB and AD remain bounded during training, hence supporting the validity of the assumption. }
    \label{fig:theory_validation}
\end{figure}

\section{Details on Numerical Experiments}
\label{sec:example_details}
In this section we provide additional details on the experiment setup to ensure reader understanding and reproducibility. Source code for reproducing the results will be released upon publication of the work. 

\subsection{Dynamics and Optimal Control Objectives}
We first present necessary details on the dynamics models used for the work, as well as the control objectives that define each corresponding example. 
In total we consider three different dynamics models, namely
\begin{itemize}
    \item \textbf{Single Integrator Dynamics:} The state evolves according to
    \[
    \frac{d}{dt} z(t) = u(t),
    \]
    where \( z(t) \in \mathbb{R}^3 \) denotes the state and \( u(t) \in \mathbb{R}^3 \) is the control input, which directly specifies the velocity of the system over a finite time horizon.

    \item \textbf{Double Integrator Dynamics:} The state evolves according to
    \[
    \frac{d^2}{dt^2} z(t) = u(t),
    \]
    or equivalently in first-order form,
    \[
    \frac{d}{dt}
    \begin{bmatrix}
        z(t) \\
        v(t)
    \end{bmatrix}
    =
    \begin{bmatrix}
        v(t) \\
        u(t)
    \end{bmatrix},
    \]
    where \( z(t) \in \mathbb{R}^2 \) denotes the position, \( v(t) = \displaystyle\frac{d}{dt}z(t) \in \mathbb{R}^2 \) the velocity, and \( u(t) \in \mathbb{R}^2 \) the control input representing acceleration.
    
    \item \textbf{Quadcopter Dynamics:} We consider the following quadcopter dynamics also used in~\cite{onken2022neural,li2025zero} with $12$-dimensional state with dynamics function $f$ defined as
\[
f(t,z)=
\begin{bmatrix}
z_7\\
z_8\\
z_9\\
z_{10}\\
z_{11}\\
z_{12}\\
0\\
0\\
-\bar{g}\\
0\\
0\\
0
\end{bmatrix},
\]
and
\[
g(t,z)=
\begin{bmatrix}
0 & 0 & 0 & 0\\
0 & 0 & 0 & 0\\
0 & 0 & 0 & 0\\
0 & 0 & 0 & 0\\
0 & 0 & 0 & 0\\
0 & 0 & 0 & 0\\
\displaystyle\frac{\sin(z_4)\sin(z_6)+\cos(z_4)\sin(z_5)\cos(z_6)}{\bar{m}} & 0 & 0 & 0\\
\displaystyle\frac{-\cos(z_4)\sin(z_6)+\sin(z_4)\sin(z_5)\cos(z_6)}{\bar{m}} & 0 & 0 & 0\\
\displaystyle\frac{\cos(z_5)\cos(z_6)}{\bar{m}} & 0 & 0 & 0\\
0 & 1 & 0 & 0\\
0 & 0 & 1 & 0\\
0 & 0 & 0 & 1
\end{bmatrix}.
\]
The controls for the problem are $u = [u_1,u_2,u_3,u_4]^\top \in \mathbb{R}^{4}$. We assume that both the mass $\bar{m}$ and gravity $\bar{g}$ are given and remain fixed for all examples. Notably, the quadcopter dynamics are nonlinear in the state; this, combined with the relatively high state dimension, adds to the complexity of the control problem.
\end{itemize}

We consider multi-agent control problems by treating the collection of all agent states as a single joint state. Specifically, if $z^{(i)}(t) \in \mathbb{R}^n$ denotes the state of agent $i$, then for a total of $N$ agents we define the joint state
\[
z(t) =
\begin{bmatrix}
z^{(1)}(t) \\
z^{(2)}(t) \\
\vdots \\
z^{(N)}(t)
\end{bmatrix}
\in \mathbb{R}^{Nn}.
\]
The joint control $u(t)$ and the corresponding control-affine dynamics are defined analogously through concatenation. With this convention, the multi-agent system is treated as a single high-dimensional control system, and we continue to use the notation $z$ and $u$ throughout for simplicity. Although the dynamics are presented independently of the controller, the primary goal of our algorithm is to recover a centralized feedback policy acting on the joint state.

Given an initial-state distribution $x \sim \rho$ and a desired target state
\[
z_{\mathrm{target}} =
\begin{bmatrix}
z^{(1)}_{\mathrm{target}} \\
z^{(2)}_{\mathrm{target}} \\
\vdots \\
z^{(N)}_{\mathrm{target}}
\end{bmatrix},
\]
we consider the quadratic running cost
\[
L(t, z, u) = \frac{1}{2}\|u\|^2.
\]
For systems with explicit velocity states, such as the double integrator, we additionally penalize the velocity to encourage smooth trajectories,
\[
L(t, z, u) = \frac{1}{2}\|u\|^2 + \frac{1}{2}\|v\|^2.
\]
Similarly, the terminal cost is defined as
\[
G(z(T)) = \frac{1}{2}\|z(T) - z_{\mathrm{target}}\|^2.
\]
These objective functions are standard choices in the optimal control literature. For simplicity of presentation, we omit the weighting coefficients used in practice. In our experiments, a larger terminal cost weight is typically employed to encourage convergence to the desired target states, with the specific values chosen according to each problem.

\subsection{Implementation Details and Hyperparameters}
We implement our proposed approach using PyTorch~\cite{paszke2019pytorch}. For comparative study, we also use the CVXPY Layer library~\cite{agrawal2019differentiable}. All experiments are conducted on the same system, which consists of a shared NVIDIA RTX 5090 GPU with VRAM capped at $16$GB and an AMD Ryzen PRO 7975WX CPU. Numerical results are seeded to ensure reproducibility.

We use the Adam optimizer~\cite{kingma2014adam} for all experiments with initial learning rate set at $0.001$ across the board. 
For the DYS projection layer, we choose update step size $\zeta = 0.5$, maximum iteration number $5000$ and early stopping criteria $\texttt{tol} = 0.005$. These settings are kept fixed for all experiments. 
We consider $T = 10$ and $dt = 0.2$ for all experiments and use Runge–Kutta 4 (RK4)~\cite{butcher2016numerical} for time integration of the forward dynamics. We note that this leads to $50$ time steps for each trajectory rollout, longer than most existing work that shares similar ideas.
Similarly, we fix the batch size at $32$ for all examples. 
We parameterize the policy network using a ResNet~\cite{he2016deep}-based architecture as the main building block, as such architectures have been shown to be numerically stable over long-horizon simulations and are commonly used in learning-based control applications.
We list problem-specific hyperparameters such as network size, weight decay strength, total number of training epochs in Table~\ref{tab:hyperparameters}. 

As is the case in both learning and classical control applications~\cite{ames2019control,li2022optimizable,amos2017optnet}, incorporating a projection-based safety filter in the loop leads to a more challenging optimization problem. As such, to ensure stable training, we employ a continuation strategy on the terminal cost weight, initializing it at a relatively small value and progressively increasing it during training. We find this approach useful for mitigating early optimization instability. To ensure a fair comparison, the same setting is applied across all experiments on all methods compared.

\begin{table}[t]
\centering
\begin{tabular}{ccccc}
\hline
Problem (\# of agents) & \# of epochs & \# of network weights  & Weight decay &  Obstacle radius    \\ \hline \hline
Single Integrator ($50)$ & $3500$ & $354582$ & $0.001$ & $0.5, 0.7$  \\ \hline
Double Integrator ($1$) & $1000$ & $25474$ & $0.0001$ & $0.3$ \\ \hline
Double Integrator ($6$) & $5000$ & $103948$ & $0.001$ & $0.35$  \\ \hline
Quadcopter ($5$) & $2800$ & $109588$ & $0.001$ &  $0.35$  \\ \hline
Quadcopter ($30$) & $4500$ & $193912$ & $0.001$ & $0.35$ \\ \hline
Quadcopter ($100$) & $5000$ & $604432$ & $0.001$ & $0.35$ \\ \hline
\end{tabular}
\caption{Problem-specific hyperparameters for each numerical experiment. Note that, to ensure a fair comparison, we use the same training configuration across different models for each example.}
\label{tab:hyperparameters}
\end{table}

\subsection{CVXPY Layers Baselines and Regularizing the QP}
The CVXPY Layers baseline differentiates through the QP solution via 
the implicit function theorem (IFT) applied to the KKT conditions of 
the problem at optimality~\cite{agrawal2019differentiable}. 
The standard hard CBF-QP safety filter takes the form of~\eqref{eq:Axbconstraint} and~\eqref{eq:projection}, which we combine and restate here.
\begin{equation}
u^\star = \arg\min_{u} \|u - u_\theta\|^2 
\quad \text{s.t.} \quad 
A(z)u \le b(z),
\label{eq:hard_qp}
\end{equation}
where $A(z) \in \mathbb{R}^{c \times m}$ and $b(z) \in \mathbb{R}^c$ 
encode the stacked CBF constraints from \ref{subsec:DYS}, 
and $c$ is the total number of constraints across all agent-obstacle pairs, $u_\theta$ is the nominal output of the policy network. At each step the QP is solved using CVXPY Layers with the Splitting Conic Solver (SCS)~\cite{diamond2016cvxpy,agrawal2019differentiating} and is differentiable with respect to the policy output. Similar to the DYS, we use a uniform configuration of 
maximum iteration number $1000$ and early stopping criteria $\texttt{tol} = 0.001$ across all experiments. 

We notice during training as the size of the problem increases, we repeatedly encounter infeasibility warnings and solver failures, especially on the quadcopter examples. We also notice numerical instability that hinders performance of CVXPY Layers-based models.

To address this and ensure the comparison is as comprehensive as possible, we consider a modified problem by introducing a relaxation variable $\delta_i \ge 0$ 
per constraint and regularizing the QP objective directly, a standard technique for restoring feasibility 
of CBF-QPs~\cite{ames2019control, xiao2021sufficient, liu2023auxiliary}, yielding the regularized QP
\begin{equation}
u^\star = \arg\min_{u,\,\delta} \; 
\|u - u_\theta\|^2 + \lambda \sum_{i=1}^c \delta_i
\quad \text{s.t.} \quad 
A_i(z)u \le b_i(z) + \delta_i, \quad 
\delta_i \ge 0,
\label{eq:reg_qp}
\end{equation}
with penalty term $\lambda \gg 1$. This formulation is always feasible: 
$\delta_i$ absorbs any violation, and the large penalty drives 
$\delta \to 0$ whenever the hard constraint is satisfiable, 
recovering the behavior of~\eqref{eq:hard_qp} in safe regions. 
The KKT system for IFT differentiation is therefore well-defined 
everywhere, since the augmented problem always has a unique solution. In our experiments, we set $\lambda = 10,000$ throughout, While this alleviates some of the numerical instability, we note that the solution to the regularized problem is generally suboptimal with respect to the original control problem. We also tested multiple learning rates ranging from $0.01$ and $1\to,es 10^{-5}$ as well as a learning rate scheduler, to ensure that scaling and normalization were not the cause of the observed instability. 

Despite the regularization and our best effort in tuning the models, the CVXPYLayers baseline still 
struggles as problem dimensionality increases, as shown in Table~\ref{tab:results_obj}, Table~\ref{tab:results_mem} and Figure~\ref{fig:training_time}. In particular we notice blowup in quadcopter examples with $30$ and $100$ agents whether we regularize the QP or not. 
We attribute this to ill-conditioning of the KKT linear system 
during backpropagation~\cite{bai2020multiscale,bai2021stabilizing}, as inaccuracies in gradients get amplified when stacked through the whole trajectories. 
Increasing the number of agents and constraints can further exacerbate the problem: 
for $N$ agents and $c$ obstacles, the total constraint count grows at factor of
$N\cdot c$, and the KKT system scales accordingly. 
In contrast, the JFB update avoids this linear system solve 
entirely, backpropagating only through the final DYS iteration 
and maintaining stable gradients regardless of problem scale, 
as confirmed by the time and memory results.

\subsection{Training Time Comparisons}
\label{sec:results_time}
Having controllable training time is crucial for scaling any of these approaches to high-dimensional problems. We provide a comparative study across all methods on different examples. The results are shown in Figure~\ref{fig:training_time}. Note that, in all examples, DYS with JFB training has the lowest computation time per epoch. This is due to the fact that we use inexact gradients for JFB training, which significantly reduces the cost of tracking deep computational graphs, especially in cases where QP projections are frequently active along the trajectories.

The computation time of CVXPY Layers remains stable, yet high, when the method works, indicating an advantage of operator-splitting methods in high-dimensional problems: replacing the more expensive interior-point method with much cheaper projections allows for more stable behavior even in large-scale problems.
We do however, want to note that the convergence behavior of different methods can vary. While our method provides a general framework for problems of arbitrary dimensionality and complexity, for smaller-scale problems a CVXPY Layers-based solver can often be sufficient in terms of both accuracy and speed.

\begin{figure}[t]
    \centering
    \begin{subfigure}[b]{0.32\textwidth}
        \centering
        \includegraphics[width=\textwidth]{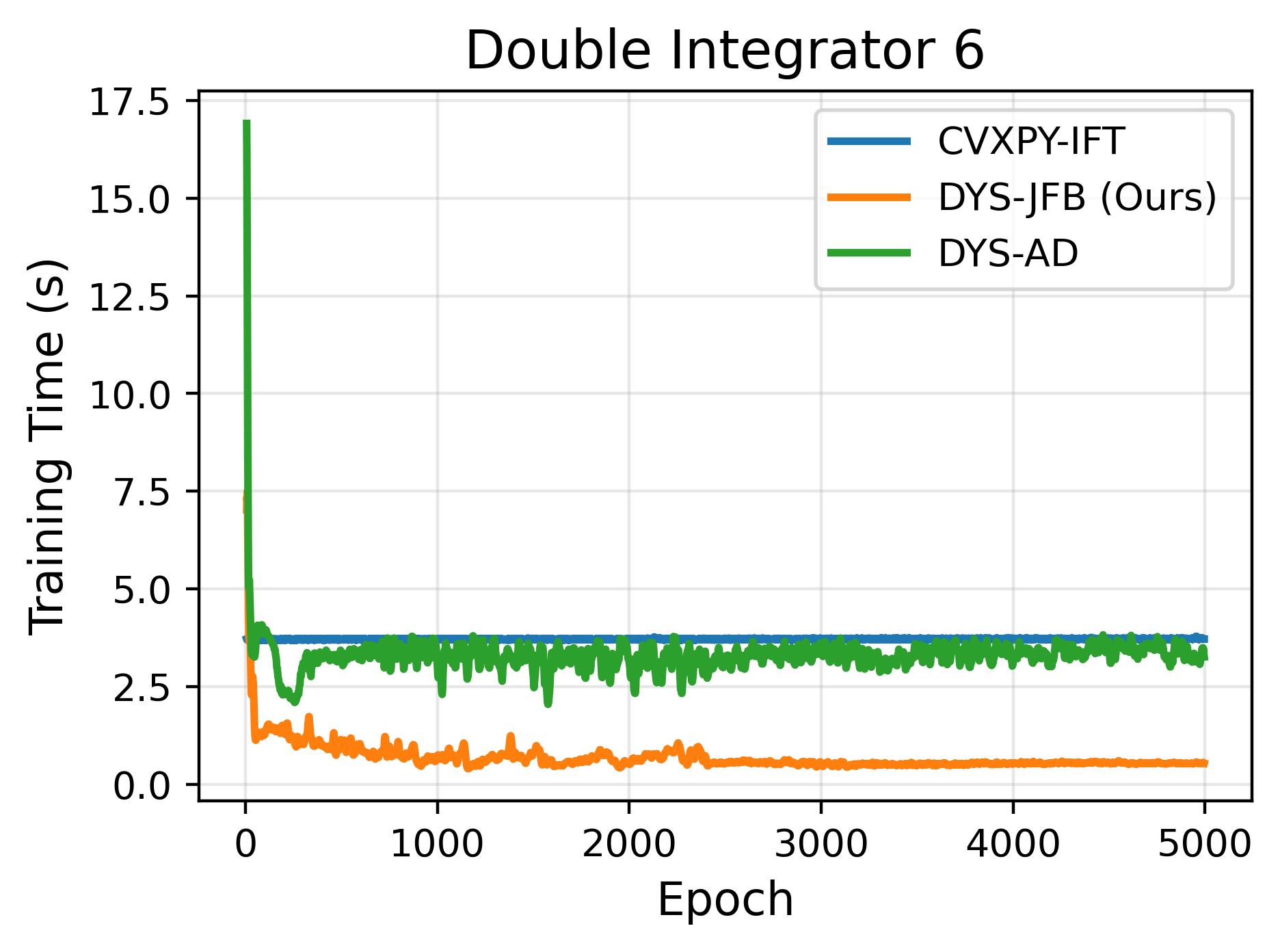}
        \caption{Double integrator 6}
        \label{fig:di_time}
    \end{subfigure}
    \hfill
    \begin{subfigure}[b]{0.32\textwidth}
        \centering
        \includegraphics[width=\textwidth]{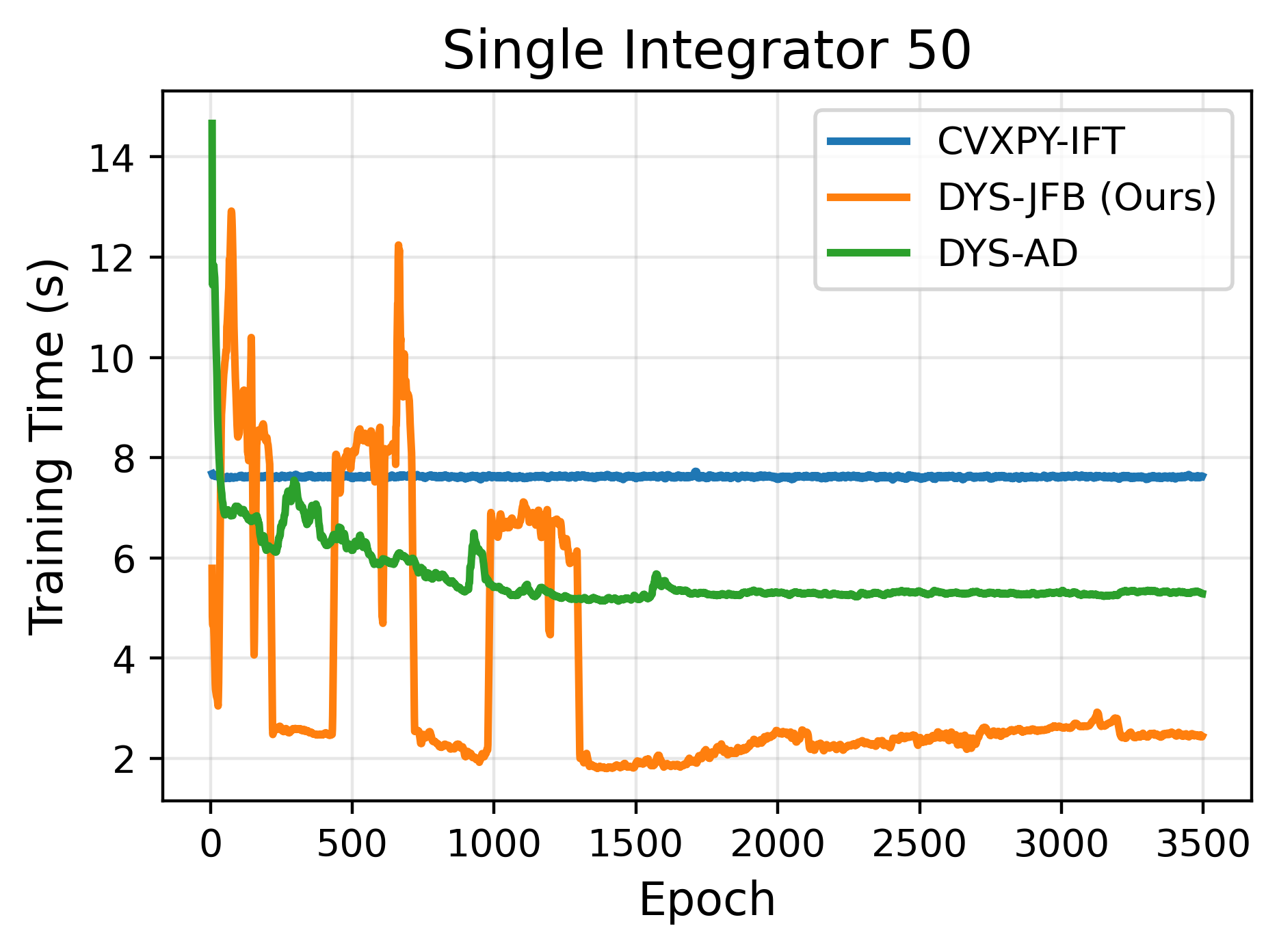}
        \caption{Single integrator 30}
        \label{fig:si_time}
    \end{subfigure}
    \hfill
    \begin{subfigure}[b]{0.32\textwidth}
        \centering
        \includegraphics[width=\textwidth]{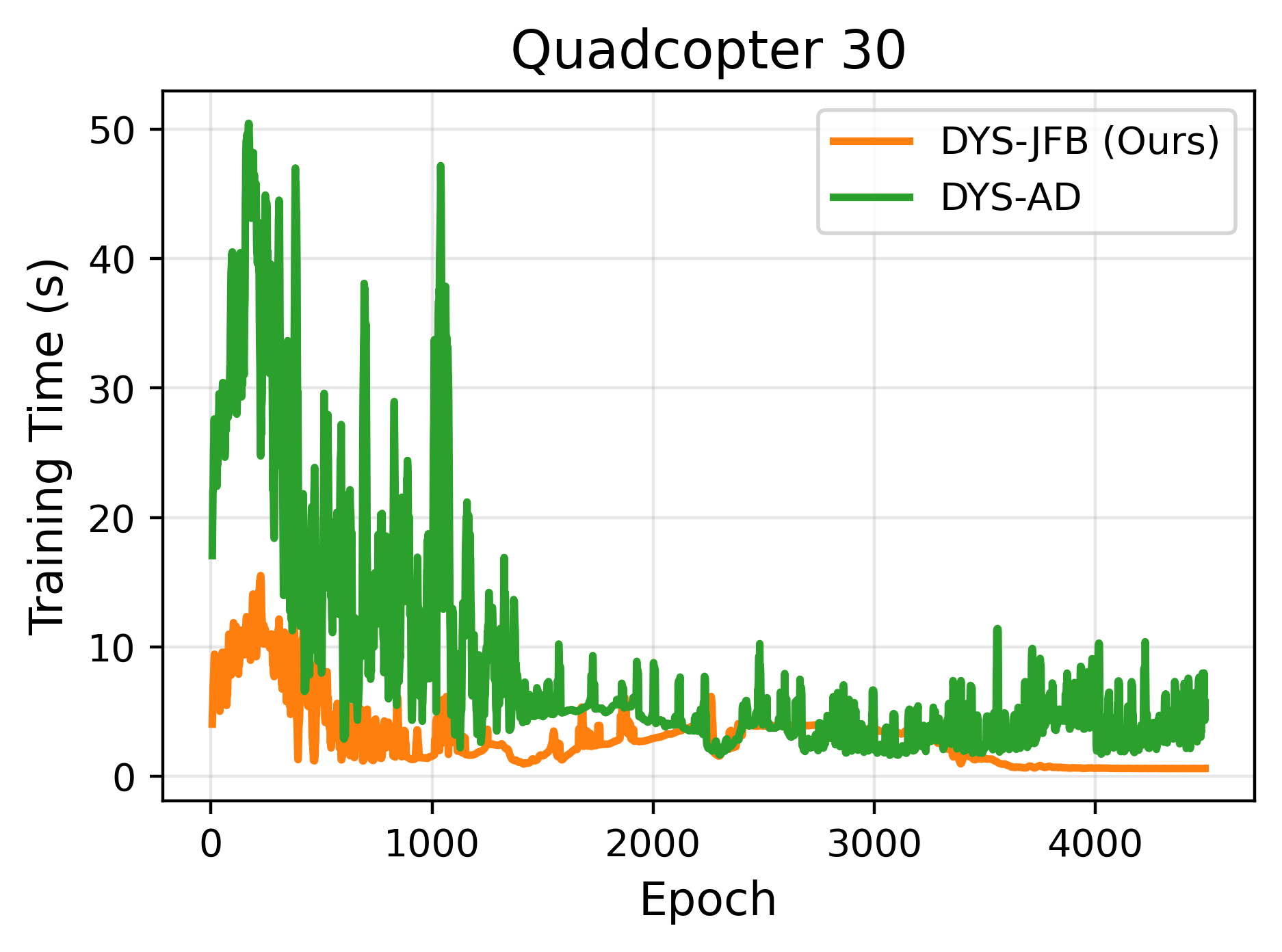}
        \caption{Quadcopter 30}
        \label{fig:quad_time}
    \end{subfigure}
    \caption{Per epoch training time results compared across all tested method under three different dynamics models. In all test cases DYS with JFB update yields the lowest per epoch time. While automatic differentiation can usually work it incurs significant cost overhead. Solver using CVXPY Layers does not converge for the quadcopter example.  We also note that, for DYS, the number of iterations varies depending on trajectory behavior. As such, the per-epoch training time can fluctuate, unlike that of CVXPY Layers; however, it typically decreases and stabilizes as training progresses.}
    \label{fig:training_time}
\end{figure}

\begin{table}[t]
\centering
\begin{tabular}{ccc}
\toprule
Problem & \makecell{Best Trained Model \\ ($L$, $G$)} & \makecell{Final Trained Model \\ ($L$, $G$)}\\
\midrule
\makecell{Single Integrator\\(50, 2)}
& ($5.48 \pm 0.18$, $0.0107 \pm 0.0050$)
& ($5.49 \pm 0.21$, $0.0189 \pm 0.0132$) \\

\makecell{Double Integrator\\(6, 3)}
& ($90.39 \pm 7.80$, $0.2639 \pm 0.1215$)
& ($90.67 \pm 7.43$, $0.4334 \pm 0.2405$) \\

\makecell{Quadcopter\\(30, 3)}
& ($170.5 \pm 17.76$, $0.3497 \pm 0.2578$)
& ($174.38 \pm 15.81$, $0.3717 \pm 0.2834$) \\
\bottomrule
\end{tabular}
\caption{Statistical results in DYS-JFB training under different dynamics using multiple seeds. While we mainly use the final epoch results for the rest of the paper, we here also include results from best trained model logged during training to provide a more comprehensive view of the solver. Variances are expected given the hard constraints nature of the problems, we do note that despite having different values, all models converge regardless of seeding under our proposed approach. Similarly we display both running cost and terminal state error here. }
\label{tab:seeded_results}
\end{table}

\subsection{Additional Numerical Results: A Case Study on Different CBF choices}
While investigating different control barrier function formulations is not the primary focus of this work, we nonetheless conduct additional numerical experiments to verify that our proposed approach remains applicable and performs well across different CBF choices. In particular, we focus on alternative CBF formulations for modeling rectangular obstacles.

\paragraph{$p$ norm-based CBFs.}
A simple and widely used approach for modeling different obstacle shapes is to generalize the norm used in the CBF definition. In particular, replacing the standard Euclidean norm with an $p$ norm yields the barrier function
\[
    h(x) = \|x - c\|_p^p - r^p,
\]
where $c$ is the obstacle center and $r$ its size. As $p$ increases, the level sets of $h$ become increasingly box-like, providing a convenient way to approximate rectangular obstacles. This formulation is smooth for $p > 1$, easy to differentiate, and requires only minimal modifications to the implementation, making it a practical first step for modeling different geometries. However, it is important to note that large values of $p$ can introduce numerical challenges. In particular, the barrier function and its derivatives become increasingly ill-conditioned, which can lead to numerical instability and as such higher computational cost in the projection layers.

\paragraph{Smooth signed box CBFs.}
A common way to represent an axis-aligned rectangular obstacle centered at $c$ with half-width $r$ is via the max-type barrier, take a 2-dimensional problem for example, we have
\begin{equation}
    h(x) = \max\{|x_1 - c_1|,\ |x_2 - c_2|\} - r,
\end{equation}
which enforces that the state remains outside the boxed region. This formulation directly captures rectangular geometry, in contrast to $p$ norm-based approximations which require large $p$ for accurate representation and can be unstable to use in practice.
For our experimental setup, to ensure smoothness and differentiability of the CBF we use a smooth approximation of the signed box barrier, which ensures compatibility with the higher-order CBF construction while remaining faithful to the underlying box geometry.

We conduct comparative experimentation using the double integrator dynamics with three different obstacles. We consider four total scenarios, namely $p=1,2,4$ as well as the box CBF definition for rectangular obstacles. We display trajectories under trained controller in Figure~\ref{fig:ablation_square_cbf}. We observe consistent convergence to the target while respecting safety constraints. 
We note in our testing that although larger values of $p$ yield shapes that more closely approximate rectangular obstacles, they lead to increasingly slower convergence unstable numerical behavior. In contrast, the smooth box formulation provides stable performance while accurately capturing the desired geometry. Overall, these results highlight that our approach is both general and principled, and can be applied across different CBF choices without modification.

\begin{figure}[t]
    \centering
    \begin{subfigure}[b]{0.24\textwidth}
        \centering
        \includegraphics[width=\textwidth]{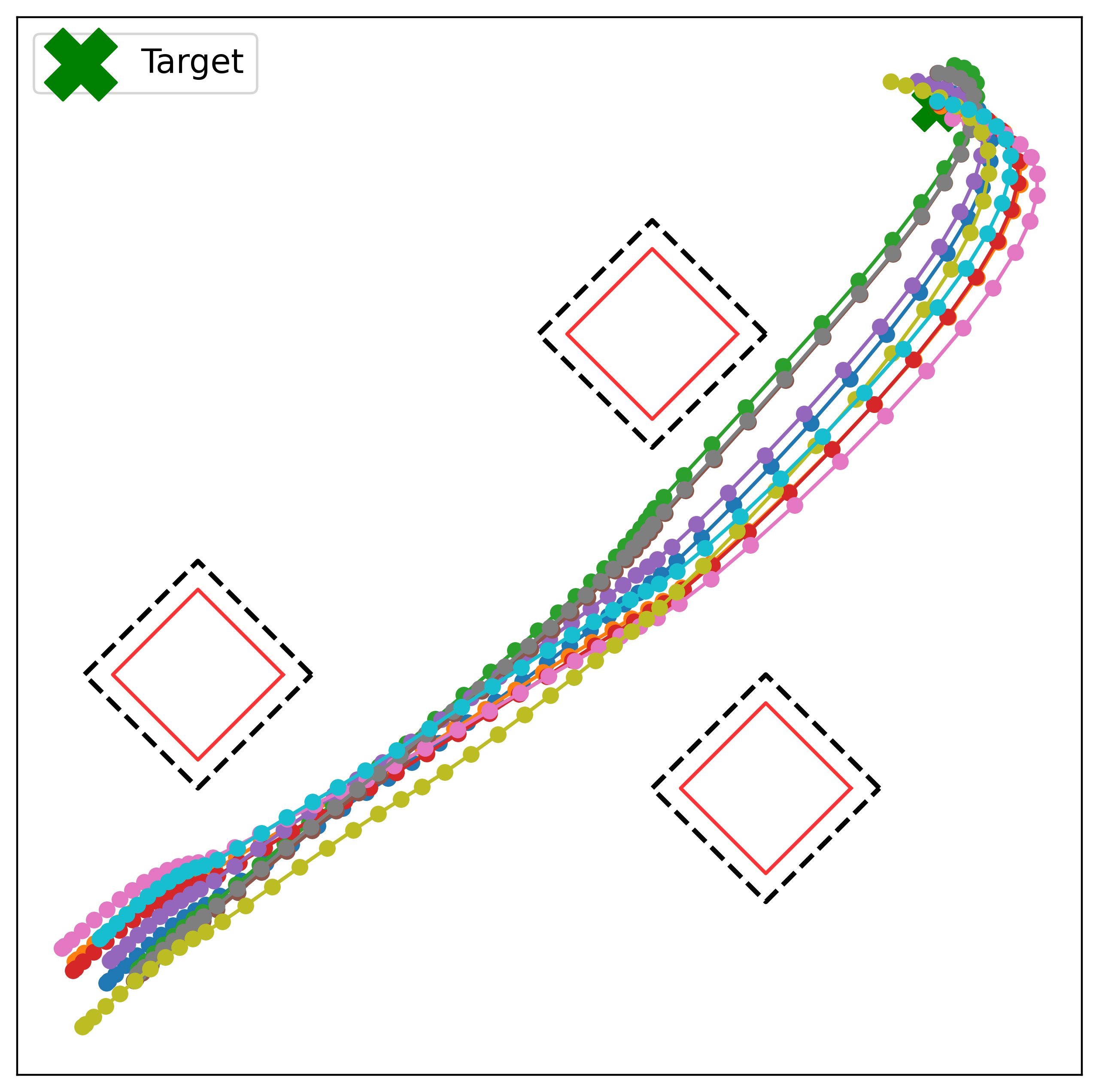}
        \caption{p1 norm CBF}
        \label{fig:p1_ball}
    \end{subfigure}
    \hfill
    \begin{subfigure}[b]{0.24\textwidth}
        \centering
        \includegraphics[width=\textwidth]{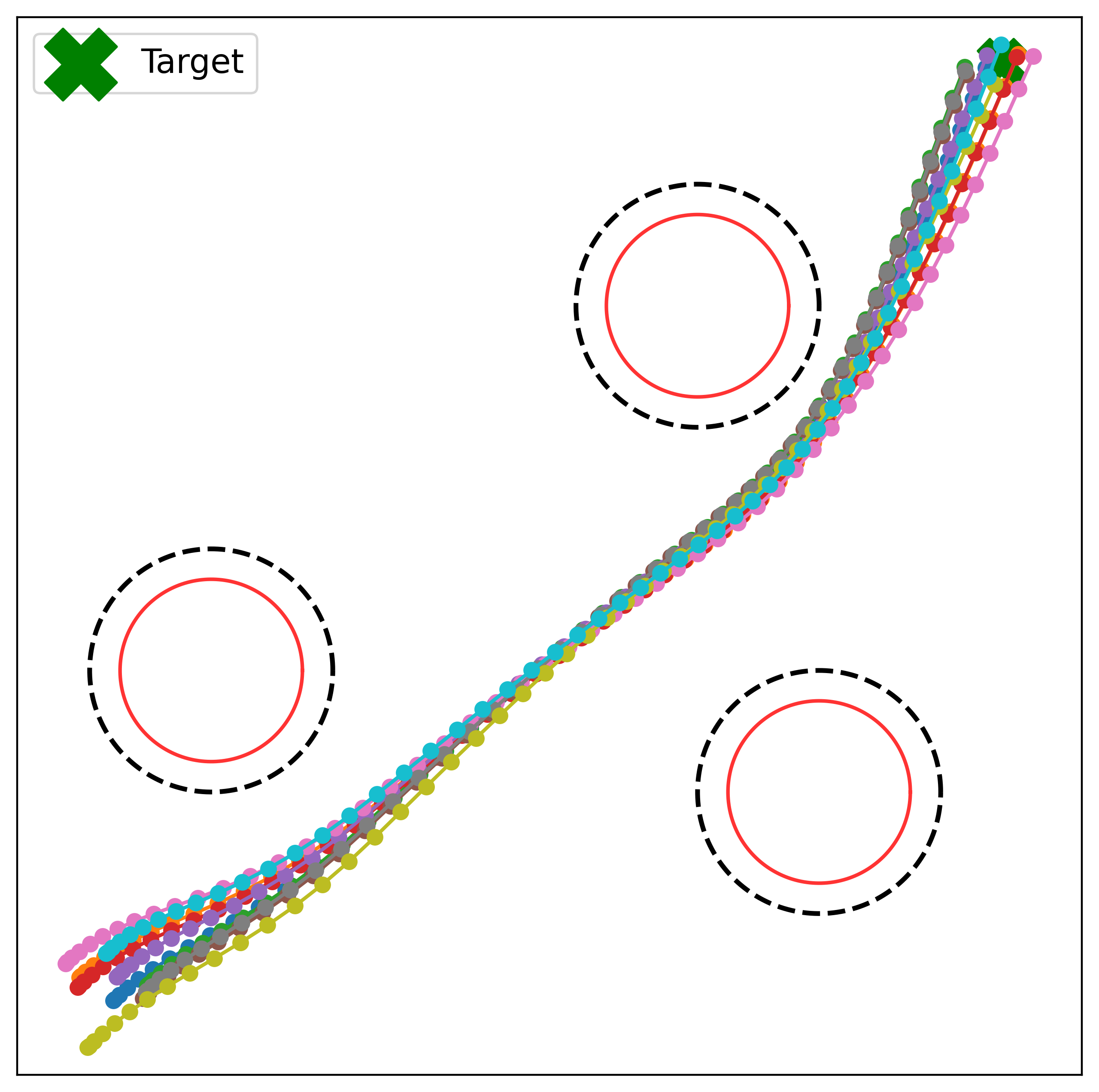}
        \caption{p2 norm CBF}
        \label{fig:p2_ball}
    \end{subfigure}
    \hfill
    \begin{subfigure}[b]{0.24\textwidth}
        \centering
        \includegraphics[width=\textwidth]{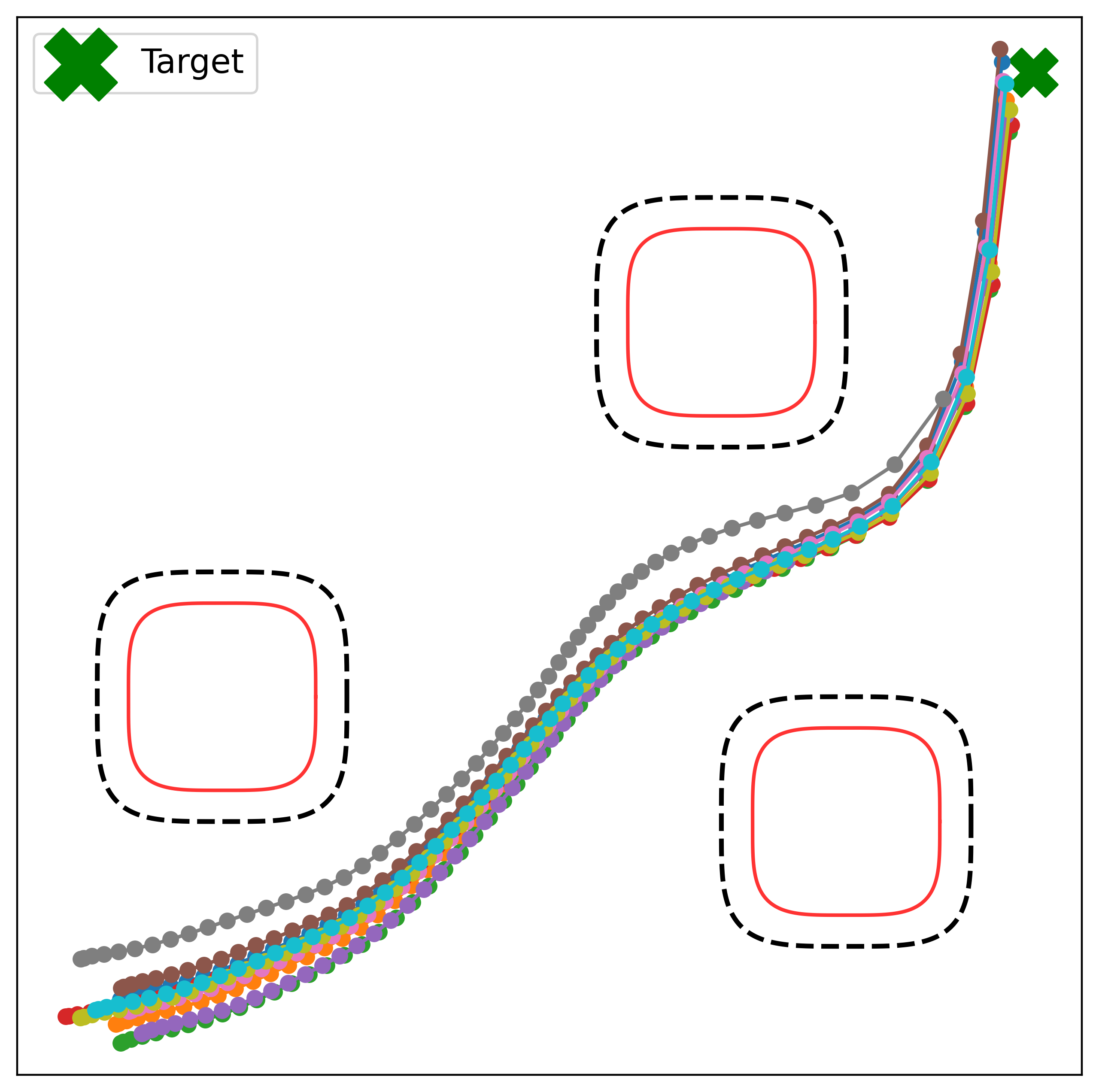}
        \caption{p4 norm CBF}
        \label{fig:p4_ball}
    \end{subfigure}
    \hfill
    \begin{subfigure}[b]{0.24\textwidth}
        \centering
        \includegraphics[width=\textwidth]{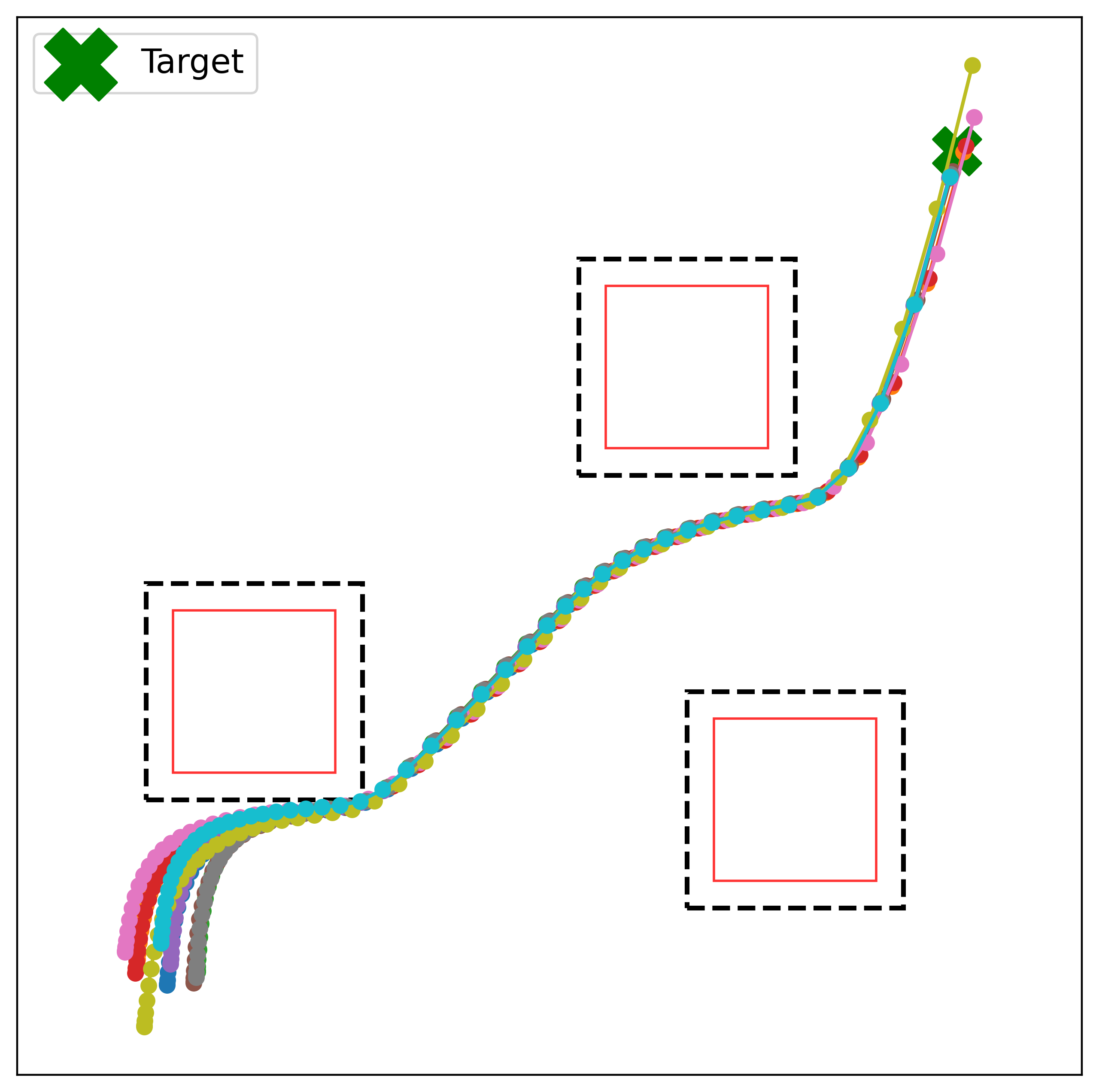}
        \caption{signed box CBF}
        \label{fig:box}
    \end{subfigure}
    \caption{Controlled trajectories using the learned feedback controller over different initial conditions and obstacles shapes, our approach yields close to optimal performances over all tested settings, while preserving the safety requirements. }
    \label{fig:ablation_square_cbf}
\end{figure}

\paragraph{Limitations and Fine Tuning.}
Training implicit neural networks can be challenging and is often less stable than training their explicit counterparts. These challenges are well documented in~\cite{bai2019deep,winston2020monotone,bai2021stabilizing,jafarpour2021robust,fung2022jfb,chu2024lyapunov}. Our problem formulation also requires stacking these implicit network models along trajectory rollouts, sometimes incurring thousands of fixed-point iterations for a forward pass, further adding to the complexity of the training problem. These difficulties are inherent to the hard-constrained control problems themselves and may not be resolved through optimization alone.
We notice and document the statistical variance in running and terminal costs across multiple seeded runs in Table~\ref{tab:seeded_results}. Despite some differences in values, we highlight that all tested trials converge regardless of initialization and random seeding, thus suggesting our method is reliable in problems across different scales.

While we use standard multi-layer perceptrons (MLPs) in all displayed results with limited tuning, upon further testing, we find that additional changes in weight decay, learning rate, terminal cost weight scheduling, and model sizes can help stabilize training and reduce variance across runs. In particular, we notice that for problems where agent behaviors share similarities, network architectures that exploit symmetry can significantly reduce training variance and improve convergence speed, though this is beyond the scope of this work and we leave further investigation as future directions. Finally, trajectory rollouts and visualizations for each example are displayed in Figure~\ref{fig:six_panel_grid}.

\section{Final Discussion}
\paragraph{Proof sketch.}
The convergence analysis proceeds in three stages. The main technical challenges are the nonsmoothness introduced by the projection layer and the bias of the JFB approximation. The former prevents the use of classical smooth analysis techniques based on Taylor expansions, while the latter means that descent is not immediate. Our proof addresses these issues in the following sequence.

\begin{itemize}
    \item \textbf{Step 1: Contraction of the DYS operator.}
    We first establish that the DYS fixed-point operator is contractive with respect to the lifted variable $y$ for any $\zeta\in(0,1)$. Since the projection onto the feasible set is nonsmooth, classical Jacobians need not exist everywhere. We therefore work with Clarke generalized Jacobians and show that every Clarke generalized Jacobian satisfies the same contraction bound. This guarantees existence and uniqueness of the fixed point and provides the stability estimates required in the subsequent analysis.

    \item \textbf{Step 2: From pointwise alignment to trajectory-level descent.}
    The core of the proof is to show that the JFB approximation remains sufficiently aligned with the true Clarke subgradient despite ignoring the inverse Jacobian. We first establish a pointwise alignment result for each time $t$ and initial condition $x$. This estimate is then lifted to the expectation over initial conditions and finally integrated over the time horizon to obtain a descent direction for the complete trajectory objective. Assumption~\ref{assumption:M}(iii) controls the contribution of the slack variables introduced by the safety layer, ensuring that they do not dominate the gradient alignment.

    \item \textbf{Step 3: Convergence under gradient flow.}
    Having established trajectory-level descent, we consider the continuous-time gradient flow generated by the JFB update. The use of gradient flow is essential here because the nonsmooth objective prevents the standard Taylor-expansion arguments underlying discrete gradient-descent analyses. Instead, we analyze the behavior of Clarke subgradients in a neighborhood of the gradient-flow trajectory, which allows us to prove asymptotic convergence to a Clarke stationary point.
\end{itemize}

\paragraph{Key advantages of Davis--Yin splitting.}
We highlight the key properties of DYS that motivate this work and explain why we believe it is particularly well suited for end-to-end training with embedded CBF-QP layers. We summarize these advantages in two parts.
\begin{itemize}
    \item \textbf{Provable contraction.} Under the assumptions of the problem formulation, the DYS operator is contractive for any step size $\zeta\in(0,1)$. This yields a unique fixed point and provides the foundation for the convergence analysis developed in this work. Importantly, no problem-dependent step-size tuning is required, removing what can otherwise become a significant practical bottleneck as constraints evolve along trajectories. This property is generally unavailable for comparable splitting methods.
    
    \item \textbf{Simple closed-form updates.} Each component of the DYS iteration consists only of projections onto sets with closed-form solutions together with a simple gradient step. Consequently, every iteration is inexpensive to evaluate, which is particularly important for scaling to high-dimensional multi-agent problems.
\end{itemize}

While alternative splitting methods, such as ADMM~\cite{boyd2011distributed} and PDHG~\cite{chambolle2011first}, are also applicable to constrained optimization problems, they generally require more problem-specific derivations to obtain equally simple updates, and establishing comparable contraction properties is considerably less direct. For the CBF-QPs considered here, DYS therefore provides a particularly convenient combination of theoretical guarantees and computational simplicity.

\paragraph{Discussion and future work.}
The CBF layer solves a strongly convex quadratic program and therefore admits a unique solution $y^\star=(u^\star,s^\star)$. Consequently, the forward and backward computations need not employ the same numerical solver. Any algorithm that computes $y^\star$ may be used during the forward pass, after which a single DYS evaluation can be applied to the converged solution to define the JFB backward pass. 
This decouples the choice of forward solver from differentiation, opening the possibility of combining JFB with more specialized or problem-dependent optimization algorithms, enabling potentially faster and more accurate forward solvers while retaining the simplicity of JFB in the backward pass, together with the descent and convergence guarantees established in this work. 
We believe this provides a promising direction for further improving the efficiency of embedded optimization layers for complex control problems.

\end{document}